\documentclass{article}

 \PassOptionsToPackage{numbers, compress}{natbib}

\usepackage[final]{neurips_2022}

\usepackage[utf8]{inputenc} 
\usepackage[T1]{fontenc}    
\usepackage{hyperref}       
\usepackage{url}            
\usepackage{booktabs}       
\usepackage{amsfonts}       
\usepackage{nicefrac}       
\usepackage{microtype}      
\usepackage{xcolor}         
\usepackage{caption}
\usepackage{subcaption}
\usepackage{amsmath}
\usepackage{comment}
\usepackage{graphicx}

\usepackage{amsmath}
\usepackage{amssymb}
\usepackage{mathtools}
\usepackage{amsthm}

\usepackage{amsfonts}  
\usepackage{algorithm}
\usepackage{algpseudocode}
\usepackage{wrapfig} 

\usepackage{multirow} 

\theoremstyle{plain}
\newtheorem{theorem}{Theorem}[section]
\newtheorem{proposition}[theorem]{Proposition}

\theoremstyle{definition}

\theoremstyle{remark}

\title{Unknown-Aware Domain Adversarial Learning \\for Open-Set Domain Adaptation}

\author{%
  JoonHo Jang\\
   KAIST \\
   \texttt{adkto8093@kaist.ac.kr} \\
   \And
   Byeonghu Na \\
   KAIST \\
   \texttt{wp03052@kaist.ac.kr} \\
   \And
   DongHyeok Shin \\
   KAIST \\
   \texttt{tlsehdgur0@kaist.ac.kr} \\
   \And
   Mingi Ji \thanks{now at Google (mingiji@google.com)} \\
   KAIST \\
   \texttt{qwertgfdcvb@kaist.ac.kr} \\
   \And
   Kyungwoo Song \\
   University of Seoul \\
   \texttt{kyungwoo.song@uos.ac.kr} \\
   \And
   Il-Chul Moon \\
   KAIST, Summary.AI \\
   \texttt{icmoon@kaist.ac.kr} \\
}

\begin{document}

\maketitle

\renewcommand{\thefootnote}{\fnsymbol{footnote}}

\begin{abstract}
Open-Set Domain Adaptation (OSDA) assumes that a target domain contains unknown classes, which are not discovered in a source domain. Existing domain adversarial learning methods are not suitable for OSDA because distribution matching with \textit{unknown} classes leads to negative transfer. Previous OSDA methods have focused on matching the source and the target distribution by only utilizing \textit{known} classes. However, this \textit{known}-only matching may fail to learn the target-\textit{unknown} feature space. Therefore, we propose Unknown-Aware Domain Adversarial Learning (UADAL), which \textit{aligns} the source and the target-\textit{known} distribution while simultaneously \textit{segregating} the target-\textit{unknown} distribution in the feature alignment procedure. We provide theoretical analyses on the optimized state of the proposed \textit{unknown-aware} feature alignment, so we can guarantee both \textit{alignment} and \textit{segregation} theoretically. Empirically, we evaluate UADAL on the benchmark datasets, which shows that UADAL outperforms other methods with better feature alignments by reporting state-of-the-art performances\footnote{The code will be publicly available on \url{https://github.com/JoonHo-Jang/UADAL}.}.
\end{abstract}

\section{Introduction}

\textit{Unsupervised Domain Adaptation} (UDA) means leveraging knowledge from a labeled source domain to an unlabeled target domain \cite{borgwardt2006integrating, ben2010theory, baktashmotlagh2014domain, long2015learning,ganin2016domain}. 
This adaptation implicitly assumes the source and the target data distributions, where each distribution is likely to be drawn from different distributions, i.e., \textit{domain shift} (see Figure \ref{fig:preli_sourceonly}). 
Researchers have approached the modeling of two distributions by statistical matching \cite{pan2010domain,tzeng2014deep,long2015trans, long2016unsupervised}, or domain adversarial learning \cite{ganin2015unsupervised,tzeng2017adversarial,long2017deep}. Among them, \textit{domain adversarial learning} has been successful in matching between the source and the target distributions via feature alignment, so the model can accomplish the domain-invariant representations.

There is another dynamic aspect of the source and the target distribution. In a realistic scenario, these distributions may expand a class set, which is called \textit{unknown} classes. This expansion creates a field of Open-Set Domain Adaptation (OSDA) \cite{saito2018open,liu2019separate}. 
Existing adversarial learning methods of UDA have limitations to solve OSDA because matching the source and the target distribution with the \textit{unknown} classes may lead to the negative transfer \cite{fang2020open} due to the class set mismatch (see Figure \ref{fig:preli_dann}).

Previous OSDA methods focused on matching between the source and the target domain only embedded in the \textit{known} class set via domain adversarial learning \cite{saito2018open,liu2019separate}. 
However, this \textit{known-only} feature alignment may fail to learn the target-unknown feature space because of no alignment signal from the target-unknown instances. 
Therefore, a classifier is not able to learn a clear decision boundary for \textit{unknown} classes because the target-unknown instances are not segregated enough in the aligned feature space (see Figure \ref{fig:preli_sta}).
On the other hand, some OSDA methods propose to learn intrinsic target structures by utilizing self-supervised learning without distribution matching \cite{DANCE2020NIPS,li2021domain}. However, this weakens the model performance under large domain shifts. 
Therefore, in order to robustly solve OSDA, distribution matching via the domain adversarial learning is required, and the class set mismatch should be resolved to prevent the negative transfer, simultaneously.

    \begin{figure*}
        \centering
        \begin{subfigure}[h]{0.232\textwidth}
            \centering
            \includegraphics[width=\textwidth]{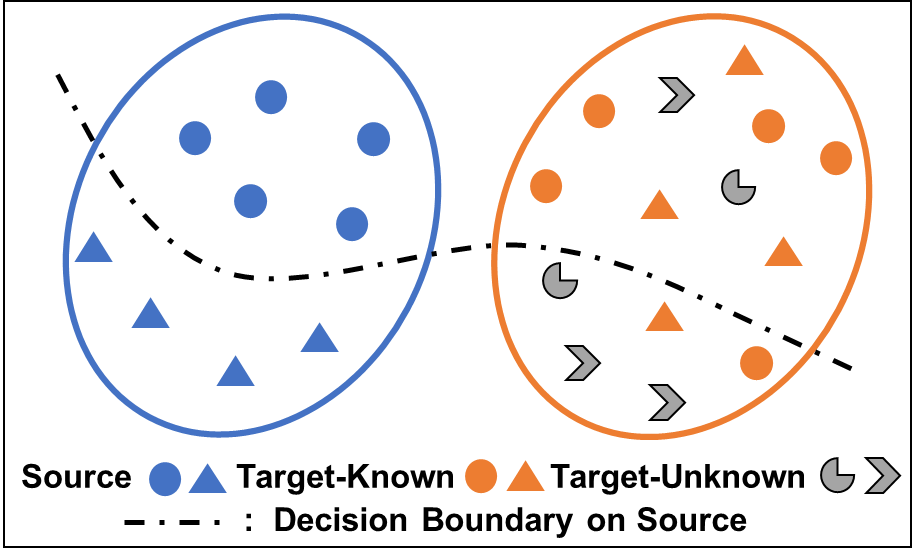}
            \caption{{\small Domain Shift}}    
            \label{fig:preli_sourceonly}
        \end{subfigure}
\hspace{0.01\textwidth}
        \begin{subfigure}[h]{0.232\textwidth}  
            \centering 
            \includegraphics[width=\textwidth]{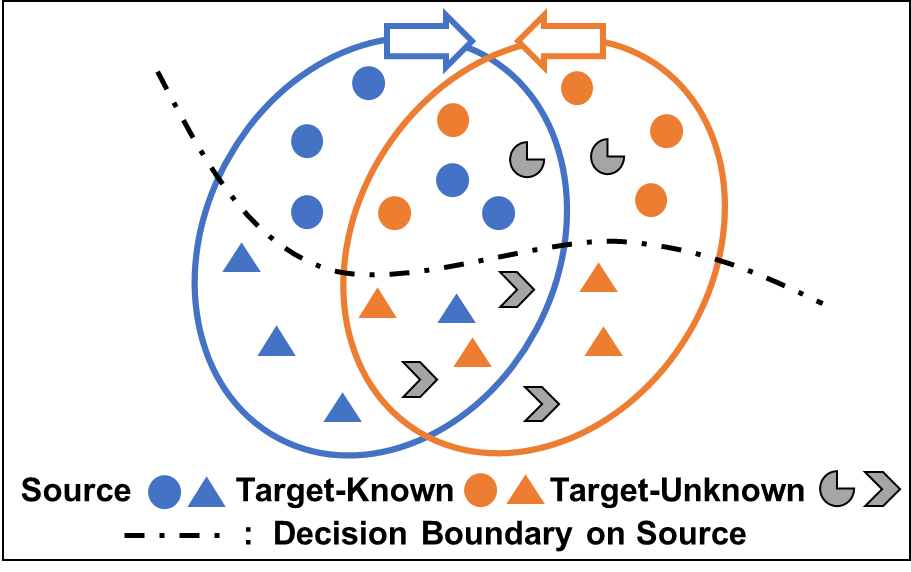}
            \caption{{\small DANN \cite{ganin2016domain}}}   
            \label{fig:preli_dann}
        \end{subfigure}
\hspace{0.01\textwidth}
        \begin{subfigure}[h]{0.232\textwidth}   
            \centering 
            \includegraphics[width=\textwidth]{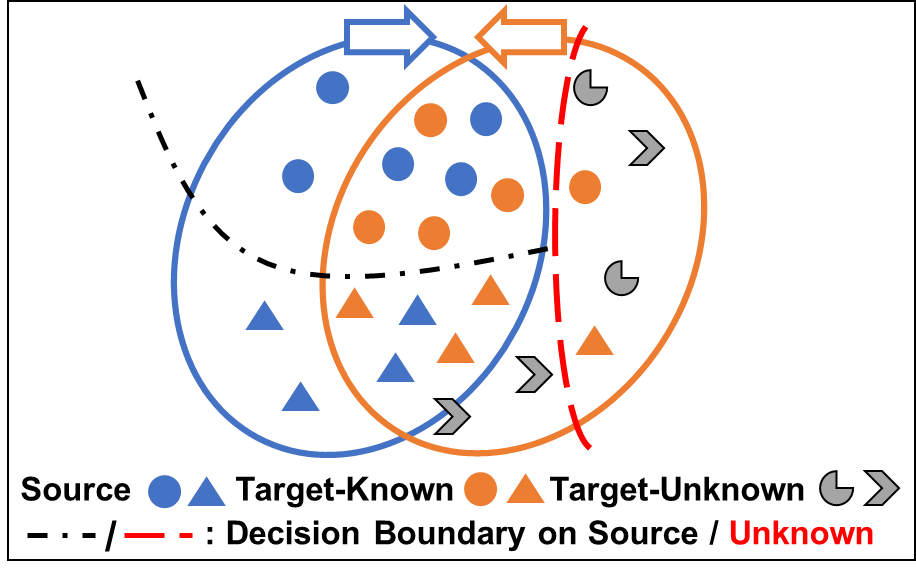}
            \caption{{\small STA \cite{liu2019separate}}}   
            \label{fig:preli_sta}
        \end{subfigure}
\hspace{0.01\textwidth}
        \begin{subfigure}[h]{0.232\textwidth}   
            \centering 
            \includegraphics[width=\textwidth]{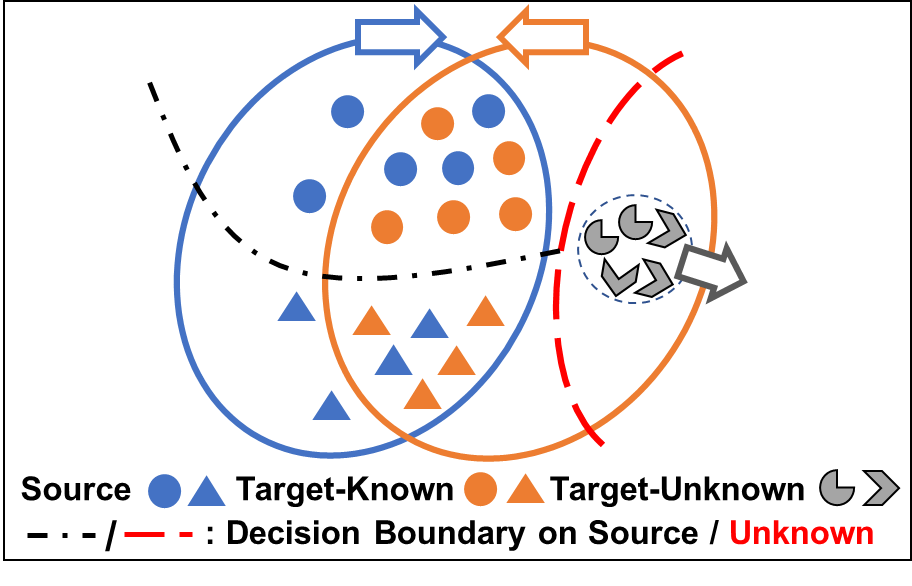}
            \caption{{\small UADAL (ours)}}   
            \label{fig:preli_mymodel}
        \end{subfigure}
  \caption{\small{The feature distributions from the source and the target domain, with the decision boundaries. The blue/orange/gray arrows represent the alignment signal on the source/target-known/target-unknown domain.}}   
        \label{fig:preli_flow}
    \end{figure*}
This paper proposes a new domain adversarial learning for OSDA, called Unknown-Aware Domain Adversarial Learning (UADAL). 
Specifically, we aim at enforcing the target-unknown features to \textit{move apart} from the source and the target-known features, while aligning the source and the target-known features. We call these alignments as the \textit{segregation} of the target-unknown features (gray-colored arrow in Figure \ref{fig:preli_mymodel}). 
Although this distribution \textit{segregation} is an essential part of OSDA, the segregation has been only implicitly modeled because of its identification on the target-unknown instances. 
UADAL is the first explicit mechanism to simultaneously align and segregate the three sets (source, target-known, and target-unknown instances) via the domain adversarial learning. 
Therefore, the proposed \textit{unknown-aware} feature alignment enables a classifier to learn a clear decision boundary on both \textit{known} and \textit{unknown} class in the feature space.

UADAL consists of three novel mechanisms. 
First, we propose a new domain discrimination loss to engage the target-unknown information. 
Second, we formulate a sequential optimization to enable the unknown-aware feature alignment, which is suitable for OSDA. 
We demonstrate that the optimized state of the proposed alignment is theoretically guaranteed. 
Third, we also suggest a posterior inference to effectively recognize target-\textit{known}/\textit{unknown} information without any thresholding.

\section{Preliminary}
\subsection{Open-Set Domain Adaptation}
This section provides a formal definition of OSDA. 
The fundamental properties are two folds: 1) the different data distributions of a source domain, $p_{s}(x, y)$, and a target domain, $p_{t}(x, y)$; and 2) the additional classes of the target domain, which were not observed in the source domain. 
Specifically, we define the source and the target datasets as $\chi_{s}=\{(x_{s}^{i}, y_{s}^{i})\}_{i=1}^{n_{s}}$ and $\chi_{t}=\{(x_{t}^{i}, y_{t}^{i})\}_{i=1}^{n_{t}}$, respectively. $y_{t}^{i}$ is not available at training under the UDA setting.
Additionally, we designate $\mathcal{C}_{s}$ to be the source class set (a.k.a. a shared-\textit{known} class), and $\mathcal{C}_{t}$ to be the target class set, i.e., $y_{t}^{i}\in\mathcal{C}_{t}$. 
OSDA dictates $\mathcal{C}_{s} \subset \mathcal{C}_{t}$ \cite{saito2018open}. $\mathcal{C}_{t}\setminus\mathcal{C}_{s}$ is called \textit{unknown} classes. 
In spite that there can be multiple \textit{unknown} classes, we consolidate $\mathcal{C}_{t}\setminus\mathcal{C}_{s}$ as $y_{unk}$ to be a single \textit{unknown} class, due to no knowledge on $\mathcal{C}_{t}\setminus\mathcal{C}_{s}$.

The learning objectives for OSDA become both 1) the optimal class classification in $\chi_{t}$ if a target instance belongs to $\mathcal{C}_{s}$, and 2) the optimal \textit{unknown} classification if a target instance belongs to $\mathcal{C}_{t}\setminus\mathcal{C}_{s}$. This objective is formulated as follows, with a classifier $f$ and the cross-entropy loss function $\mathcal{L}_{ce}$,
\begin{align}
\min_{f} \mathbb E_{p_{t}(x,y)} & [\textbf{1}_{y_t\in\mathcal{C}_{s}}  \mathcal{L}_{ce}(f(x_t), y_t) + \textbf{1}_{y_t\in \{\mathcal{C}_{t}\setminus\mathcal{C}_{s}\}}  \mathcal{L}_{ce}(f(x_t), y_{unk})].
\end{align}

\subsection{Adversarial Domain Adaptation}   
We first step back from OSDA to Closed-set DA (CDA) with $\mathcal{C}_{t}\setminus\mathcal{C}_{s}=\phi$, where the domain adversarial learning is widely used.
Domain-Adversarial Neural Network (DANN) \cite{ganin2015unsupervised} proposes, 
\begin{equation}
\min_{f, G}\max_{D}\{\mathbb E_{p_{s}(x,y)} [\mathcal{L}_{ce} (f(G(x_s)),y_s)]-\mathcal{L}_{d}\}.
\end{equation}
This objective assumes $G(x)$ to be a feature extractor that learns the domain-invariant features, enabled by the minimax game with a domain discriminator $D$, with respect to $\mathcal{L}_{d}$, as below.
\begin{equation}
\mathcal{L}_{d} = \mathbb E_{p_{s}(x)}\left[ -\log D_{s}(G(x))\right]+\mathbb E_{p_{t}(x)}[-\log D_{t}(G(x))]
\end{equation}
The adversarial framework adapts $G(x)$ toward the indistinguishable feature distributions between the source and the target domain by the minimax game on $D(G(x))=[D_{s}(G(x)), D_{t}(G(x))]$. Here, $D$ is set up to have a two-dimensional output to indicate either source or target domain, denoted as $D_s$ and $D_t$, respectively. 
Given these binarized outputs of $D$, this formulation is not appropriate to differentiate the target-known and the target-unknown features in OSDA. 
DANN enforces to include the target-\textit{unknown} features in the distribution matching, which leads to performance degradation by negative transfer \cite{fang2020open}. 
Figure \ref{fig:preli_dann} represents this undesired alignment of \textit{unknown} classes. 

\textbf{Separate To Adapt (STA)} \cite{liu2019separate} proposes a weighted domain adversarial learning, to resolve the negative transfer by utilizing a weighting scheme. Therefore, STA modifies $\mathcal{L}_{d}$ as follows, 
\begin{equation}\label{preli:sta}
\mathcal{L}_{d}^{\text{STA}} =\mathbb E_{p_{s}(x)}\left[ -\log D_{s}(G(x))\right]+\mathbb E_{p_{t}(x)}[-w_{x}\log D_{t}(G(x))], 
\end{equation}
where $w_{x}$ represents a probability of an instance $x$ belonging to the shared-known classes. 
$\mathcal{L}_{d}^{\text{STA}}$ enables the domain adversarial learning to align the features from the source and the most likely target-known instances with the high weights. 
However, the feature extractor, $G$, is not able to move apart the target-unknown instances due to lower weights, i.e., no alignment signals (see Figure \ref{fig:preli_sta}). 
\textcolor{black}{STA modeled this known-only feature alignment by assuming the unknown feature would be implicitly learned through the estimation on $w_x$. Besides learning $w_x$, the target-unknown instances cannot contribute to the training of STA because their loss contribution will be limited by lower $w_x$, in the second term of Eq. (\ref{preli:sta}). 
Therefore, STA does not take information from the target-unknowns, and later Section \ref{sec_ablation} empirically shows that the target-unknown segregation of STA is limited.}

\subsection{Comparison to Recent OSDA Research}
In terms of the domain adversarial learning, in addition to STA, 
OSBP \cite{saito2018open} utilizes a classifier to predict the target instances to the pre-determined threshold, and trains the feature extractor to deceive the classifier for aligning to \textit{known} classes or rejecting as \textit{unknown} class. 
However, their recognition on \textit{unknown} class only depends on the threshold value, without considering the data instances.
PGL \cite{luo2020progressive} introduces progressive graph-based learning to regularize the class-specific manifold, while jointly optimizing the feature alignment via domain adversarial learning. However, their adversarial loss includes all instances of the target domain, which is critically weakened by the negative transfer.

On the other hand, self-supervised learning approaches have been proposed recently, in order to exploit the intrinsic structure of the target domain \cite{ROS2020ECCV,DANCE2020NIPS, li2021domain}. However, these approaches do not have any feature alignments between the source and the target domain, which leads to performance degradation under significant domain shifts. There is also a notable work, OSLPP \cite{wang2021progressively} optimizing projection matrix toward a common subspace to class-wisely align the source and the target domain. 
However, the optimization requires pair-wise distance calculation, which results in a growing complexity of $O(n^2)$ by the $n$ data instances. It could be limited to the large-scale domain.
We provide detailed comparisons and comprehensive literature reviews in Appendix \ref{sup-literature_review}.

\begin{figure*}
    \centering
    \begin{subfigure}[t!]{0.37\textwidth} 
        \centering
        \includegraphics[width=\textwidth]{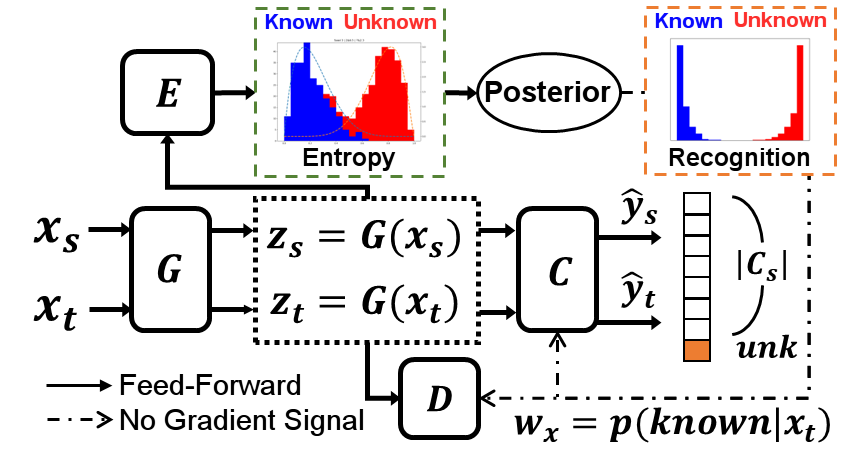}
        \caption{The overall structure of UADAL}
        \label{fig:method_overview_a}
    \end{subfigure}
\hspace{0.2em}
    \begin{subfigure}[t!]{0.61\textwidth} 
        \centering
        \includegraphics[width=\textwidth]{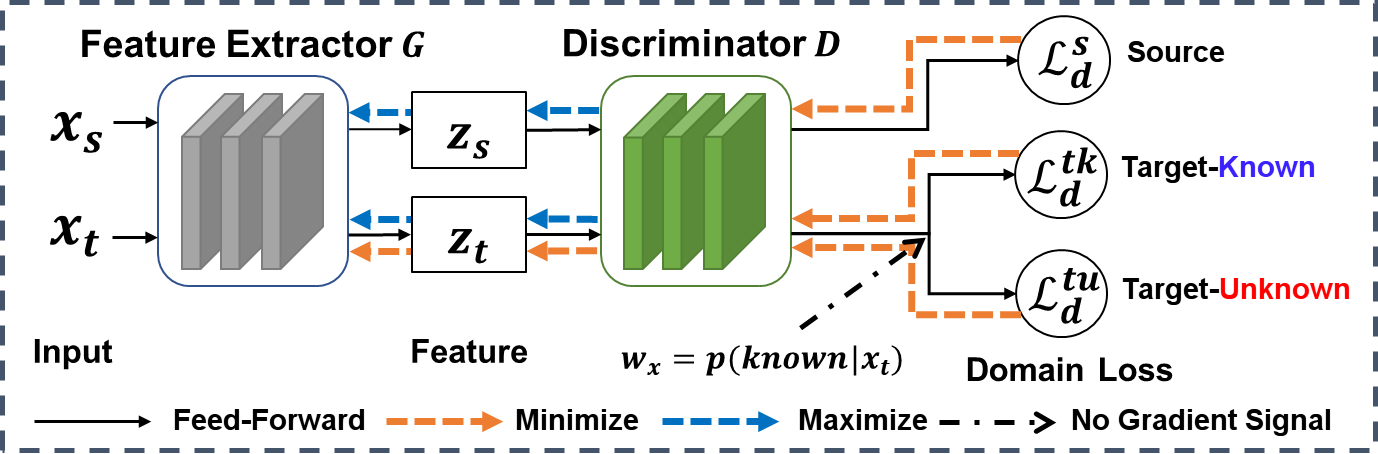}
        \caption{The optimization procedure of UADAL}
        \label{fig:method_overview_b}
    \end{subfigure}
\caption{\small{Overview of the proposed Unknown-Aware Domain Adversarial Learning (UADAL) approach.}} %
    \label{fig:method_overview}
\end{figure*}

\section{Methodology}
\subsection{Overview of UADAL}
Figure \ref{fig:method_overview_a} illustrates the neural network compositions and the information flow. Our model consists of four networks: 1) a feature extractor, $G$; 2) a domain discriminator, $D$; 3) an open-set recognizer, $E$; and 4) a class classifier network, $C$. 
The networks $G$, $D$, $E$, and $C$ are parameterized by $\theta_{g}, \theta_{d}, \theta_{e}$, and $\theta_{c}$, respectively.  
First, we propose a new domain discrimination loss, and we formulate the sequential optimization problem, followed by theoretical analyses. 
In order to recognize target-\textit{unknown} information, we introduce a posterior inference method, followed by open-set classification. 

\subsection{Sequential Optimization Problem for Unknown-Aware Feature Alignment}
\textbf{Model Structure} We assume that we identify three domain types, which are 1) the source (\textit{s}), 2) the target-known (\textit{tk}), and 3) and target-unknown (\textit{tu}). 
As the previous works such as STA and DANN utilize a two-way domain discriminator with ($D_s$, $D_t$), they can treat only \textit{s} and \textit{tk} (or \textit{t}).
However, a domain discriminator in OSDA should be able to handle \textit{tu} information for \textit{segregation}. Therefore, we designed the discriminator dimension to be three ($D_s$, $D_{tk}$, $D_{tu}$), as follows,
\begin{align}\label{d_output}
D(G(x)) = [D_{s}(G(x)), D_{tk}(G(x)), D_{tu}(G(x))].
\end{align}
Given these three dimensions of $D$, we propose new domain discrimination losses, $\mathcal{L}_{d}^{s}$ and $\mathcal{L}_{d}^{t}$, as
\begin{align}
\begin{split}\label{eqn:loss_d_s}
    \mathcal{L}_{d}^{s}(\theta_{g}, \theta_{d}) ={}& \mathbb E_{p_{s}(x)}\left[ - \log D_{s}(G(x))\right],
\end{split}\\
\begin{split}\label{eqn:loss_d_t}
    \mathcal{L}_{d}^{t}(\theta_{g}, \theta_{d}) ={}& \mathbb E_{p_{t}(x)}[\ -w_{x} \log D_{tk}(G(x)) - (1-w_{x}) \log D_{tu}(G(x))],
\end{split}
\end{align}
where $w_{x}:=p(\textit{known}|x)$, named as \textit{open-set recognition}, is the probability of a target instance, $x$, belonging to a shared-\textit{known} class. 
We introduce the estimation of this probability, $w_{x}$, in the section \ref{section_bmm} later. 
From $\mathcal{L}_{d}^{t}$ in Eq. (\ref{eqn:loss_d_t}), our modified $D$ becomes to be able to discriminate \textit{tk} and \textit{tu}, explicitly.

Our new unknown awareness of $D$ allows us to segregate the target-unknown features via the domain adversarial learning framework. 
Firstly, we decompose $\mathcal{L}_{d}^{t}(\theta_{g}, \theta_{d})$ into $\mathcal{L}_{d}^{tk}(\theta_{g}, \theta_{d})$ and $\mathcal{L}_{d}^{tu}(\theta_{g}, \theta_{d})$,
\begin{equation} \label{eqn:dt_decompose}
\mathcal{L}_{d}^{t}(\theta_{g}, \theta_{d}) = \mathcal{L}_{d}^{tk}(\theta_{g}, \theta_{d})  +  \mathcal{L}_{d}^{tu}(\theta_{g}, \theta_{d}),
\end{equation}
\begin{align}
\begin{split} \label{eqn:dt_decompose_tk}
\mathcal{L}_{d}^{tk}(\theta_{g}, \theta_{d}) := {}& \lambda_{tk}\mathbb E_{p_{tk}(x)}\left[ - \log D_{tk}(G(x))\right],
\end{split}\\
\begin{split} \label{eqn:dt_decompose_tu}
\mathcal{L}_{d}^{tu}(\theta_{g}, \theta_{d}) :={}& \lambda_{tu}\mathbb E_{p_{tu}(x)}\left[ - \log D_{tu}(G(x))\right],
\end{split}
\end{align}
where $ p_{tk}(x):=p_{t}(x|\textit{known})$, $p_{tu}(x):=p_{t}(x|\textit{unknown})$, $\lambda_{tk}:=p(\textit{known})$, and $\lambda_{tu}:=p(\textit{unknown})$. 
The derivation of Eq. (\ref{eqn:dt_decompose}) is based on $p_t(x)=\lambda_{tk} p_{tk}(x)+\lambda_{tu} p_{tu}(x)$, which comes from the law of total probability (Details in Appendix \ref{sup-lossDecompose}). 
Therefore, this decomposition of $\mathcal{L}_{d}^{t}$ into $\mathcal{L}_{d}^{tu}$ and $\mathcal{L}_{d}^{tk}$ enables the different treatments on \textit{tk} and \textit{tu} feasible. 
The three-way discriminator and its utilization in Eq. (\ref{eqn:dt_decompose_tk})-(\ref{eqn:dt_decompose_tu}) becomes the unique contribution from UADAL. 

\textbf{Optimization} Our goal of this unknown-aware domain discrimination is to {align} the features from the source and the target-known instances while simultaneously \textit{segregating} the target-unknown features. 
To achieve this goal, we propose a new sequential optimization problem w.r.t. $G$ and $D$. Based on the losses, $\mathcal{L}_{d}^{s}(\theta_{g}, \theta_{d})$, $\mathcal{L}_{d}^{tk}(\theta_{g}, \theta_{d})$, and $\mathcal{L}_{d}^{tu}(\theta_{g}, \theta_{d})$, we formulate the optimization problem as,
\begin{align}
\begin{split}\label{eqn:max_D}
\min_{\theta_{d}}\ \mathcal{L}_{D}(\theta_{g}, \theta_{d})&=\mathcal{L}_{d}^{s}(\theta_{g}, \theta_{d}) +\ \mathcal{L}_{d}^{tk}(\theta_{g}, \theta_{d}) + \mathcal{L}_{d}^{tu}(\theta_{g}, \theta_{d}),
\end{split} \\
\begin{split}
\label{eqn:max_G}
\max_{\theta_{g}}\ \mathcal{L}_{G}(\theta_{g}, \theta_{d})&= \mathcal{L}_{d}^{s}(\theta_{g}, \theta_{d})  + \ \mathcal{L}_{d}^{tk}(\theta_{g}, \theta_{d}) - \mathcal{L}_{d}^{tu}(\theta_{g}, \theta_{d}).
\end{split}
\end{align}
This alternating objective by Eq. (\ref{eqn:max_D}) and Eq. (\ref{eqn:max_G}) is equivalent to the adversarial domain adaptation models \cite{ganin2015unsupervised, liu2019separate} to learn the domain-invariant features. 
Unlike the previous works, however, Eq. (\ref{eqn:max_D}) represents that we train the domain discriminator, $D$, to classify an instance as either source, target-known, or target-unknown domain. 
In terms of the feature extractor, $G$, we propose Eq. (\ref{eqn:max_G}) to maximize the domain discrimination loss for the source and the target-known domain while minimizing the loss for the target-unknown domain. From the different signals on $\mathcal{L}_{d}^{tk}$ and $\mathcal{L}_{d}^{tu}$, we only treat the adversarial effects on $\mathcal{L}_{d}^{s}$ and $\mathcal{L}_{d}^{tk}$. Meanwhile, minimizing $\mathcal{L}_{d}^{tu}$ provides for the network $G$ to learn the discriminative features on $tu$.
Eventually, $G$ and $D$ align the source and target-known features and segregate the target-unknown features from the source and the target-known features (the optimization details in Figure \ref{fig:method_overview_b}). 
We provide the theoretic analysis of the optimized state of the proposed feature alignment in the next subsection, which is unexplored in the OSDA field.

\subsection{Theoretic Analysis of Sequential Optimization}
\label{sec:theoretic}
This section provides a theoretic discussion on the proposed feature alignment by our sequential optimization, w.r.t. Eq. (\ref{eqn:max_D})-(\ref{eqn:max_G}). Since our ultimate goal is training $G$ to learn unknown-aware feature alignments, 
we treat $G$ as a \textit{leader} and $D$ as a \textit{follower} in the game theory. We first optimize Eq. (\ref{eqn:max_D}) to find $D^{*}$ with a fixed $G$. The optimal $D^{*}$ given the fixed $G$ results in the below output :
\begin{align}
\label{eqn:optimal_D}
D^{*}(z)= \Big[\frac{p_{s}(z)}{2p_{avg}(z)}, \frac{\lambda_{tk} p_{tk}(z)}{2p_{avg}(z)}, \frac{\lambda_{tu} p_{tu}(z)}{2p_{avg}(z)} \Big], 
\end{align}
where $p_{avg}(z)=(p_{s}(z) + \lambda_{tk} p_{tk}(z) + \lambda_{tu} p_{tu}(z))/2$ (Details in Appendix \ref{sup-appendidx_optim_D}). Here, $z\in\mathcal{Z}$ stands for the feature space from $G$, i.e., $p_{d}(z) = \{G(x;\theta_{g})|x \sim p_{d}(x)\}$
where $d$ is $s$, $tk$, or $tu$.  
Given $D^{*}$ with the optimal parameter $\theta^{*}_{d}$, we optimize $G$ in Eq. (\ref{eqn:max_G}). 
Here, we show that the optimization w.r.t. $G$ is equivalent to the weighted summation of KL Divergences, by Theorem \ref{theorem1_optim_G}.
\begin{theorem}
(Proof in Appendix \ref{sup-appendix_optim_G}) Let $\theta^{*}_{d}$ be the optimal parameter of $D$ by optimizing Eq. (\ref{eqn:max_D}). Then, $- \mathcal{L}_{G}(\theta_{g}, \theta^{*}_{d})$ can be expressed as, with a constant $C_{0}$,
\begin{equation}
\begin{split}
 -\mathcal{L}_{G}(\theta_{g},\theta^{*}_{d})=D_{KL}( p_{s} \Vert p_{avg})+\lambda_{tk}D_{KL}(p_{tk}\Vert p_{avg}) - \lambda_{tu} D_{KL}( p_{tu} \Vert p_{avg} )+C_{0}.
\end{split}\label{eqn:optimal_G}
\end{equation}\label{theorem1_optim_G}
\end{theorem}\vspace{-1.5em}
We note that $p_{s}, p_{tk}$, and $p_{tu}$ are the feature distribution of each domain, mapped by $G$, respectively. Therefore, we need to minimize Eq. (\ref{eqn:optimal_G}) to find the optimal $G^{*}$. Having observed Theorem \ref{theorem1_optim_G},
Eq. (\ref{eqn:optimal_G}) requires $p_{tu}$ to have a higher deviation from the average distribution, $p_{avg}$, while matching towards $p_{s}\approx p_{avg}$ and $p_{tk}\approx p_{avg}$. From this optimization procedure, we obtain the domain-invariant features over the source and the target-known domain, while segregating the target-unknowns, which is formalized as $p_{s}\approx p_{tk}$ and $p_{tu}\leftrightarrow \{ p_{tk}, p_{s} \}$.
We show that minimizing Eq. (\ref{eqn:optimal_G}) does not lead to the negative infinity by the third term, regardless of the other KL Divergences, by Proposition \ref{thm:noninf2}.
\begin{proposition}\label{thm:noninf2}  
(Proof in Appendix \ref{sup-boundness}) $D_{KL}( p_{tu} \Vert p_{avg} )$ is bounded to\ $\log2-\log{\lambda_{tu}}$. 
\end{proposition}
This boundness guarantees that the first two KL Divergence terms maintain their influence on optimizing the parameter $\theta_{g}$ of $G$ stably while \textit{segregating} the target-unknown features. 
Furthermore, we show Proposition \ref{thm:f_div} to characterize the feature alignment between $p_s$ and $p_{tk}$ under Eq. (\ref{eqn:optimal_G}). 
\begin{proposition}\label{thm:f_div}
(Proof in Appendix \ref{sup-appendix_fdivergence}) Assume that $\text{supp}(p_s) \cap \text{supp}(p_{tu})=\emptyset$ and $\text{supp}(p_{tk})\cap \text{supp}(p_{tu})=\emptyset$, where $\text{supp}(p):=\{ z \in \mathcal{Z} | p(z) > 0 \}$ is the support set of $p$. Then, the minimization w.r.t. $G$, by Eq. (\ref{eqn:optimal_G}), is equivalent to the minimization of the summation on two $f$-divergences: 
\begin{equation*}
\begin{aligned}
D_{f_1}(p_{s}||p_{tk}) + \lambda_{tk} D_{f_2}(p_{tk}||p_{s}),
\end{aligned} \label{eq_f_divergence}
\end{equation*}
where $f_{1}(u)=u \log \frac{u}{(1-\alpha)u+\alpha}$ \text{and} $f_{2}(u)=u \log \frac{u}{\alpha u+(1-\alpha)}$, with $\alpha=\frac{\lambda_{tk}}{1+\lambda_{tk}}$.
Therefore, the minimum of Eq. (\ref{eqn:optimal_G}) is achieved if and only if $p_{s}=p_{tk}$.
\end{proposition}
Therefore, under the assumption, Proposition \ref{thm:f_div} theoretically guarantees that the source and the target-known feature distributions are aligned, which represents $p_s = p_{tk}$. From these analyses, we can theoretically guarantee both \textit{alignment} and \textit{segregation} from the proposed optimization.

\subsection{Open-Set Recognition via Posterior Inference} \label{section_bmm}
This section starts from estimating $w_{x}=p(\textit{known}|x)$ for a target instance, $x$. 
First, we utilize the labeled source dataset, $\chi_s$, to train the feature extractor $G$ and the open-set recognizer $E$, as below,
\begin{equation} \label{eqn:e_source}
\mathcal{L}_{e}^{s}(\theta_{g}, \theta_{e}) =\frac{1}{n_{s}}\sum_{(x_{s}, y_{s})\in \chi_{s}}\mathcal{L}_{ce}(E(G(x_{s})), y_{s}),
\end{equation}
where $E(G(x)) \in \mathbb{R}^{|\mathcal{C}_{s}|}$ contains the softmax activation. We note that the open-set recognizer, $E$, is a classifier to learn the decision boundary over $\mathcal{C}_{s}$ by the source domain.

Given the decision boundary by $E$, our assumption for open-set recognition includes two statements: for the target instances, 1) the higher entropy is caused by an uncertain classification case, and 2) the \textit{unknown} class may induce the higher entropy because there are no prior training instances in the \textit{unknown} classes. 
Based on two assumptions, the open-set recognizer, $E$, will provide higher entropy for the target-unknown instances than the target-known instances. 
Therefore, we consider the entropy value as the open-set recognition indicator. 
We verified our assumptions empirically in Figure \ref{fig:histo_warm-up}.

While we notice the entropy value from a single instance as the indicator between \textit{known} and \textit{unknown}, the open-set recognition should be holistically modeled over the target domain. Therefore, we model the mixture of two Beta distributions on the normalized target entropy values as below,
\begin{equation} \label{eqn:mixturemodel}
p(\ell_{x})=\lambda_{tk} p(\ell_{x}|\textit{known}) + \lambda_{tu} p( \ell_{x}|\textit{unknown}),
\end{equation} 
where the definition of $\lambda_{tk}$ and $\lambda_{tu}$ is the same as in the previous subsection. $\ell_{x}$ is the entropy value for the target instance, $x$, i.e., $\ell_{x}=H(E(G(x)))$ with the entropy function $H$. The Beta distribution is a perfect fit for this case of the closed interval support of $[0,1]$, which is the normalized range of the entropy. 
Thus, we introduce the estimator $\hat{w}_{x}$ of $w_x$ by \textit{Posterior Inference}, from fitting the Beta mixture model through the Expectation-Maximization (EM) algorithm (Details in Appendix \ref{sup-BMMem}),
\begin{align}\label{eqn:posterior}
\hat{w}_{x}:=p(\textit{known}|\ell_{x})=\frac{\lambda_{tk}p(\ell_{x}|known)}{\lambda_{tk} p(\ell_{x}|\textit{known}) + \lambda_{tu} p( \ell_{x}|\textit{unknown})},
\end{align}
where the denominator is from Eq. (\ref{eqn:mixturemodel}). Since the fitting process incorporates all target information, the entire dataset is utilized to set a more informative weighting without any thresholds. 
Here, $\lambda_{tk}$ and $\lambda_{tu}$ are also explicitly estimated by the EM algorithm since they are not given (see Eq. (\ref{sup-eqn:lambda_tk_tu}) in Appendix \ref{sup-BMMem}), whereas a threshold hyperparameter is used by the previous researches \cite{luo2020progressive, wang2021progressively}.

\subsection{Open-Set Classification}
Given the proposed feature alignments, we train a classifier, $C$, to correctly classify the target instance over $\mathcal{C}_s$ or to reject as \textit{unknown}. 
Note that the classifier, $C$, is the extended classifier with the dimension of $|\mathcal{C}_{s}|+1$ to include the \textit{unknown} class.
Firstly, we construct the source classification loss on $\chi_{s}$. Secondly, for the target domain, we need to learn the decision boundary for the \textit{unknown} class. Based on the proposed open-set recognition, $\hat{w}_{x}$, we train the classifier, $C$, as follows:
\begin{equation} \label{eqn:c_source}
\mathcal{L}_{cls}(\theta_{g}, \theta_{c}) =\frac{1}{n_{s}} \sum_{(x_{s}, y_{s})\in \chi_{s}}\mathcal{L}_{ce}(C(G(x_{s})), y_{s})+\frac{1}{{n_{t}}'} \sum_{x_{t}\in \chi_{t}} (1-\hat{w}_{x}) \mathcal{L}_{ce}(C(G(x_{t})), y_{unk}),
\end{equation}
where ${n_{t}}'$ is a normalizing constant, and $y_{unk}$ is the one-hot vector for the \textit{unknown} class. 
Furthermore, we incorporate entropy minimization loss for enforcing the unlabeled target instances to be recognized effectively \cite{grandvalet2005semi, liu2019separate, ROS2020ECCV}, with the entropy loss function, $\mathcal{L}_{H}$,
\begin{align} \label{eqn:c_target}
\mathcal{L}_{ent}^{t}(\theta_{g}, \theta_{c}) =  &\frac{1}{n_{t}} \sum_{x_{t}\in \chi_{t}} \mathcal{L}_{H}(C(G(x_{t}))).
\end{align}
\subsection{Conditional UADAL}
Domain adversarial learning may fail to capture a class-wise discriminative pattern in the features when the data distributions are complex \cite{long2017conditional}. Therefore, we also experimented with the conditional UADAL (cUADAL) with the discriminative information between classes. Inspired by \cite{long2017conditional}, we replace the input information of the domain discriminator, $D$, from $G(x)$ to $[G(x),C(x)]$ in Eq. (\ref{eqn:loss_d_s})-(\ref{eqn:loss_d_t}). cUADAL has the conditions on the prediction of the extended classifier, $C$, which is including the \textit{unknown} dimension. Therefore, cUADAL provides more discriminative information than UADAL. 

\subsection{Training Procedure}
We introduce how to update the networks, $E$, $G$, $D$, and $C$, based on the proposed losses. 
Before we start the training, we briefly fit the parameters of the posterior inference to obtain $\hat{w}_{x}$. This starting stage trains $E$ and $G$ on the source domain with a few iterations to catch the two modalities of the target entropy values, inspired by the previous work \cite{arazo2019unsupervised}, as below, 
\begin{equation} \label{eqn:e_update}
(\hat{\theta}_{e}, \hat{\theta}_{g}) = \operatorname{argmin}_{\theta_{g}, \theta_{e}} \mathcal{L}_{e}^{s}(\theta_{g}, \theta_{e}).
\end{equation}
After the initial setup of $\hat{w}_x$, we train $D$, $C$, $E$, and $G$, alternatively. This alternation is composed of two phases: (i) learning $D$; and (ii) learning $G$, $C$, and $E$ networks.
The first phase of the alternative iterations starts from learning the domain discriminator, $D$, by the optimization on $\mathcal{L}_{D}$ in Eq. (\ref{eqn:max_D}),
\begin{equation} \label{eqn:d_update}
\hat{\theta}_{d} = \operatorname{argmin}_{\theta_{d}} \mathcal{L}_{D}(\theta_{g}, \theta_{d}).
\end{equation}
Based on the updated $\hat{\theta}_{d}$, the parameter, $\theta_{g}$, is associated with both optimization loss and classification loss. Therefore, for the second phase, the parameters, $\theta_{g}$ and $\theta_{c}$, are optimized as below,
\begin{equation}\label{eqn:g_update}
   (\hat{\theta}_{g}, \hat{\theta}_{c}) ={} \operatorname{argmin}_{\theta_{g}, \theta_{c}} \mathcal{L}_{cls}(\theta_{g}, \theta_{c}) +\mathcal{L}_{ent}^{t}(\theta_{g}, \theta_{c})- \mathcal{L}_{G}(\theta_{g},\hat{\theta}_{d}).
\end{equation}
We also keep updating $\theta_{e}$ by Eq. (\ref{eqn:e_update}), which leads to a better fitting of the posterior inference on the aligned features (empirically shown in Figure \ref{fig:entropy_trhesholding}).
Algorithm \ref{sup-alg:algorithm1} in Appendix \ref{sup-appendix:trainig_algorithm_detail} enumerates the detailed training procedures of UADAL. 
\textbf{Computational Complexity} is asymptotically the same as existing domain adversarial learning methods, which linearly grow by the number of instances $O(n)$. The additional complexity may come from the posterior inference requiring $O(nk)$; $k$ is a constant number of EM iterations (the detailed analysis on the computations in Appendix \ref{sup-appendix:computational_complexity} and \ref{sup-appendix:sampling_bmm}). 

\section{Experiments}
\subsection{Experimental Settings} \label{main_exp_setting}
\textbf{Datasets} We utilized several benchmark datasets. \textbf{Office-31} \cite{office31} consists of three domains: Amazon (A), Webcam (W), and DSLR (D) with 31 classes. \textbf{Office-Home} \cite{officehome} is a more challenging dataset with four different domains: Artistic (A), Clipart (C), Product (P), and Real-World (R), containing 65 classes. \textbf{VisDA} \cite{peng2017visda} is a large-scale dataset from synthetic images to real one, with 12 classes. In terms of the class settings, we follow the experimental protocols by \cite{saito2018open}.

\textbf{Baselines} We compared UADAL and cUADAL with several baselines by choosing the recent works in CDA, OSDA, and Universal DA (UniDA). For CDA, we choose DANN \cite{ganin2015unsupervised} and CDAN \cite{long2017conditional} to show that the existing domain adversarial learning for CDA is not suitable for OSDA. For OSDA, we compare UADAL to OSBP \cite{saito2018open}, STA \cite{liu2019separate}, PGL \cite{luo2020progressive}, ROS \cite{ROS2020ECCV}, and OSLPP \cite{wang2021progressively}. For UniDA, we conduct the OSDA setting in DANCE \cite{DANCE2020NIPS} and DCC \cite{li2021domain}. 

\textbf{Implementation} We commonly utilize three alternative pre-trained backbones: 1) ResNet-50 \cite{He_2016_CVPR}, 2) DenseNet-121 \cite{huang2017densely}, and 3) EfficientNet-B0 \cite{tan2019efficientnet} to show the robustness on 
backbone network choices, while using VGGNet \cite{simonyan2014very} on VisDA for a fair comparison. 
(Details in Appendix \ref{sup-appendix_imple_detail}).

\textbf{Metric} We utilize HOS, which is a harmonic mean of OS$^{*}$, and UNK \cite{ROS2020ECCV}. OS$^{*}$ is a class-wise averaged accuracy over \textit{known} classes, and UNK is an accuracy only for the \textit{unknown} class. 
HOS is suitable for OSDA because it is higher when performing well in both \textit{known} and \textit{unknown} classifications. 
\begin{table*}[t!]
\centering
\resizebox{\textwidth}{!}{%
\begin{tabular}{c|c|cccccc|c||cccccccccccc|c}
\hline \hline 
\multicolumn{2}{l|}{\textbf{Backbone (\#)/}}  & \multicolumn{7}{c||}{\textbf{Office-31}} & \multicolumn{13}{c}{\textbf{Office-Home}} \\ \cline{3-22}
 \multicolumn{2}{r|}{\textbf{Model}}  & A-W & A-D & D-W & W-D & D-A & W-A & \textbf{Avg.} & P-R & P-C & P-A & A-P & A-R & A-C & R-A & R-P & R-C & C-R & C-A & C-P & \textbf{Avg.} \\ \hline
 {\multirow{9}{*}{\rotatebox[origin=c]{90}{\textbf{EfficientNet-B0 (5.3M)}}}} & DANN & 63.2 & 72.7 & 92.6 & 94.8 & 63.7 & 57.2 & 74.0$\pm$0.3 & 35.7 & 16.5 & 18.2 & 34.1 & 46.3 & 22.9 & 40.7 & 47.8 & 28.2 & 12.4 & 7.5 & 13.4 & 27.0$\pm$0.3 \\
 & CDAN & 65.5 & 73.6 & 92.4 & 94.6 & 64.8 & 57.9 & 74.8$\pm$0.2 & 37.9 & 18.1 & 20.4 & 35.6 & 47.0 & 24.6 & 44.1 & 49.8 & 30.1 & 13.5 & 8.9 & 15.0 & 28.8$\pm$0.6 \\
 & STA& 58.3 & 62.2 & 81.6 & 79.6 & {69.8} & 67.4 & 69.8$\pm$1.2 & 59.4 & 43.6 & 51.9 & 53.8 & 60.6 & 49.5 & 58.8 & 53.5 & 49.9 & 53.4 & 49.5 & 49.4 & 52.8$\pm$0.2 \\
 & OSBP & 82.9 & 87.0 & 33.8 & 96.7 & 27.3 & {69.9} & 66.3$\pm$2.1 & 65.0 & {46.0} & {58.6} & 64.2 & 71.0 & \underline{54.0} & 58.3 & 62.5 & 50.3 & \underline{63.7} & 50.7 & 55.6 & 58.3$\pm$1.6 \\
 & ROS & 69.7 & 80.1 & 94.7 & 99.6 & \underline{73.0} & 59.2 & 79.4$\pm$0.3 & 66.9 & 44.9 & 53.7 & 62.5 & 69.5 & 50.0 & 62.0 & 67.0 & {52.0} & 61.2 & 50.5 & 54.7 & 57.9$\pm$0.1 \\
 & DANCE & 68.1 & 68.8 & 91.3 & 85.0 & 68.5 & 63.3 & 74.2$\pm$4.0 & 17.2 & {47.5} & 7.2 & 26.6 & 19.6 & 36.6 & 2.2 & 19.8 & 10.9 & 6.4 & 4.3 & 19.0 & 18.1$\pm$2.9 \\
 & DCC & \underline{87.2} & 69.1 & 89.4 & 94.4 & 63.5 & \textbf{76.1} & 79.9$\pm$2.9 & {72.2} & 41.0 & 56.5 & \textbf{66.4} & \textbf{75.7} & 52.8 & 55.9 & {71.5} & 49.9 & 60.4 & 48.1 & {60.8} & 59.3$\pm$1.5 \\ \cline{2-22}
 & \textbf{UADAL} & \textbf{87.5} & \underline{88.3} & \textbf{97.4} & \textbf{96.9} & \textbf{74.1} & 68.9 & \underline{85.5$\pm$0.5} & \textbf{75.0} & \underline{50.0} & \underline{62.9} &  \textbf{66.4} & \underline{74.1} &\underline{52.7} &  \textbf{71.5} & \underline{72.6} &  \textbf{53.6} &  \textbf{65.3} & \underline{60.8} &  \textbf{63.7} &  \underline{64.1$\pm$0.1} \\
 & \textbf{cUADAL} & 86.5 & \textbf{89.1} & \underline{97.3} & \underline{98.0} & 72.5 & \underline{71.0} & \textbf{85.7$\pm$0.7} & \underline{74.7} & \textbf{ 54.4} &  \textbf{64.2} & \underline{66.3} & 73.9 & 50.8 & \underline{71.4} &  \textbf{73.0} & \underline{52.4} &  \textbf{65.3} &  \textbf{61.0} & \underline{63.3} & \textbf{64.2$\pm$0.1} \\\hline \hline

 {\multirow{9}{*}{\rotatebox[origin=c]{90}{\textbf{DenseNet-121 (7.9M)}}}}&  DANN & 71.9 & 72.0 & 90.2 & 85.3 & 73.8 & 72.3 & 77.6$\pm$0.5 & 68.8 & 35.4 & 48.7 & 62.6 & 71.9 & 45.3 & 62.8 & 68.7 & 45.9 & 62.2 & 47.0 & 54.7 & 56.2$\pm$0.3 \\
 & CDAN & 69.5 & 69.8 & 86.8 & 84.5 & 73.8 & 72.5 & 76.2$\pm$0.2 & 68.9 & 39.2 & 51.9 & 62.6 & 71.8 & 47.1 & 63.6 & 68.0 & 48.7 & 62.8 & 49.3 & 55.2 & 57.4$\pm$0.3 \\
 & STA & 77.0 & 68.6 & 84.0 & 77.2 & 76.6 & 75.1 & 76.4$\pm$1.5 & 65.6 & 46.1 & 58.4 & 55.8 & 64.3 & 50.4 & 62.6 & 58.6 & 51.1 & 61.0 & 56.0 & 55.9 & 57.1$\pm$0.1 \\
 & OSBP & 81.9 & \underline{83.0} & 88.9 & {96.6} & 73.1 & 74.9 & 83.1$\pm$2.2 & 71.9 & 46.0 & {60.3} & 67.1 & 72.3 & 54.5 & 65.9 & 71.7 & 53.7 & 66.8 & \underline{59.3} & 64.1 & 62.8$\pm$0.1 \\
 & ROS & 67.0 & 67.8 & \textbf{97.4} & \underline{99.4} & 77.1 & 71.8 & 80.1$\pm$1.3 & 73.0 & \underline{49.6} & 59.2 & 67.8 & {75.5} & 52.8 & 66.4 & {74.6} & 54.3 & 64.8 & 53.0 & 57.8 & 62.4$\pm$0.1 \\
 & DANCE & 69.9 & 67.8 & 84.0 & 82.8 & \textbf{79.9} & \textbf{81.1} & 77.6$\pm$0.3 & 51.8 & \textbf{51.0} & 59.7 & 63.9 & 58.2 & \textbf{58.2} & 43.4 & 48.9 & {55.0} & 41.3 & 54.6 & 60.6 & 53.9$\pm$0.5 \\
 & DCC & {83.9} & 80.8 & 88.4 & 93.1 & \underline{79.7} & \underline{80.4} & 84.4$\pm$1.3 & {75.1} & 46.6 & 58.0 & \textbf{70.8} & \textbf{78.6} & {56.6} & 63.4 & {75.5} & {55.8} & \textbf{71.3} & 55.0 & 63.3 & {64.2$\pm$0.2} \\ \cline{2-22}
&\textbf{UADAL} & \textbf{86.0} & 82.3 & \underline{96.7} & 99.2 & {77.9} & 74.2 & \underline{86.0$\pm$0.6} & \textbf{75.7} & 45.5 & \underline{61.5} & \underline{70.0} & \underline{76.9} & 57.3 &\underline{71.5} & \underline{76.1} & \textbf{60.4} & \underline{70.0} & \textbf{60.1} & \underline{67.2} & \underline{66.0 $\pm$0.2} \\
&\textbf{cUADAL} & \underline{85.1} & \textbf{83.6} &{96.4} & \textbf{99.6} & 77.5 & 75.9 & \textbf{86.4$\pm$0.6} &\underline{75.6} & 48.9 & \textbf{61.7} & \underline{70.0} & 76.7 & \underline{57.8} & \textbf{71.9} & \textbf{76.7} & \underline{59.1} & 69.6 & \textbf{60.1} & \textbf{67.5} & \textbf{66.3$\pm$0.3}\\\hline \hline

 {\multirow{9}{*}{\rotatebox[origin=c]{90}{\textbf{ResNet-50 (25.5M)}}}} & DANN & 68.1 & 71.5 & 86.7 & 82.5 & 73.7 & 72.6 & 75.9$\pm$0.5 & 69.8 & 44.6 & 56.3 & 65.2 & 71.0 & 51.2 & 65.4 & 68.4 & 50.9 & 66.7 & 57.6 & 60.9 & 60.7$\pm$0.2 \\
 & CDAN & 64.9 & 66.8 & 84.3 & 80.5 & 72.7 & 71.0 & 73.4$\pm$1.3 & 69.7 & 47.2 & 58.6 & 65.1 & 70.7 & 52.9 & 66.0 & 67.6 & 52.7 & 67.1 & 58.2 & 61.7 & 61.4$\pm$0.3 \\
 & STA$^{*}$ & 75.9 & 75.0 & 69.8 & 75.2 & 73.2 & 66.1 & 72.5$\pm$0.8 & 69.5 & 53.2 & 61.9 & 54.0 & 68.3 & 55.8 & 67.1 & 64.5 & 54.5 & 66.8 & 57.4 & 60.4 & 61.1$\pm$0.3 \\
 & OSBP$^{*}$& 82.7 &82.4 &97.2 &91.1 & 75.1 & 73.7 & 83.7$\pm$0.4 & 73.9 & 53.2 & {63.2} & 65.2 & 72.9 & 55.1 & 66.7 & 72.3 & 54.5 & {70.6} & \underline{64.3} & 64.7 & 64.7$\pm$0.2 \\
 & PGL$^{*}$ &74.6   &72.8 & 76.5  & 72.2 & 69.5  & 70.1  & 72.6$\pm$1.5  &   41.6&46.6   &  47.2& 45.6  & 55.8  & 29.3 & 11.4 & 52.5  &  0.0 & 45.6  &  10.0 &  36.8&  35.2 \\
 & ROS$^{*}$ & 82.1 & 82.4 & {96.0} & \textbf{99.7} & 77.9 & 77.2 & 85.9$\pm$0.2 & 74.4 & {56.3} & 60.6 & 69.3 & 76.5 & 60.1&{68.8} & 75.7 & {60.4} & 68.6 & 58.9 & 65.2 & {66.2$\pm$0.3} \\
 & DANCE & 66.9 & 70.7 & 80.0 & 84.8 &65.8 & 70.2 & 73.1$\pm$1.0 & 41.2 & 55.7 & 54.2 & 49.8 & 39.4 & 53.1 & 27.5 & 44.0 & 48.3 & 30.2 & 40.9 & 45.9 & 44.2$\pm$0.6 \\
 & DCC$^{*}$ &87.1 & 85.5 & 91.2 & 87.1 & \textbf{85.5} & \textbf{84.4} & 86.8 & 64.0 & 52.8 & 59.5 & 67.4 & \textbf{80.6} & 52.9 & 56.0 & 62.7 & \textbf{76.9} & 67.0 & 49.8 & 66.6 & 64.2 \\
  & OSLPP$^{*}$ & {89.0} & \textbf{91.5} & 92.3 & 93.6 & 79.3 & \underline{78.7} & {87.4} & 74.0 & \textbf{59.3} & \textbf{63.6} & \textbf{72.8} & 74.3 & 61.0 & 67.2 & {74.4} & {59.0} & 70.4 & {60.9} & {66.9} & {67.0} \\ \cline{2-22}
 & \textbf{UADAL} & \underline{89.1} &86.0 &\underline{97.8}&\underline{99.5} &79.7 & 76.5 & \underline{88.1$\pm$0.2}& \textbf{76.9} & \underline{56.6} & \underline{63.0} & 70.8 & {77.4} & \underline{63.2} & \underline{72.1} & \textbf{76.8} & \underline{60.6} & \textbf{73.4} & {64.2} &\textbf{69.5} & \textbf{68.7$\pm$0.2}  \\
 & \textbf{cUADAL} & \textbf{90.1} & \underline{87.9} & \textbf{98.2} &99.4 & \underline{80.5} & 75.1 & \textbf{88.5$\pm$0.3} &\underline{76.8} & 54.6 & 62.9 & \underline{71.6} & \underline{77.5} & \textbf{63.6} & \textbf{72.6} & \underline{76.7} & 59.9 & \underline{72.6} & \textbf{65.0} & \underline{68.3} & \underline{68.5$\pm$0.1} \\ \hline \hline
\end{tabular}%
}
\caption{\small{HOS score (\%) on Office-31 \& Office-Home using EfficientNet-B0 (5.3M), DenseNet-121 (7.9M), and ResNet-50 (25.5M) as backbone network, where (\#) represents the number of the parameters.
A-W in a column means that A is the source domain and W is the target domain. \textbf{Avg.} is the averaged HOS over all tasks in each dataset.
(bold: best performer, underline: second-best performer, $^{*}$: officially reported performances.) The detailed experimental results including the other metrics such as OS$^{*}$ and UNK are in Appendix \ref{sup-appendix:full_table}} } \label{tab:hos_office}
\end{table*}

  \begin{minipage}[t!]{\columnwidth}
  \begin{minipage}[b]{0.70\columnwidth} 
\scalebox{0.63}{
\begin{tabular}{|c|ccc|ccc|ccc|ccc|}
 \hline
   \multirow{2}{*}{\textbf{Method}}    & \multicolumn{3}{|c|}{\textbf{EfficientNet-B0}} & \multicolumn{3}{c|}{\textbf{DenseNet-121}} & \multicolumn{3}{c|}{\textbf{ResNet-50}} &\multicolumn{3}{c|}{\textbf{VGGNet}} \\   \cline{2-13}
       & OS*  & UNK  & \textbf{HOS}& OS*  & UNK& \textbf{HOS} & OS*  & UNK& \textbf{HOS}& OS*  & UNK& \textbf{HOS} \\ \hline
 STA &  49.3 & 56.4 & 52.5 & 53.1 & 76.7 & 62.7& 56.9 & 75.8 & 65.0&63.9$^{*}$ &84.2$^{*}$ &72.7$^{*}$ \\
 OSBP &48.8 & 70.4 & 57.6 & 48.7 & 65.7 & 55.9 & 50.0 & 77.2 & 60.7&59.2$^{*}$ &85.1$^{*}$ &69.8$^{*}$ \\
 PGL & 56.9 & 26.1 & 35.8 & - & - & -  & 70.3 & 33.4 & 45.3&82.8$^{*}$ &68.1$^{*}$ &74.7$^{*}$ \\
 ROS & 36.7 & 72.0 & 48.7 & 42.7 & 81.1 & 55.9& 45.8 & 64.8 & 53.7 &-&-&- \\
 DANCE& 38.9 & 63.2 & 48.1 & 59.8 & 67.3 & 62.3 & 61.3 & 72.9 & 66.5&-&-&- \\
 DCC &  38.3 & 54.2 & 44.8& 16.7 & 76.9 & 27.2& 13.4 & 88.0 & 23.3&68.0$^{*}$ &73.6$^{*}$ &70.7$^{*}$ \\
 OSLPP & - & - & - & - & - & -  & - & - & - &-&-&-\\ \hline
\textbf{UADAL} & 47.0 & 76.8 & \underline{58.3} & 56.2 & 81.5 & \underline{66.5} & 58.0 & 86.2 & \underline{69.4} &63.1&93.3&\underline{75.3}\\ 
 \textbf{cUADAL} & 47.2 & 76.5 & \textbf{58.4} & 57.9 & 84.1 & \textbf{68.6} & 58.5 & 87.6 & \textbf{70.1}&64.3&92.6&\textbf{75.9}\\ \hline 
\end{tabular}
}
\captionof{table}{\small{Results on VisDA dataset. (bold/underline/$^{*}$: Refer to the Table \ref{tab:hos_office})}}
\label{tab:res:visda}
  \end{minipage}
  \hfill
  \begin{minipage}[b]{0.28\columnwidth}
    \centering
\scalebox{0.71}{
\begin{tabular}{|c|c|c|}
\hline
\multicolumn{3}{|c|}{\textbf{Office-31 with ResNet-50}} \\
 \hline
\textbf{Weight} &  \textbf{Core} &  \textbf{HOS Avg.} \\ \hline
\multirow{2}{*}{STA} & STA    & 72.5\small{$\pm$0.8}          \\
                     &  \textbf{UADAL}  &\underline{85.5\small{$\pm$0.9}}         \\ \hline
\multirow{2}{*}{ROS} &  ROS    & 85.9\small{$\pm$0.2}          \\
                     &  \textbf{UADAL}  & \underline{86.7\small{$\pm$0.2}}          \\ \hline
\textbf{UADAL}                & \textbf{UADAL} & \textbf{88.1\small{$\pm$0.2}}   \\ \hline
        \end{tabular} 
} 
      \captionof{table}{\small{Ablation Studies for Different Weighting / Core Schemes on STA and ROS.}}
\label{tab:different_weight}
    \end{minipage}
  \end{minipage}

\subsection{Experimental Results}

\textbf{Quantitative Analysis}
Table \ref{tab:hos_office} reports the quantitative performances of UADAL, cUADAL, and baselines applied to Office-31 and Office-Home, with three backbone networks. 
UADAL and cUADAL show statistically significant improvements compared to \textit{all} existing baselines in both Office-31 and Office-Home. 
Moreover, there is a large improvement by UADAL from baselines in EfficientNet-B0 and DenseNet-121, even with fewer parameters. 
In terms of the table on ResNet-50, for a fair comparison, UADAL outperforms the most competitive model, OSLPP (HOS=67.0). Compared to the domain adversarial learning methods, such as STA and OSBP, UADAL significantly outperforms in all cases. UADAL also performs better than other self-supervised learning methods such as ROS and DCC.
These performance increments are commonly observed in all three backbone structures and in two benchmark datasets, which represents the effectiveness of the proposed model.

Table \ref{tab:res:visda} shows the experimental results on VisDA, which is known to have a large domain shift from a synthetic to a real domain. 
It demonstrates that UADAL and cUADAL outperform the other baselines \footnote{OSLPP is infeasible to a large dataset, such as VisDA dataset (Detailed discussions in Appendix \ref{sup-appendix_oslpp}).} under the significant domain shift. 
Moreover, the baselines with the feature alignments by the domain adversarial learning, such as STA, OSBP, and PGL, outperform the self-supervised learning models, ROS, DANCE, and DCC, in most cases. 
Therefore, these results support that OSDA essentially needs the feature alignments when there are significant domain shifts. 
The detailed discussions about the low accuracies of some baselines are in Appendix \ref{sup-appendix:low_accuracy_results}.
In addition, we observe that cUADAL mostly performs better than UADAL, which shows that cUADAL provides more discriminative information both on \textit{known} and \textit{unknown} by the prediction of the classifier $C$.

\begin{figure}[h!]
\vspace{-1em}
\begin{minipage}[t]{0.51\textwidth}%
\vspace{0pt}
\begin{subfigure}[t]{0.33\textwidth}
    \includegraphics[width=\textwidth]{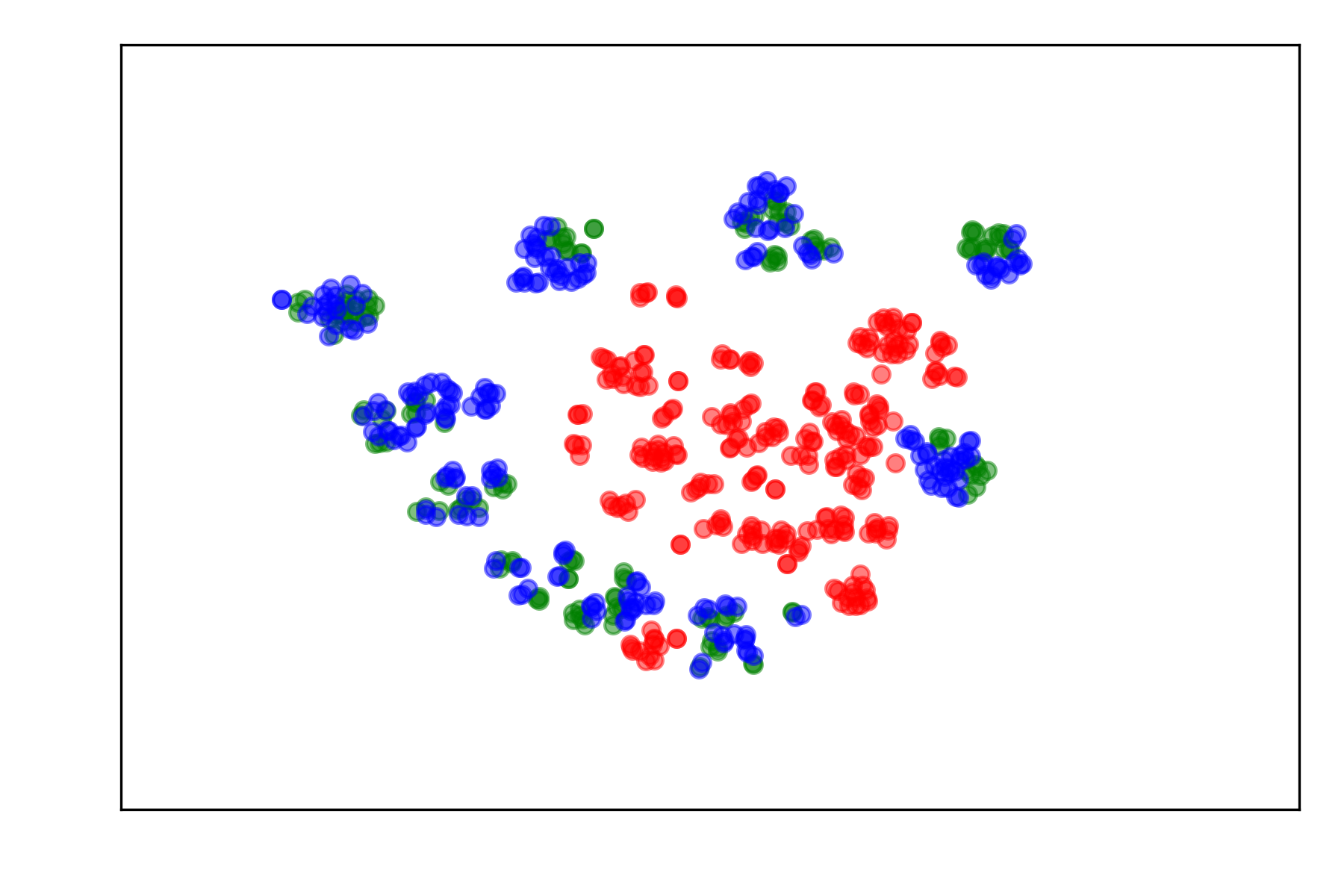}
    \caption{\small{DANN}}
    \label{fig:tsne_dann}
\end{subfigure} \hspace{-0.2em}%
\begin{subfigure}[t]{0.33\textwidth}
    \includegraphics[width=\textwidth]{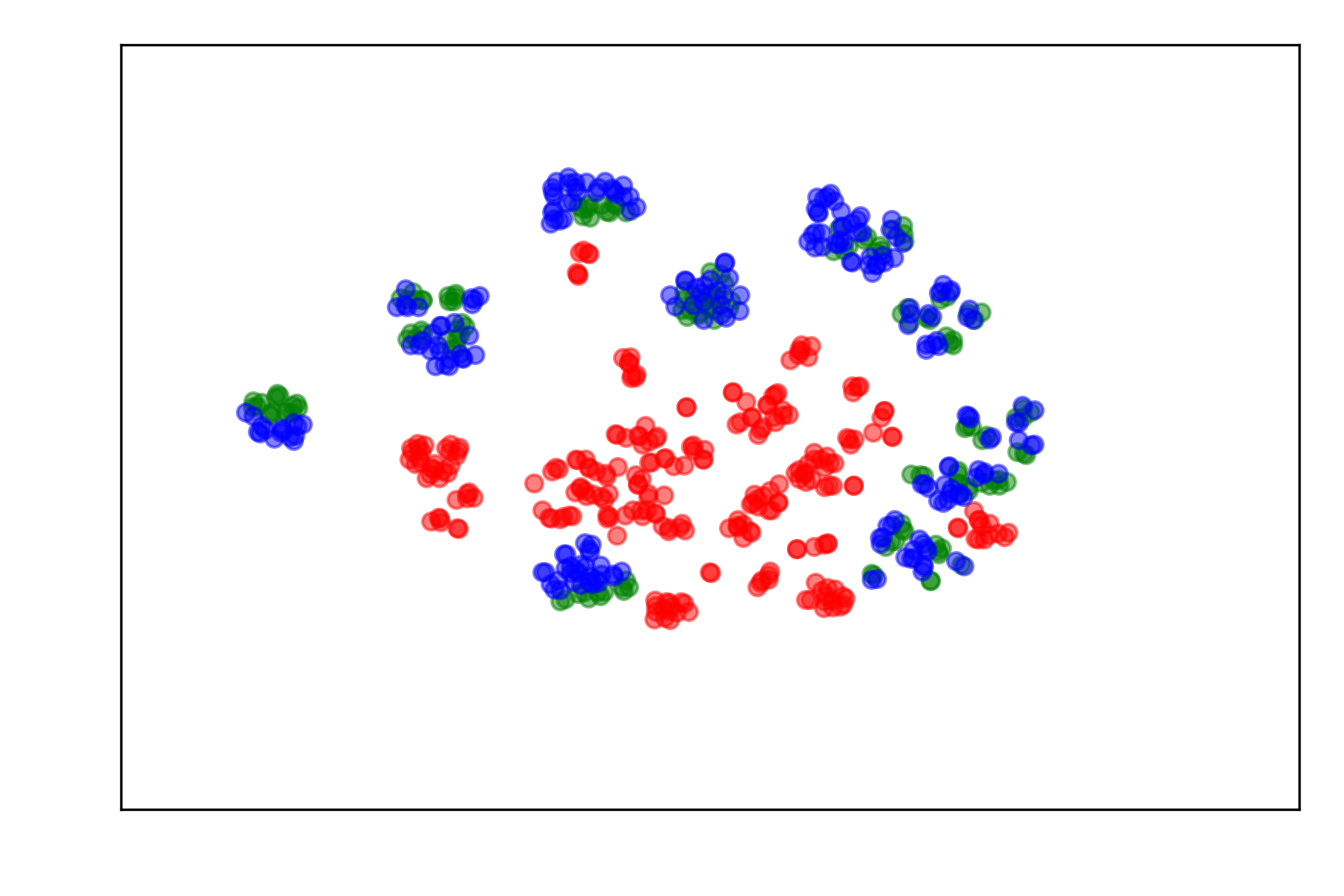}
    \caption{\small{STA}}
    \label{fig:tsne_sta}
\end{subfigure} \hspace{-0.2em}%
\begin{subfigure}[t]{0.33\textwidth}
    \includegraphics[width=\textwidth]{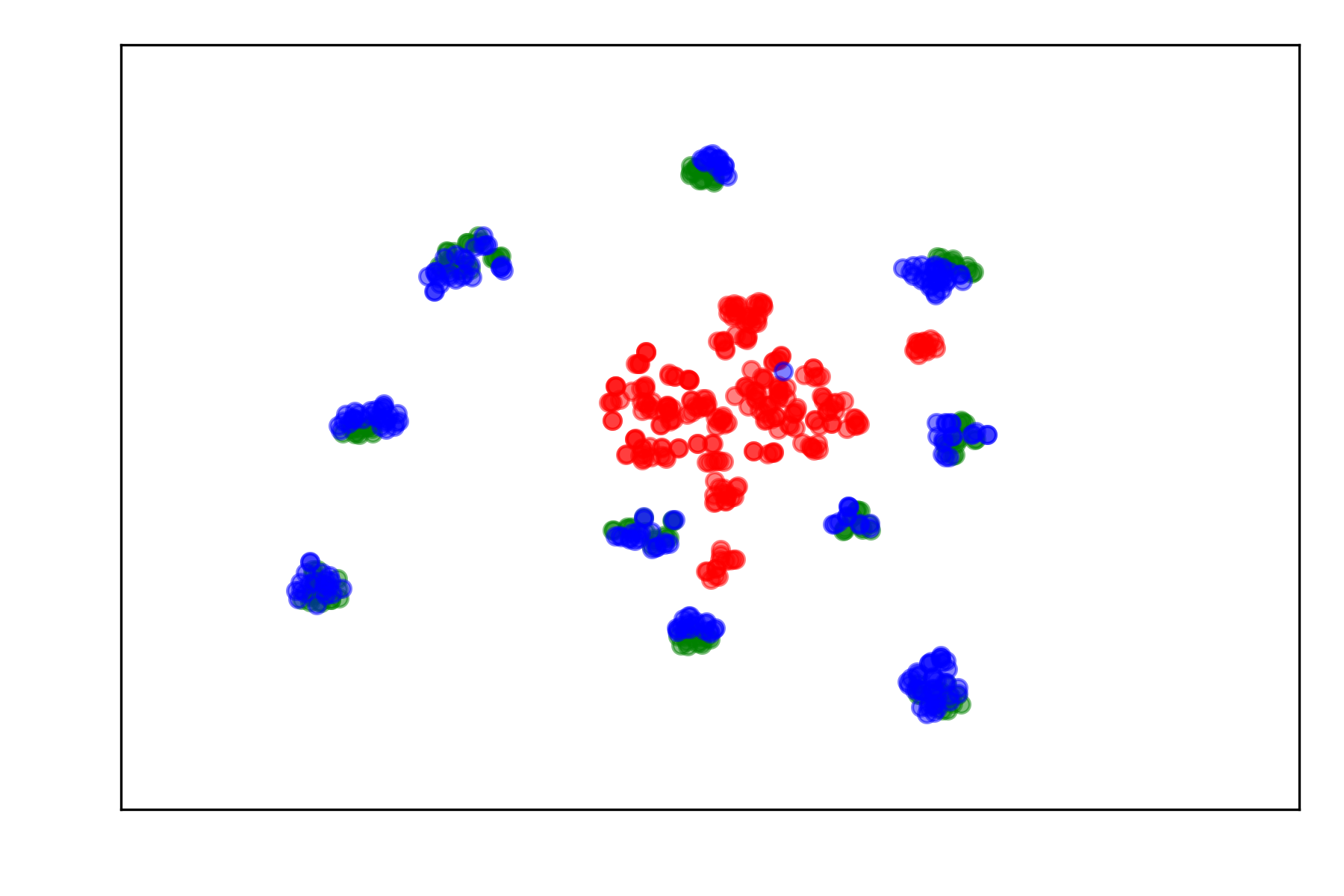}
    \caption{\small{cUADAL}}
    \label{fig:tsne_uadal}
\end{subfigure} 
\caption{\small{t-SNE Visualization on D $\rightarrow$ W of Office-31. (Blue/Red/Green: Target-Known/-Unknown/Source)}} \label{fig:tsne}
\end{minipage}%
\hfill
\begin{minipage}[t]{0.23\textwidth}%
\vspace{0pt}
\includegraphics[width=\linewidth]{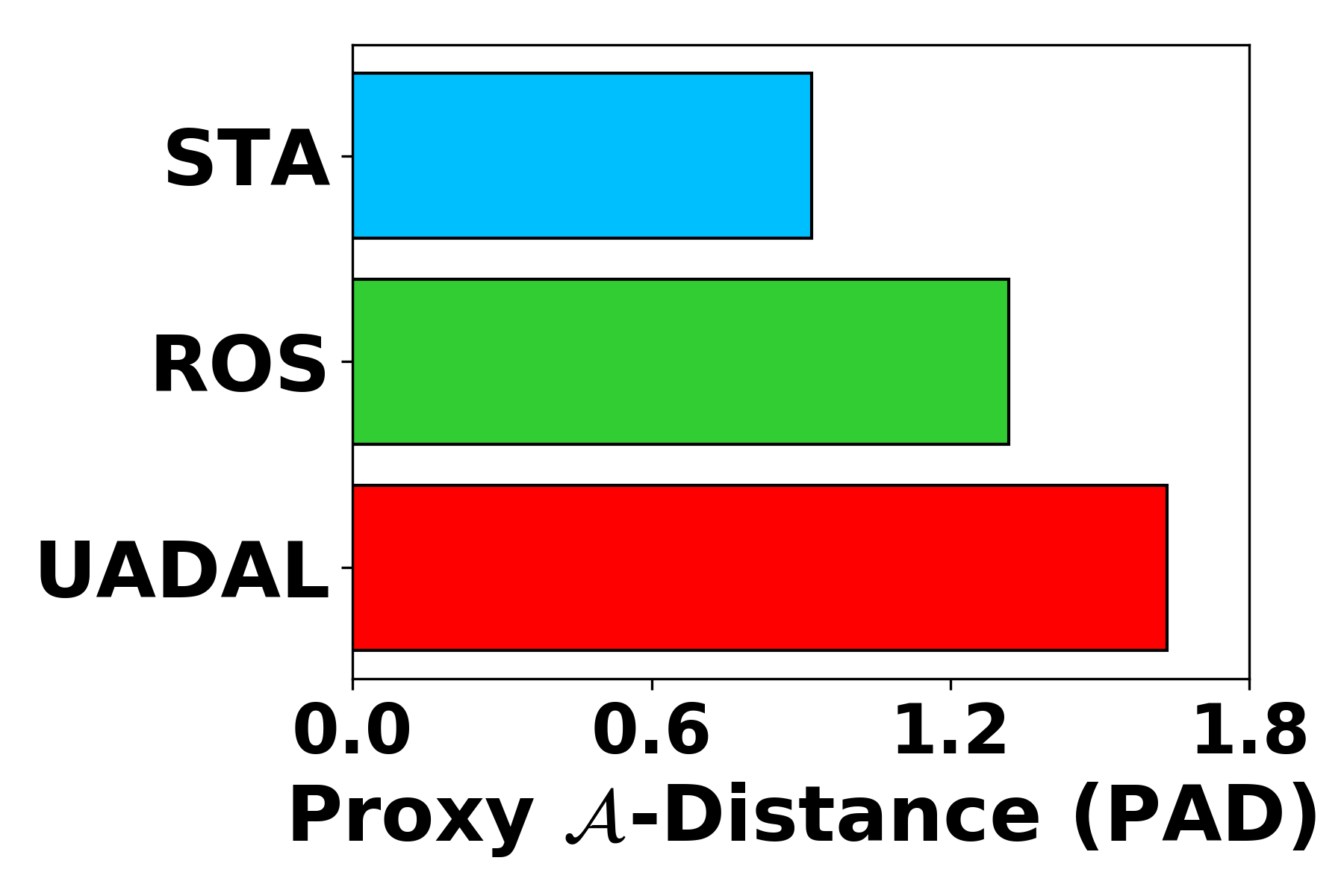}
\caption{\small{PAD value on \textit{tk}/\textit{tu} feature distributions.}}      \label{fig:exp_pad_value}
\end{minipage}
\hfill
\begin{minipage}[t]{0.23\textwidth}%
\vspace{0pt}
\includegraphics[width=\textwidth]{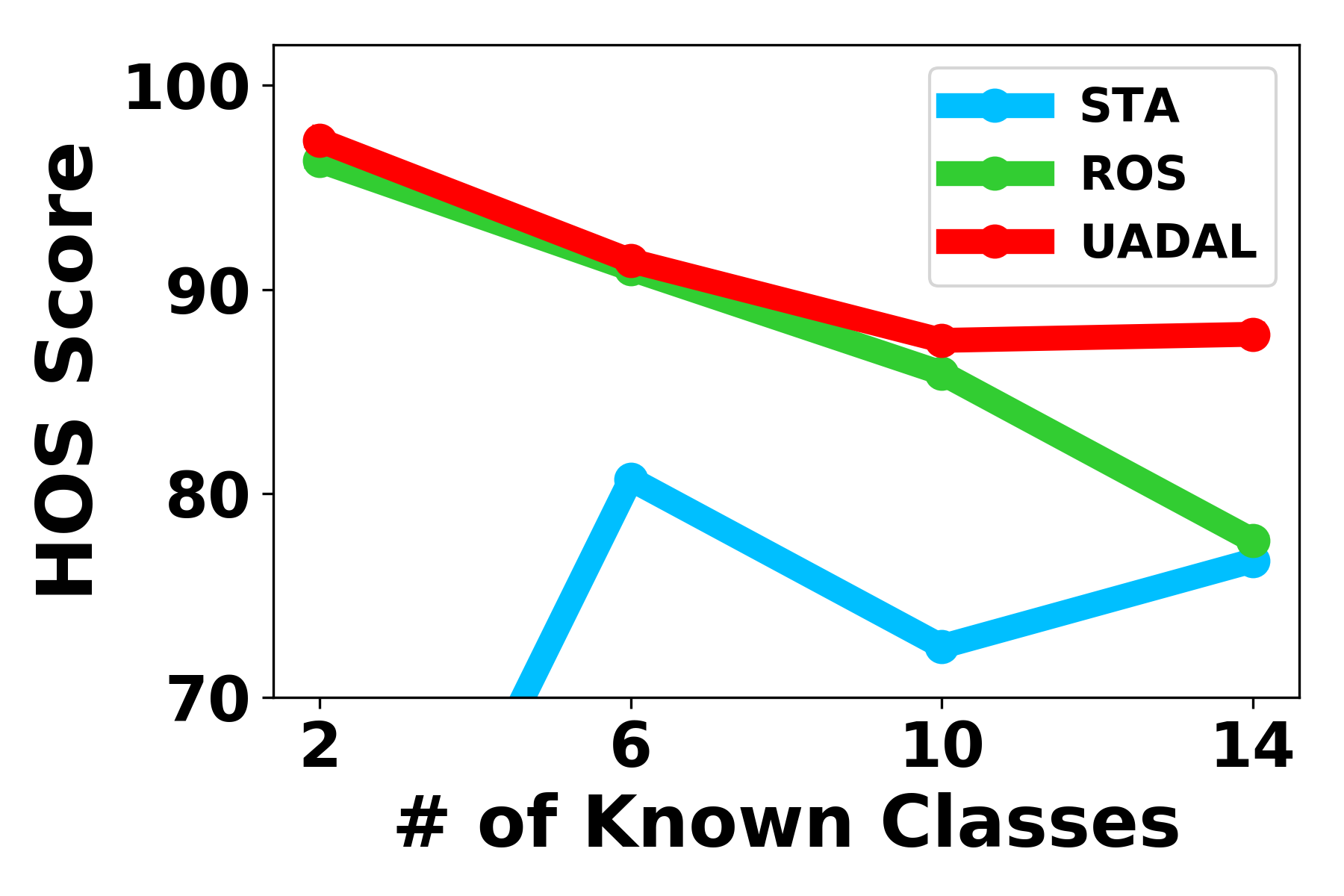}%
\caption{\small{Openness on Office-31 (ResNet-50).}}         \label{fig:exp_openness}  
\end{minipage}

\end{figure}

\textbf{Qualitative Analysis}
Figure \ref{fig:tsne} represents the t-SNE visualizations \cite{van2008visualizing} of the learned features by ResNet-50 (the details in Appendix \ref{sup-appendix_tsne}). We observe that the target-unknown (red) features from the baselines are not segregated from the source (green) and the target-known (blue) features. On the contrary, UADAL aligns the features from the source and the target-known instances accurately with clear \textit{segregation} of the target-unknown features. 
In order to investigate the learned feature distributions quantitatively, 
Figure \ref{fig:exp_pad_value} provides Proxy $\mathcal{A}$-Distance (PAD) between the feature distributions from STA, ROS, and UADAL. 
PAD is an empirical measure of distance between domain distributions \cite{ganin2016domain} (Details in Appendix \ref{sup-appendix_proxy}), and a higher PAD means clear discrimination between two distributions. Thus, we calculate PAD between the target-known ($tk$) and the target-unknown ($tu$) feature distributions from the feature extractor, $G$. We found that PAD from UADAL is much higher than STA and ROS. 
It shows that UADAL explicitly segregates the unknown features.

\begin{figure*}[h]
\centering
\begin{subfigure}[h]{0.38\textwidth}
    \includegraphics[width=\textwidth]{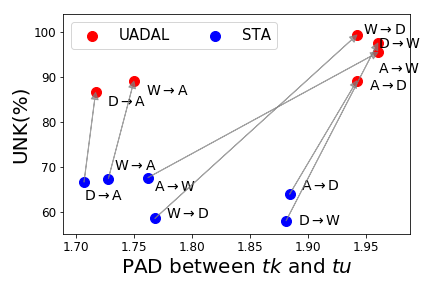}
    \label{fig:abl_sample_31}
\end{subfigure}  \vspace{-1em}
\begin{subfigure}[h]{0.38\textwidth}
    \includegraphics[width=\textwidth]{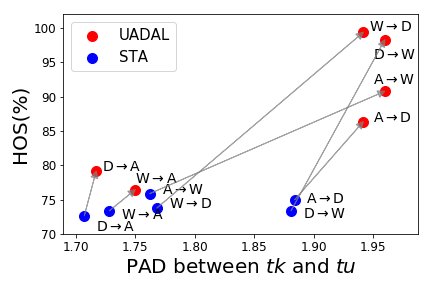}
    \label{fig:abl_sample_home}
\end{subfigure} \vspace{-1em}
\caption{\small{Correlation analysis between the PAD value and each of UNK (left) and HOS (right), in Office-31 dataset (with ResNet-50). The arrow represents the same task, such as A$\rightarrow$W, for UADAL and STA.}}
\label{main_fig:sta_ablation}
\end{figure*}
In order to provide more insights and explanations compared to STA, we provide an analysis of the correlation between the evaluation metric and the distance measure of the feature distributions. 
We utilize UNK, and HOS as evaluation metrics, and PAD between $tk$ and $tu$ as the distance measure. 
Figure \ref{main_fig:sta_ablation} shows the scatter plots of (UNK \& PAD) and (HOS \& PAD). 
The grey arrows mean the corresponding tasks by UADAL and STA. 
As shown in the figure, HOS and UNK have a positive correlation with the distance between the target-unknown ($tu$) and the target-known ($tk$). 
In simple words, better segregation enables better HOS and UNK. Specifically, from STA to UADAL on the same task, the distances (PAD on $tk$/$tu$) are increased, and the corresponding performances, UNK and HOS, are also increased. 
It means that UADAL effectively segregates the target feature distributions, and then leads to better performance for OSDA. 
In other words, the proposed unknown-aware feature alignment is important to solve the OSDA problem. 
From this analysis, our explicit segregation of $tu$ is a key difference from the previous methods.
Therefore, we claimed that this explicit segregation of UADAL is the essential part of OSDA, and leads to better performances.  

\textbf{Robust on Openness} 
We conduct the openness experiment to show the robustness of the varying number of the known classes. 
Figure \ref{fig:exp_openness} represents that UADAL performs better than the baselines over all cases. 
Our \textit{open-set recognition} does not need to set any thresholds since we utilize the loss information of the target domain, holistically. Therefore, UADAL is robust to the degree of openness.

\subsection{Ablation Studies} \label{sec_ablation}
\textbf{Effectiveness of Unknown-Aware Feature Alignment} We conducted ablation studies to investigate whether the unknown-aware feature alignment is essential for the OSDA problem. 
\begin{wrapfigure}{r}{0.25\linewidth}\vspace{-0.5em}
        \includegraphics[width=0.25\textwidth]{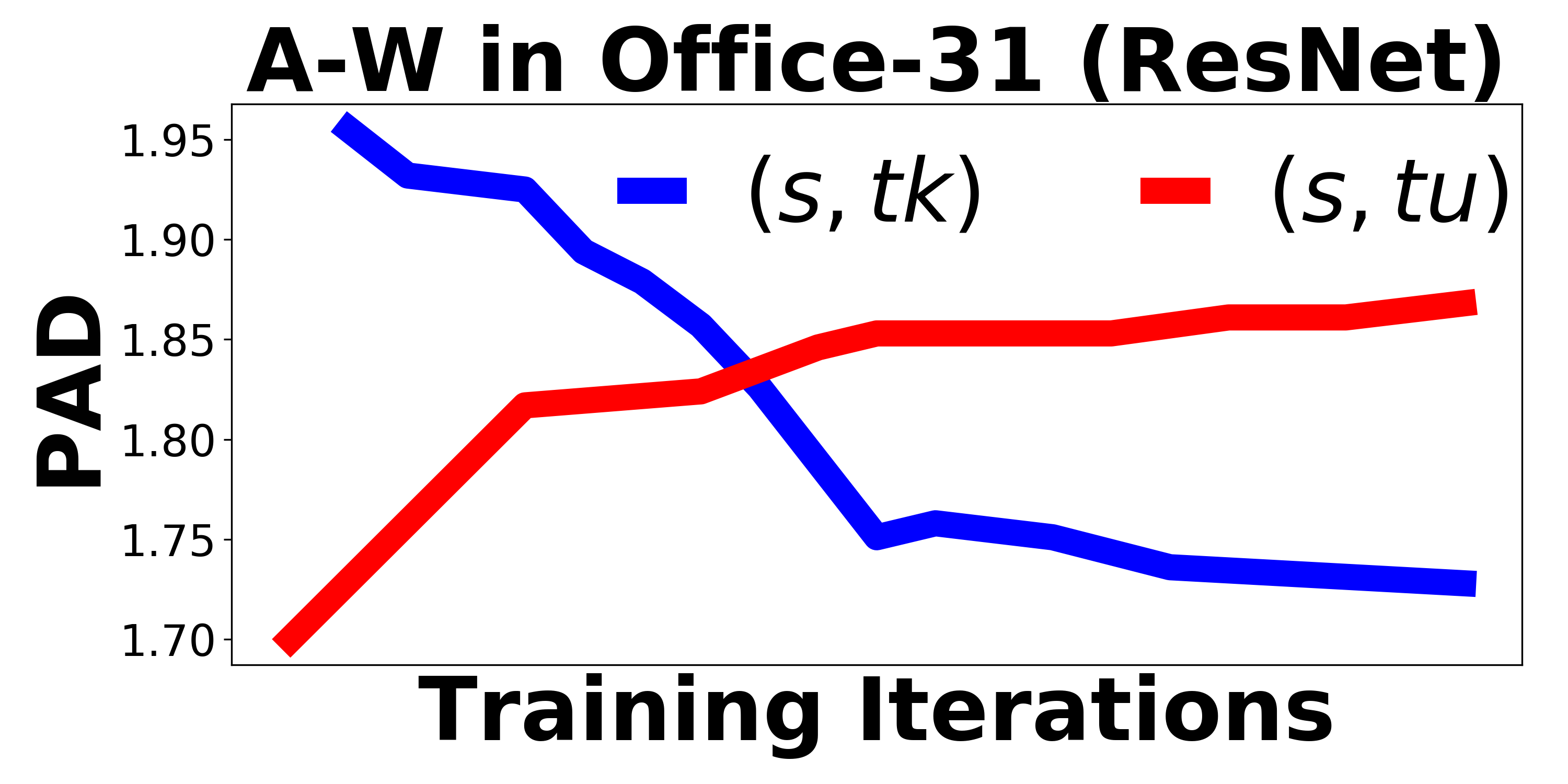}
    \caption{\small{Convergence of Sequential Optimization.}}   \label{converge_pad}
\vspace{-1em}
\end{wrapfigure}
Table \ref{tab:different_weight} represents the experimental results where the feature alignment procedure of STA or ROS is replaced by UADAL, maintaining their weighting schemes, in Office-31 with ResNet-50. 
In terms of each weighting, the UADAL core leads to better performances in both STA and ROS cases. 
Surprisingly, when we replace STA's known-only matching with the unknown-aware feature alignment of UADAL, the performance increases dramatically. 
It means that the implicit separation of STA at the classification level is limited to segregating the target-unknown features. Therefore, explicit \textit{segregation} by UADAL is an essential part to solve the OSDA problem. 
Moreover, Figure \ref{converge_pad} shows the empirical convergence of the proposed sequential optimization as PAD value between $p_{s}$ and $p_{tk}$ and between $p_{s}$ and $p_{tu}$ over training, i.e., $p_{s}\approx p_{tk}$ and $p_{tu}\leftrightarrow \{ p_{tk}, p_{s} \}$.
\vspace{-0.6em}

\begin{figure}[h]
\begin{minipage}[h]{0.62\textwidth}%
\begin{minipage}[h]{0.52\textwidth}%
\includegraphics[width=\textwidth]{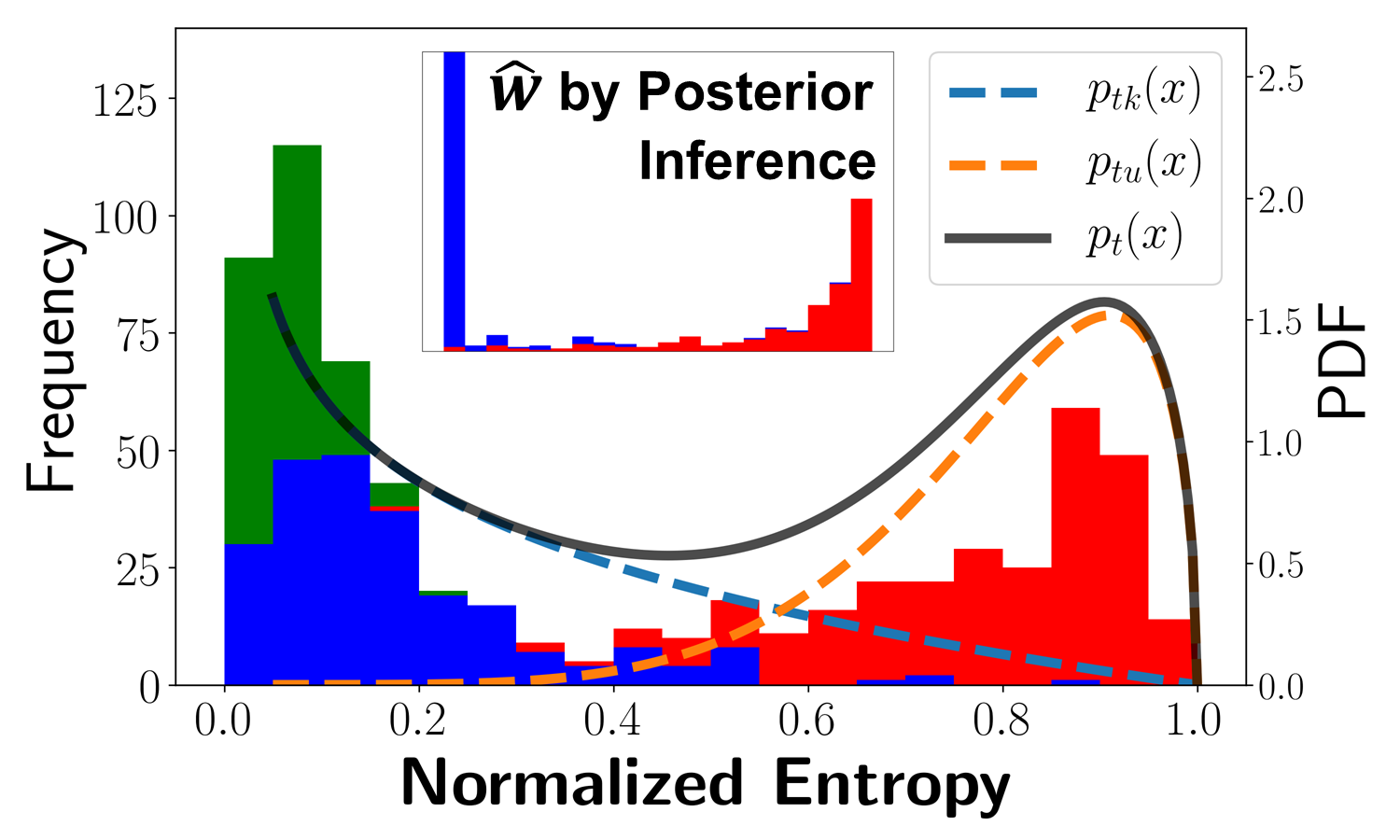}%
\subcaption{\small{Early stage of training}}     \label{fig:histo_warm-up}
\end{minipage}%
\hfill
\begin{minipage}[h]{0.48\textwidth}%
\includegraphics[width=\linewidth]{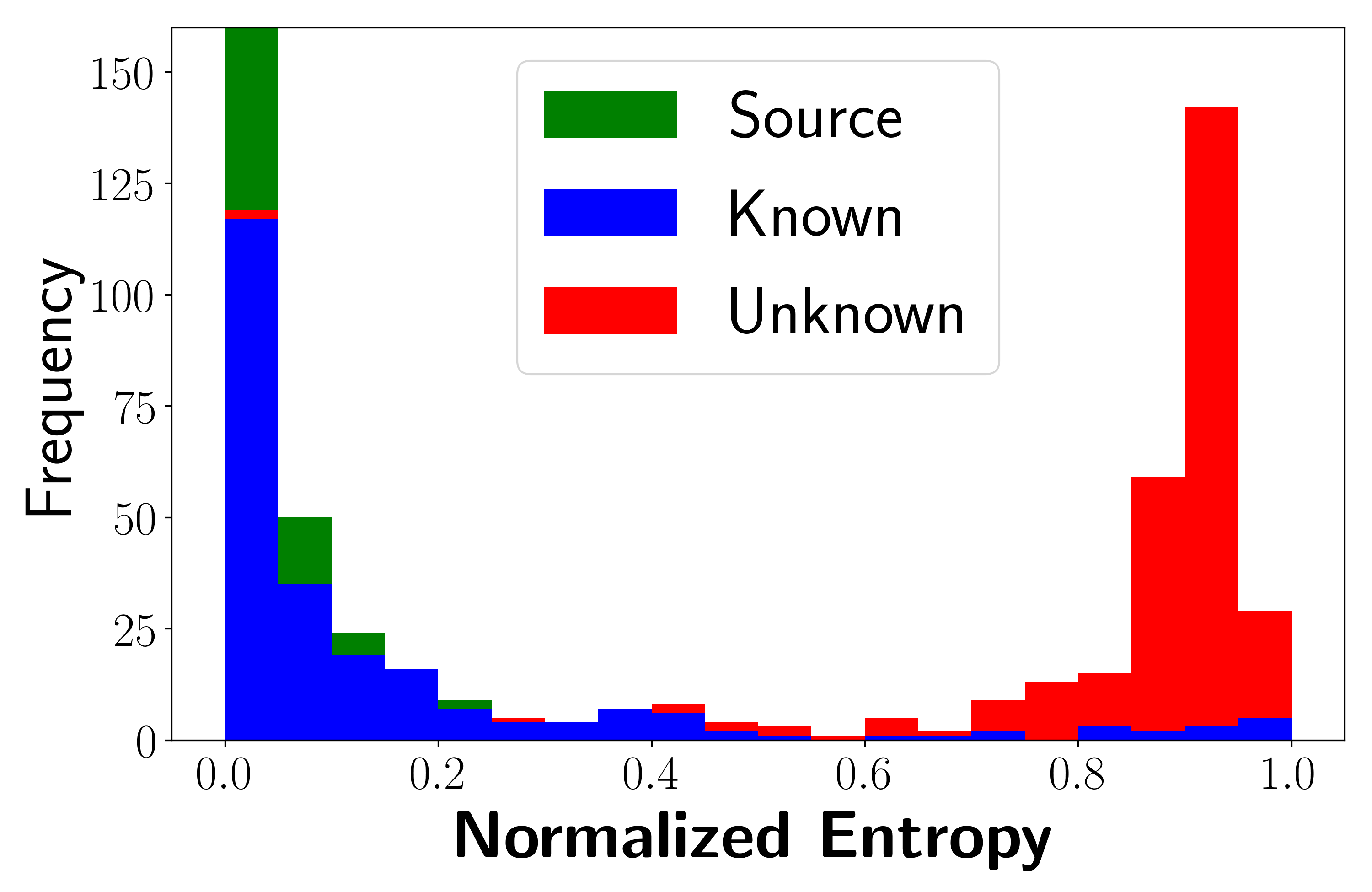}
\subcaption{\small{End stage of training}}     \label{fig:histo_main}
\end{minipage}
\caption{\small{(a) The histogram of the source and target entropy with the fitted posterior inference model, at the Early / End stage of training, where the subfigure in (a) shows the histogram of $\hat{w}_{x}$ at this stage.}}
\end{minipage}%
\hfill
\begin{minipage}[h]{0.36\textwidth}%
\begin{minipage}[h]{0.39\textwidth}%
\includegraphics[width=\textwidth]{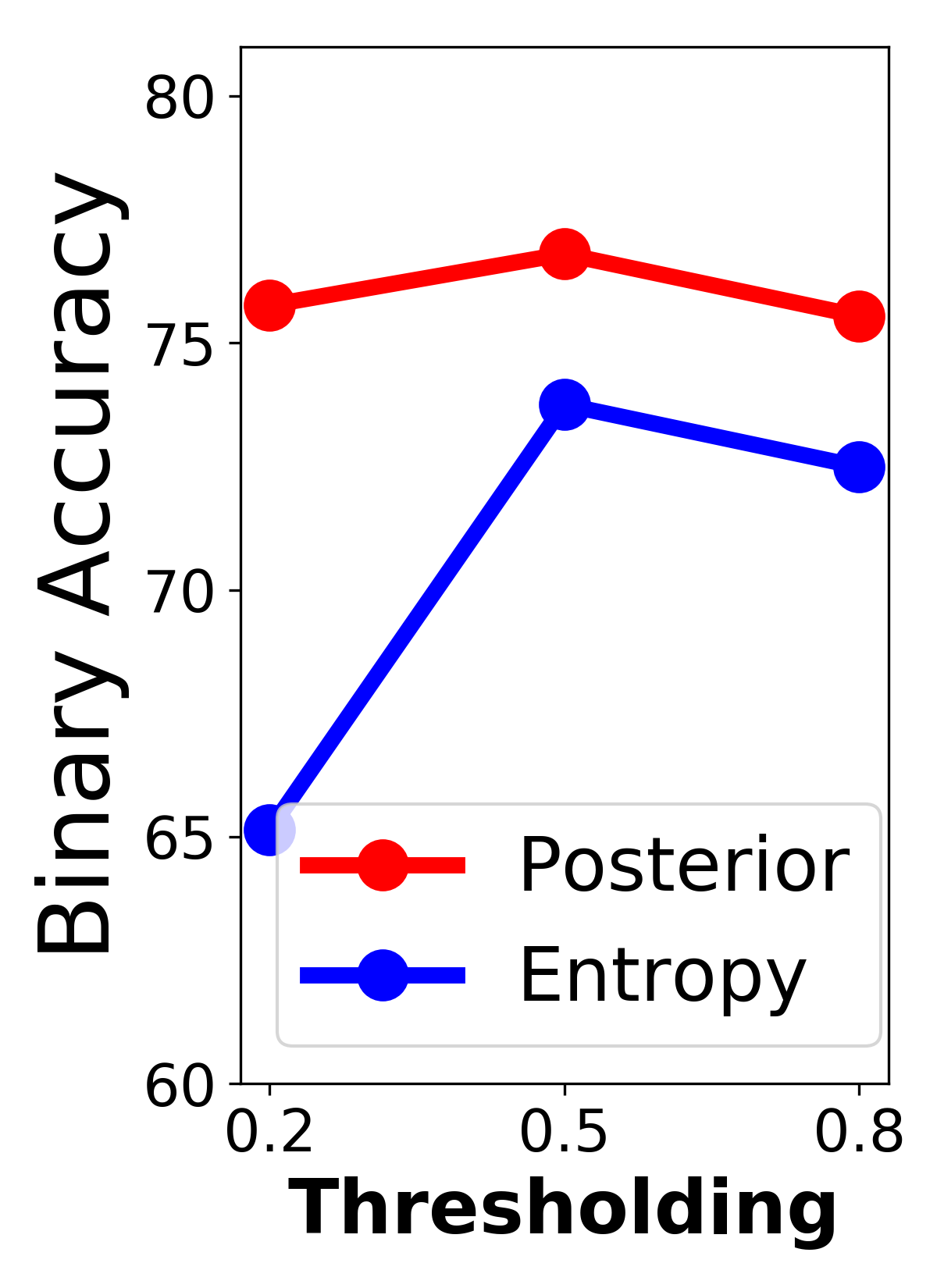}%
\subcaption{\small{Binary Acc.}}         \label{fig:bin_class}
\end{minipage}%
\hfill
\begin{minipage}[h]{0.59\textwidth}%
\includegraphics[width=\linewidth]{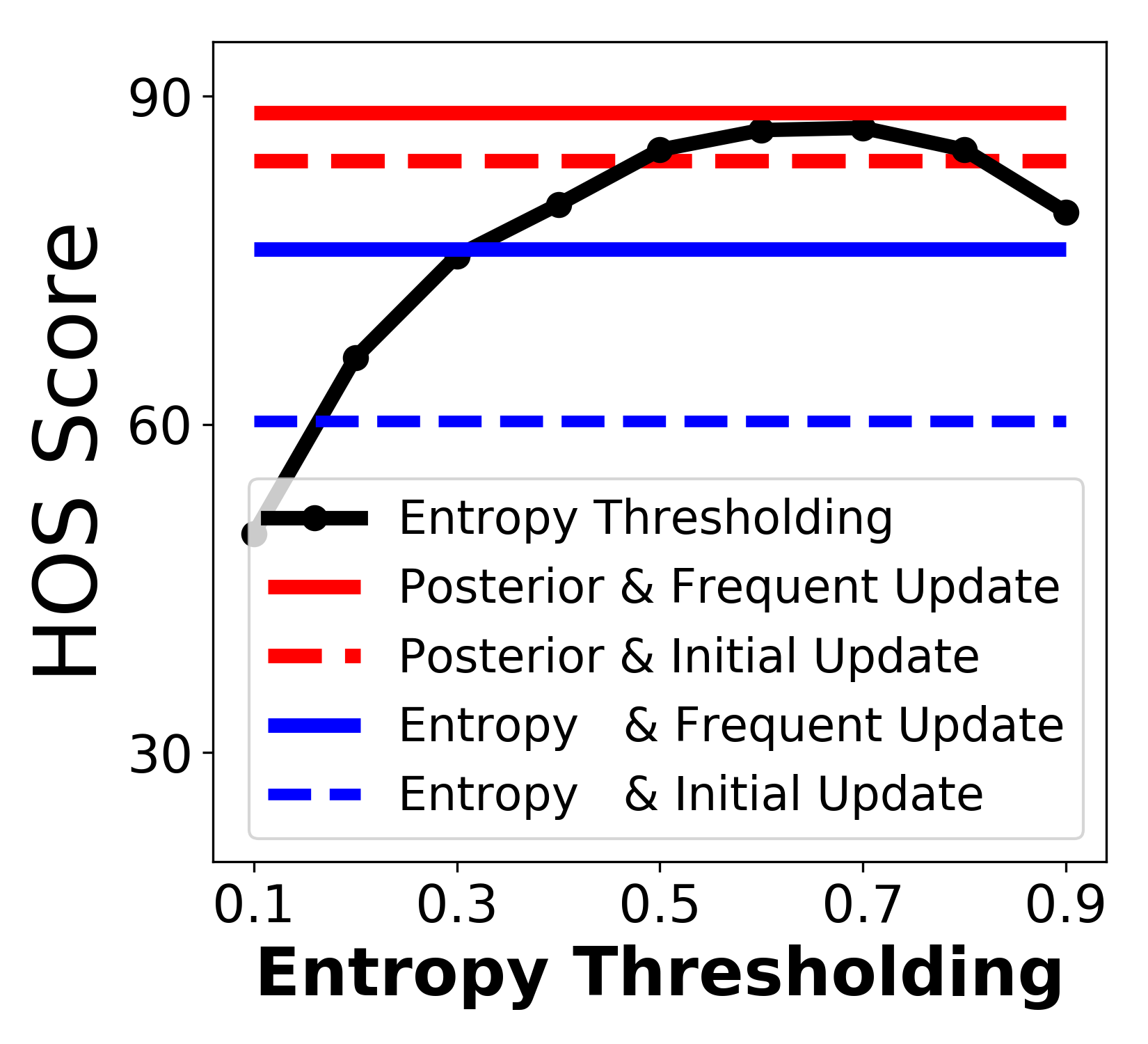}
\subcaption{\small{Ablation Studies }}      \label{fig:entropy_trhesholding}
\end{minipage}
\caption{\small{(a) Binary accuracy on \textit{tk}/\textit{tu} by thresholding $\hat{w}_{x}$. (b) Ablations for posterior inference with thresholding.}}
\end{minipage}%
\end{figure}

\textbf{Superiorities of Open-set Recognition:} 
\textbf{\textit{(Effective Recognition)}} Figure \ref{fig:histo_warm-up} represents the histogram of the source/target entropy values from $E$ after the early stage of training by Eq. (\ref{eqn:e_update}) (sensitivity analysis on the varying number of iterations for the early stage in Appendix \ref{sup-appendix:robust_ws}). Therefore, it demonstrates that our assumption on two modalities of the entropy values holds, robustly. 
Furthermore, the subfigure in Figure \ref{fig:histo_warm-up} shows the histogram of the $\hat{w}_{x}$ by the posterior inference. 
It shows that our \textit{open-set recognition}, $\hat{w}_x$, provides a more clear separation for the target-known (blue) and the target-unknown (red) instances than the entropy values. 
Figure \ref{fig:histo_main} represents the histogram of the entropy at the end of the training, which is further split than Figure \ref{fig:histo_warm-up}. 
This clear separation implicitly supports that the proposed feature alignment accomplishes both \textit{alignment} and \textit{segregation}.

\textbf{\textit{(Informativeness)}} 
Furthermore, we conduct a binary classification task by thresholding the weights, $\hat{w}_{x}$, whether it predicts $tk$ or $tu$ correctly. 
Figure \ref{fig:bin_class} represents the accuracies over the threshold values when applying the posterior inference or the normalized entropy values themselves, in Office-31. 
It shows that the posterior inference discriminates $tk$ and $tu$ more correctly than using the entropy value itself. 
Additionally, the posterior inference has similar accuracies over the threshold values (0.2, 0.5, and 0.8), which also indicates a clear separation. 
Therefore, we conclude that the posterior inference provides a more informative weighting procedure than the entropy value itself. 

\textbf{\textit{(Comparison to Thresholding)}} Figure \ref{fig:entropy_trhesholding} shows that UADAL with the posterior inference outperforms all cases of thresholding the entropy values. 
It means that the posterior inference is free to set thresholds. Moreover, the results demonstrate that UADAL with the posterior inference outperforms the case of using the normalized entropy value, and its frequent updating (solid line) leads to a better performance than only initial updating (dashed line). It is consistent with the analysis from Figure \ref{fig:histo_main}.

\begin{figure}[h]
\begin{minipage}[h]{0.63\textwidth}%
\begin{minipage}[h]{0.5\textwidth}%
\includegraphics[width=\textwidth]{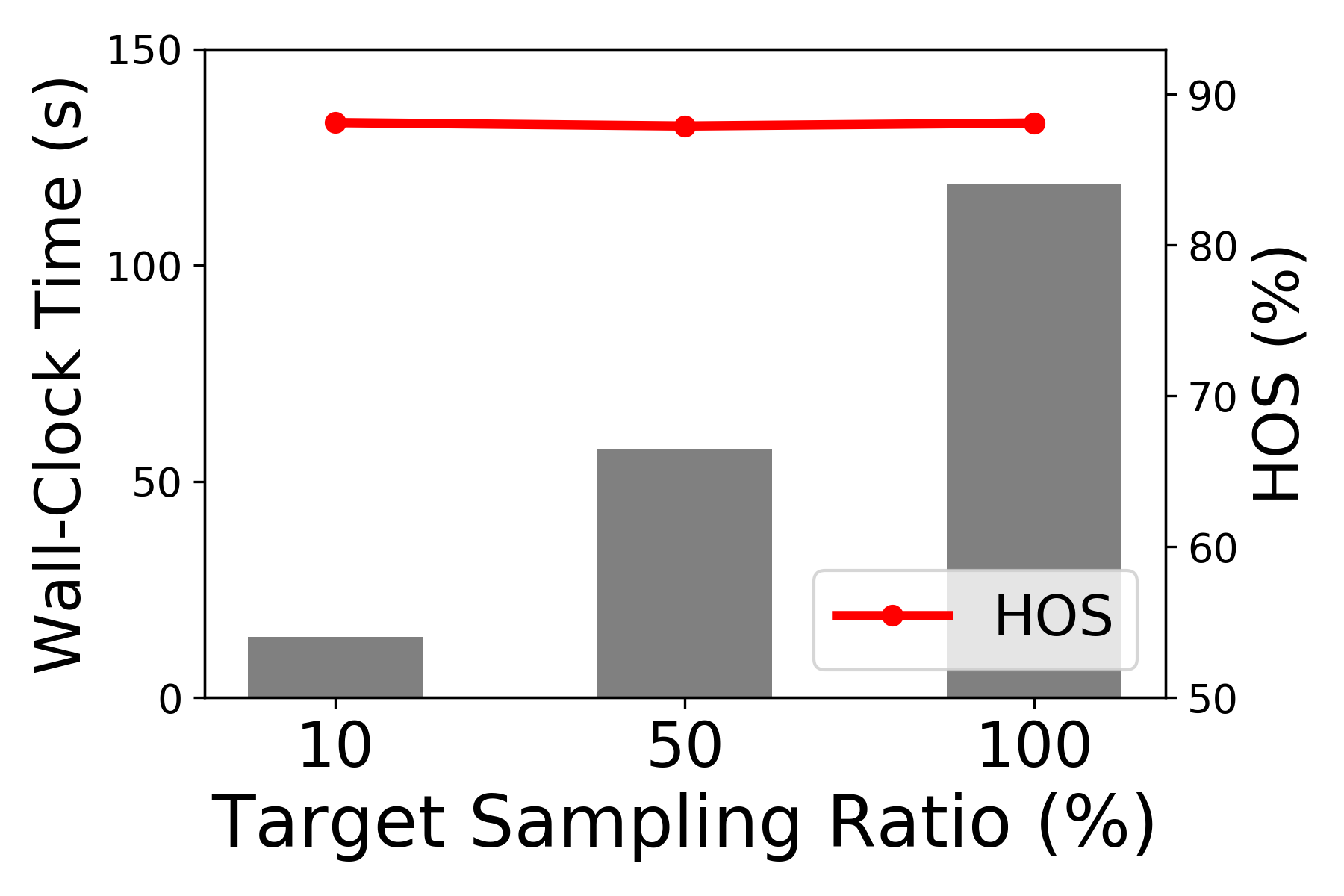}%
\end{minipage}%
\hfill
\begin{minipage}[h]{0.5\textwidth}%
\includegraphics[width=\linewidth]{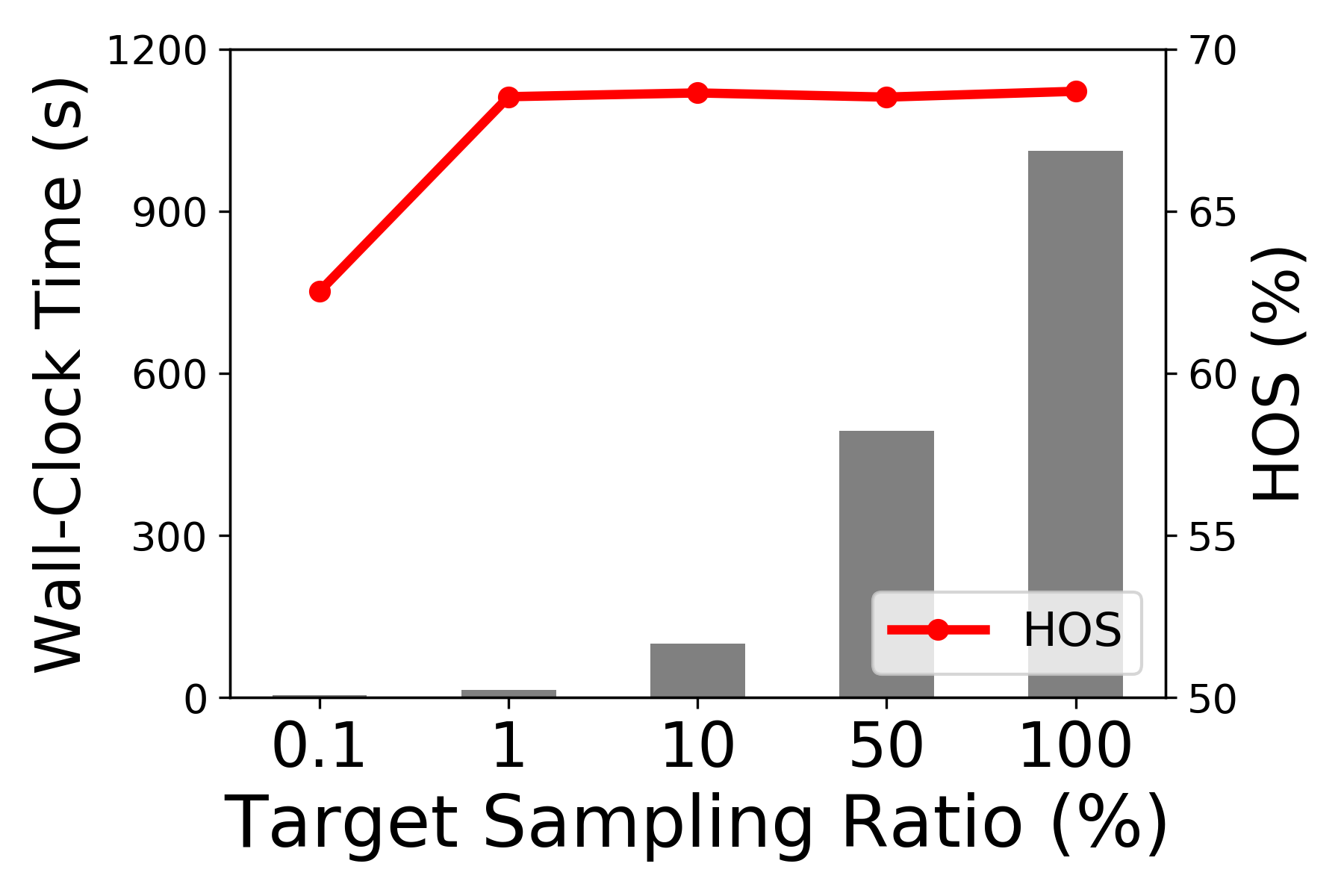}
\end{minipage} 
\captionof{figure}{\small{Ablation study of applying sampling on the target domain with Office-31 (left) and Office-Home (right) w.r.t. the averaged scores. }} \label{main:target_sample} 
\end{minipage}%
\hfill
\begin{minipage}[h]{0.36\textwidth}%
\scalebox{0.75}{
\begin{tabular}{|c|c|c|c|c|}
 \hline
\textbf{Dataset} &  \textbf{Ent.} &  \textbf{OS*}&  \textbf{UNK} &  \textbf{HOS}  \\ \hline
\multirow{2}{*}{\textbf{Office-31}} & X    & 85.4     & 90.0 &87.5   \\
                     &  {O}  &84.8     & 82.1 &\textbf{88.1}     \\ \hline
\multirow{2}{*}{\textbf{Office-Home}} &  X    & 61.9     & 75.2 &67.6     \\
                     &  {O}  & 62.6  & 78.0 & \textbf{68.7} \\ \hline
        \end{tabular} 
}
      \captionof{table}{\small{Ablation on Entropy Minimization (Ent.) Loss from Eq. (\ref{eqn:c_target}) on the Office-31 and Office-Home. The values in the table are the averaged score over the tasks in each dataset. The detailed results are in Table \ref{sup-tab:ablation_entropy_loss} in Appendix \ref{sup-appendix:entropymin}.}} \label{main:ablation_ent}
\end{minipage}%
\end{figure}

\textbf{\textit{(Efficiency)}} In terms of computation, the posterior inference increases the complexity because we need to fit the mixture model, where the computational complexity is $O(nk)$ with the number of samples, $n$, and the number of fitting iterations, $k$ ($n\gg k$).  
As an alternative, we fit the mixture model only by sampling the target instance. Figure \ref{main:target_sample} represents the wall-clock time increased by the fitting process and the performances over the sampling ratio (\%) for the target domain. Since the computational complexity is linearly increased by the number of samples, the wall-clock time is also linearly increased by increasing the sampling ratio. Interestingly, we observed that even when the sampling ratio is small, i.e. 10\%, the performances do not decrease, on both Office-31 and Office-Home datasets. The qualitative analysis of this sampling is provided in Appendix \ref{sup-appendix:sampling_bmm}.

\textbf{Entropy Minimization Loss} 
In order to show the effect of the entropy minimization by Eq. (\ref{eqn:c_target}), we conduct the ablation study on the Office-31 and Office-Home. 
Table \ref{main:ablation_ent} shows that the entropy minimization leads to better performance. However, UADAL without the loss still performs better than the other baselines (compared with Table \ref{tab:hos_office}). 
It represents that UADAL learns the feature space appropriately as we intended to suit Open-Set Domain Adaptation. The properly learned feature space leads to effectively classifying the target instances without entropy minimization.

\section{Conclusion}
We propose Unknown-Aware Domain Adversarial Learning (UADAL), which is the first approach to explicitly design the \textit{segregation} of the unknown features in the domain adversarial learning for OSDA. 
We design a new domain discrimination loss and formulate the sequential optimization for the unknown-aware feature alignment. 
Empirically, UADAL outperforms all baselines on all benchmark datasets with various backbone networks, by reporting state-of-the-art performances.

\begin{ack}
This research was supported by Research and Development on Key Supply Network Identification, Threat Analyses, Alternative Recommendation from Natural Language Data (NRF) funded by the Ministry of Education (2021R1A2C200981612).
\end{ack}

{
\small
\bibliography{egbib}
\bibliographystyle{abbrvnat}
}


\section*{Checklist}

\begin{enumerate}

\item For all authors...
\begin{enumerate}
  \item Do the main claims made in the abstract and introduction accurately reflect the paper's contributions and scope?
    \answerYes{}
  \item Did you describe the limitations of your work?
    \answerYes{Details in Appendix \ref{sup-sec:limitation}.}
  \item Did you discuss any potential negative societal impacts of your work?
    \answerYes{Details in Appendix \ref{sup-sec:limitation}.}
  \item Have you read the ethics review guidelines and ensured that your paper conforms to them?
    \answerYes{}
\end{enumerate}

\item If you are including theoretical results...
\begin{enumerate}
  \item Did you state the full set of assumptions of all theoretical results?
    \answerYes{}
        \item Did you include complete proofs of all theoretical results?
    \answerYes{}
\end{enumerate}

\item If you ran experiments...
\begin{enumerate}
  \item Did you include the code, data, and instructions needed to reproduce the main experimental results (either in the supplemental material or as a URL)?
    \answerYes{We provide code in supplemental material, and URL of data and instructions are provided in Appendix \ref{sup-appendix_dataset}.}
  \item Did you specify all the training details (e.g., data splits, hyperparameters, how they were chosen)?
    \answerYes{}
        \item Did you report error bars (e.g., with respect to the random seed after running experiments multiple times)?
    \answerYes{}
        \item Did you include the total amount of compute and the type of resources used (e.g., type of GPUs, internal cluster, or cloud provider)?
    \answerYes{The details in Appendix \ref{sup-appendix_optim_detail}.}
\end{enumerate}

\item If you are using existing assets (e.g., code, data, models) or curating/releasing new assets...
\begin{enumerate}
  \item If your work uses existing assets, did you cite the creators?
    \answerYes{}
  \item Did you mention the license of the assets?
    \answerYes{}
  \item Did you include any new assets either in the supplemental material or as a URL?
    \answerYes{See Appendix \ref{sup-appendix_baseline} and \ref{sup-appendix_dataset}.}
  \item Did you discuss whether and how consent was obtained from people whose data you're using/curating?
    \answerYes{We mentioned that we use the public benchmark dataset, in Section \ref{main_exp_setting}.}
  \item Did you discuss whether the data you are using/curating contains personally identifiable information or offensive content?
    \answerYes{We use image object recognition datasets, so there are no personaly identifiable information or offensive content.}
\end{enumerate}

\item If you used crowdsourcing or conducted research with human subjects...
\begin{enumerate}
  \item Did you include the full text of instructions given to participants and screenshots, if applicable?
    \answerNA{}
  \item Did you describe any potential participant risks, with links to Institutional Review Board (IRB) approvals, if applicable?
    \answerNA{}
  \item Did you include the estimated hourly wage paid to participants and the total amount spent on participant compensation?
    \answerNA{}
\end{enumerate}

\end{enumerate}

\newpage
\appendix
This is an Appendix for "Unknown-Aware Domain Adversarial Learning for Open-Set Domain Adaptation". 
First, Section \ref{sup-literature_review} provides comprehensive literature reviews.
Second, Section \ref{sup-sec1_algo} provides the details for the proposed model, UADAL. It consists of three parts; 1) Sequential Optimization Problem (Section \ref{sup-appendidx_sequential_optimization}),  2) Posterior Inference (Section \ref{sup-bmm}), and 3) Training Details (Section \ref{sup-appendix:trainig_algorithm}). 
Third, Section \ref{sup-exp_section} provides the experimental details, including the implementation details (Section \ref{sup-appendix_imple_detail}) and the detailed experimental results (Section \ref{sup-appendix_result_analysis}). Finally, Section \ref{sup-sec:limitation} shows `Limitations and Potential Negative Societal Impacts' of this work.

Note that all references and the equation numbers are independent of the main paper. We utilize the expression of "Eq. (XX) in the main paper", especially when referring the equations from the main paper. 

\section{Literature Reviews}\label{sup-literature_review}

\textbf{(Closed-Set) Domain Adaptation (DA)} is the task of leveraging the knowledge of the labeled source domain to the target domain \cite{borgwardt2006integrating, ben2010theory, baktashmotlagh2014domain, long2015learning,ganin2016domain}. Here, DA assumes that the class sets from the source and the target domain are identical. The typical approaches to solve DA have focused on minimizing the discrepancy between the source and the target domain since they are assumed to be drawn from the different distributions \cite{pan2010domain,tzeng2014deep,long2015trans, long2016unsupervised}. The discrepancy-based approaches have been proposed to define the metric to measure the distance between the source and the target domain in the feature space \cite{long2015learning,xu2019larger}. Other works are based on adversarial methods \cite{ganin2015unsupervised,tzeng2017adversarial,long2017deep}. These approaches introduce the domain discriminator and a feature generator in order to learn the feature space to be domain-invariant. Self-training methods are also proposed to mitigate the DA problems by pseudo-labeling the target instances \cite{liu2021cycle,kumar2020understanding,mey2016soft,prabhu2021sentry}, originally tailored to semi-supervised learning \cite{grandvalet2004semi,sohn2020fixmatch}.

\textbf{Open-Set Recognition (OSR)} is the task to classify an instance correctly if it belongs to the known classes or to reject outliers otherwise, during the testing phase \cite{scheirer2012toward}. 
Here, the outliers are called `open-set' which is not available in the training phase. 
Many approaches have been proposed to solve the OSR problem \cite{Bendale_2016_CVPR,neal2018open,sun2020conditional,oza2019c2ae,shu2020p,chen2020learning,chen2021adversarial}. OpenMax \cite{Bendale_2016_CVPR} is the first deep learning approach for OSR, introducing a new model layer to estimate the probability of an instance being \textit{unknown} class, based on the Extreme Value Theory (EVT). 
OSRCI \cite{neal2018open}, another stream of the approaches utilizing GANs, generates the virtual images which are similar to the training instances but do not belong to the known classes. 
Other approaches contain the reconstruction-based methods \cite{sun2020conditional,oza2019c2ae} and prototype-based methods \cite{shu2020p,chen2020learning,chen2021adversarial}. 
Also, \cite{vaze2021open} claims that the performance of OSR is highly correlated with its accuracy on the closed-set classes. This claim is associated with our open-set recognizer.

\textbf{Open-Set Domain Adaptation} is a more realistic and challenging task of Domain Adaptation, where the target domain includes the open-set instances which are not discovered by the source domain.
In terms of the domain adversarial learning, in addition to STA, 
OSBP \cite{saito2018open} utilizes a classifier to predict the target instances to the pre-determined threshold, and trains the feature extractor to deceive the classifier for aligning to \textit{known} classes or rejecting as \textit{unknown} class. 
However, their recognition on \textit{unknown} class only depends on the threshold value, without considering the data instances.
PGL \cite{luo2020progressive} introduces progressive graph-based learning to regularize the class-specific manifold, while jointly optimizing with domain adversarial learning. However, their adversarial loss includes all instances of the target domain, which is critically weakened by the negative transfer.


In terms of the self-supervised learning, 
ROS \cite{ROS2020ECCV} is rotation-based self-supervised learning to compute the normality score for separating target known/unknown information. 
DANCE \cite{DANCE2020NIPS} is based on a self-supervised clustering to move a target instance either to shared-known prototypes in the source domain or to its neighbor in the target domain. DCC \cite{li2021domain} is a domain consensus clustering to exploit the intrinsic structure of the target domain. However, these approaches do not have any feature alignments between the source and the target domain, which leads to performance degradation under the significant domain shifts. 

There is also a notable work, OSLPP \cite{wang2021progressively} optimizing projection matrix toward a common subspace to class-wisely align the source and the target domain. 
This class-wise matching depends on the pseudo label for the target instances. However, inaccurate pseudo labels such as early mistakes can result in error accumulation and domain misalignment \cite{prabhu2021sentry}. 
Moreover, the optimization requires the pair-wise distance calculation, which results in a growing complexity of $O(n^2)$ by the $n$ data instances. It could be limited to the large-scaled domain.


\section{Algorithm and Optimization Details} \label{sup-sec1_algo}
\subsection{Sequential Optimization Problem} \label{sup-appendidx_sequential_optimization}
\subsubsection{Decomposition of Domain Discrimination Loss} \label{sup-lossDecompose}
The domain discrimination loss for the target domain is as below, (Eq. (\ref{eqn:loss_d_t}) in the main paper)
\begin{equation*}\label{sup-eqn:dom_t}
\mathcal{L}_{d}^{t}(\theta_{g}, \theta_{d})=\mathbb E_{p_{t}(x)}[\ -w_{x} \log D_{tk}(G(x))-(1-w_{x}) \log D_{tu}(G(x))],
\end{equation*}
where $w_{x}=p(known|x)$ is the probability of a target instance, $x$, belonging to a target-known class. Then, we decompose $\mathcal{L}_{d}^{t}(\theta_{g}, \theta_{d})$ into the two terms, $\mathcal{L}_{d}^{tk}(\theta_{g}, \theta_{d})$ and $\mathcal{L}_{d}^{tu}(\theta_{g}, \theta_{d})$, as follows,
\begin{align*} 
\mathcal{L}_{d}^{t}(\theta_{g}, \theta_{d}) &= \mathcal{L}_{d}^{tk}(\theta_{g}, \theta_{d})  +  \mathcal{L}_{d}^{tu}(\theta_{g}, \theta_{d}), \\
\mathcal{L}_{d}^{tk}(\theta_{g}, \theta_{d}) &:=\lambda_{tk}\cdot\mathbb E_{p_{tk}(x)}\left[ - \log D_{tk}(G(x))\right], \\
\mathcal{L}_{d}^{tu}(\theta_{g}, \theta_{d}) &:=\lambda_{tu}\cdot\mathbb E_{p_{tu}(x)}\left[ - \log D_{tu}(G(x))\right],
\end{align*}
where $ p_{tk}(x):=p_{t}(x|known)$ and $p_{tu}(x):=p_{t}(x|unknown)$; $\lambda_{tk}=p(known)$; and $\lambda_{tu}=p(unknown)$. 
\begin{proof}
We start this proof from Eq. (\ref{eqn:loss_d_t}) in the main paper,
\begin{align*}
\mathcal{L}_{d}^{t}(\theta_{g}, \theta_{d})=&\ \mathbb E_{p_{t}(x)}[\ -w_{x} \log D_{tk}(G(x))-(1-w_{x}) \log D_{tu}(G(x))].
\end{align*}
For the convenience of the derivation, we replace $w_{x}$ as $p(known|x)$.
  \begin{align*}
\mathcal{L}_{d}^{t} & (\theta_{g}, \theta_{d})  = \mathbb E_{x \sim p_{t}(x)}\left[\ - p_{t}(known |x)\log D_{tk}(G(x))- p_{t}(unknown |x) \log D_{tu}(G(x))\right]
\end{align*}
\begin{align*}
&=  \int\limits_{x \sim p_{t}(x)} \Big(-p_{t}(known \vert x)\log D_{tk}(G(x)) - p_{t}(unknown \vert x)\log D_{tu}(G(x))\Big) \, dx \\
&=  \int\limits_{x}\Big( -p_{t}(x)(p_{t}(known \vert x)\log D_{tk}(G(x)) - p_{t}(x)p_{t}(unknown \vert x)\log D_{tu}(G(x))\Big) \, dx \\
&=  \int\limits_{x} \Big(-p_{t}(known, x)\log D_{tk}(G(x)) - p_{t}(unknown, x)\log D_{tu}(G(x))\Big) \, dx \\
&=  \int\limits_{x} -p_{t}(x\vert known)p_{t}(known)\log D_{tk}(G(x)) \, dx  + \int\limits_{x}- p_{t}(x\vert unknown)p_{t}(unknown)\log D_{tu}(G(x)) \, dx \\
&=  p_{t}(known)\hspace{-0.1em}\int\limits_{x}-p_{t}(x \vert known)\log D_{tk}(G(x)) \, dx + p_{t}(unknown)\hspace{-0.1em}\int\limits_{x} - p_{t}(x |unknown)\log D_{tu}(G(x)) \, dx \\
&= p_{t}(known)\hspace{-0.1em} \int\limits_{x \sim p_{tk}(x)} - \log D_{tk}(G(x)) \,dx + p_{t}(unknown)\hspace{-0.1em} \int\limits_{x \sim p_{tu}(x)} - \log D_{tu}(G(x)) \,dx \\
& = p_{t}(known)\mathbb E_{x \sim p_{tk}(x)}\left[ - \log D_{tk}(G(x))\right] +p_{t}(unknown) \mathbb E_{x \sim p_{tu}(x)}\left[ - \log D_{tu}(G(x))\right] \\
& =\lambda_{tk}\mathbb E_{x \sim p_{tk}(x)}\left[ - \log D_{tk}(G(x))\right] +\lambda_{tu} \mathbb E_{x \sim p_{tu}(x)}\left[ - \log D_{tu}(G(x))\right].
\end{align*}
Thus, we define new terms with respect to $p_{tk}(x)$ and $p_{tu}(x)$ as follow.
\begin{align*}
\mathcal{L}_{d}^{tk}(\theta_{g}, \theta_{d}) & :=  \lambda_{tk}\mathbb E_{x \sim p_{tk}(x)}\left[ - \log D_{tk}(G(x))\right]  \\
\mathcal{L}_{d}^{tu}(\theta_{g}, \theta_{d}) & := \lambda_{tu}\mathbb E_{x \sim p_{tu}(x)}\left[ - \log D_{tu}(G(x))\right]
\end{align*}
Therefore, by the above derivation, we decompose $\mathcal{L}_{d}^{t}(\theta_{g}, \theta_{d})$ as follow, 
\begin{align*}
\mathcal{L}_{d}^{t}(\theta_{g}, \theta_{d}) = \mathcal{L}_{d}^{tk}(\theta_{g}, \theta_{d})  +  \mathcal{L}_{d}^{tu}(\theta_{g}, \theta_{d}).
\end{align*}

\end{proof}

\subsubsection{Optimal point of the domain discriminator $D$} \label{sup-appendidx_optim_D}
The optimal $D^{*}$ given the fixed $G$ is as follow (Eq. (\ref{eqn:optimal_D}) in the main paper), 
\begin{align*}
\label{sup-appendix_eqn:optimal_D}
D^{*}(G(x;\theta_{g})) =D^{*}(z)=\Big[ \frac{p_{s}(z)}{2 p_{avg}(z)}, \frac{\lambda_{tk} p_{tk}(z)}{2 p_{avg}(z)}, \frac{\lambda_{tu} p_{tu}(z)}{2 p_{avg}(z)} \Big], 
\end{align*}
where $p_{avg}(z)=(p_{s}(z) + \lambda_{tk} p_{tk}(z) + \lambda_{tu} p_{tu}(z))/2$. Note that $z\in\mathcal{Z}$ stands for the feature space from $G$. In other words, $p_{d}(z) = \{G(x;\theta_{g})|x \sim p_{d}(x)\}$ where $d$ is $s$, $tk$, or  $tu$. 
\begin{proof}
First, we fix $G$, and optimize the problem with respect to $D$.
  \begin{align*}
 \min_{\theta_{d}}\ & \mathcal{L}_{D}(\theta_{g}, \theta_{d}) =\mathcal{L}_{d}^{s}(\theta_{g}, \theta_{d}) +\ \mathcal{L}_{d}^{tk}(\theta_{g}, \theta_{d}) + \mathcal{L}_{d}^{tu}(\theta_{g}, \theta_{d})\\
&=  -\int\limits_{x \sim p_{s}(x)}\hspace{-1em} \log D_{s}(G(x)) \,dx - \lambda_{tk}  \hspace{-1em}  \int\limits_{x \sim p_{tk}(x)}\hspace{-1em}\log D_{tk}(G(x)) \,dx -  \lambda_{tu}\hspace{-1em} \int\limits_{x \sim p_{tu}(x)}\hspace{-1em} \log D_{tu}(G(x)) \,dx\\
&=  -\int\limits_{z \sim p_{s}(z)}\hspace{-1em} \log D_{s}(z) \,dz - \lambda_{tk} \hspace{-1em}  \int\limits_{z \sim p_{tk}(z)}\hspace{-1em} \log D_{tk}(z) \,dz -  \lambda_{tu}\hspace{-1em} \int\limits_{x \sim p_{tu}(z)}\hspace{-1em} \log D_{tu}(z) \,dz\\
    &= \int_{z} \Big(-p_{s}(z)\log D_{s}(z) - \lambda_{tk} p_{tk}(z)\log D_{tk}(z) - \lambda_{tu} p_{tu}(z)\log D_{tu}(z)\Big)\,dz
  \end{align*}
Also, note that $D_{s}(z)+D_{tk}(z)+D_{tu}(z)=1$ for all $z$. Therefore, we transform the optimization problem as follow \cite{goodfellow2014generative}:
  \begin{align*}
\min_{\theta_{d}}& \quad -p_{s}(z)\log D_{s}(z) - \lambda_{tk} p_{tk}(z)\log D_{tk}(z) - \lambda_{tu} p_{tu}(z)\log D_{tu}(z) \\
 \text{s.t.}& \quad  D_{s}(z)+D_{tk}(z)+D_{tu}(z)=1
  \end{align*}
for all $z$. We introduce the Lagrange variable $v$ to use Lagrange multiplier method.  
\begin{align*}
\begin{split}
\min_{\theta_{d}} \mathcal{\tilde{L}}_D :=& -p_{s}(z)\log D_{s}(z) - \lambda_{tk} p_{tk}(z)\log D_{tk}(z) - \lambda_{tu} p_{tu}(z)\log D_{tu}(z) \\
	&\quad +  v  (D_{s}(z)+D_{tk}(z)+D_{tu}(z) -1) \\
\end{split}  
\end{align*}
To find optimal $D^{*}$, we find the derivative of $\mathcal{\tilde{L}}_D$ with respect to $D$ and $v$.
\begin{align*}
&\frac{\partial \mathcal{\tilde{L}}_D}{\partial D_{s}(z)} = \frac{-p_{s}(z)}{D_{s}(z)}+v=0\quad \Leftrightarrow \quad D_{s}(z) = \frac{p_{s}(z)}{v} \\
&\frac{\partial \mathcal{\tilde{L}}_D}{\partial D_{tk}(z)} = \frac{-\lambda_{tk} p_{tk}(z)}{D_{tk}(z)}+v=0\quad \Leftrightarrow \quad D_{tk}(z) = \frac{\lambda_{tk} p_{tk}(z)}{v} \\
&\frac{\partial \mathcal{\tilde{L}}_D}{\partial D_{tu}(z)} = \frac{-\lambda_{tu} p_{tu}(z)}{D_{tu}(z)}+v=0\quad \Leftrightarrow \quad D_{tu}(z) = \frac{\lambda_{tu} p_{tu}(z)}{v} \\
&\frac{\partial \mathcal{\tilde{L}}_D}{\partial v} =D_{s}(z)+D_{tk}(z)+D_{tu}(z) -1 =0 \quad \Leftrightarrow \quad D_{s}(z)+D_{tk}(z)+D_{tu}(z)=1
\end{align*}
From the above equations, we have 
\begin{align*}
D_{s}(z)+D_{tk}(z)+D_{tu}(z) = \frac{p_{s}(z)}{v}  + \frac{\lambda_{tk} p_{tk}(z)}{v} +  \frac{\lambda_{tu} p_{tu}(z)}{v} = 1,
\end{align*}
then,
\begin{align*}
 v=p_{s}(z) + \lambda_{tk} p_{tk}(z) + \lambda_{tu} p_{tu}(z) = 2 p_{avg}(z).
\end{align*}
Thus, we get optimal $D^{*}$ as
\begin{align*}
D^{*}(z)=[D_{s}^{*}(z), D_{tk}^{*}(z), D_{tu}^{*}(z) ]= \Big[ \frac{p_{s}(z)}{2 p_{avg}(z)}, \frac{\lambda_{tk} p_{tk}(z)}{2 p_{avg}(z)}, \frac{\lambda_{tu} p_{tu}(z)}{2 p_{avg}(z)} \Big].
\end{align*}
\end{proof}

\subsubsection{Proof of Theorem \ref{theorem1_optim_G}} \label{sup-appendix_optim_G}
\begin{theorem}
Let $\theta^{*}_{d}$ be the optimal parameter of $D$ by optimizing Eq. (\ref{eqn:max_D}) in the main paper. Then, $- \mathcal{L}_{G}(\theta_{g}, \theta^{*}_{d})$ can be expressed as, with a constant $C_{0}$,
\begin{align*}
 -\mathcal{L}_{G}(\theta_{g},& \theta^{*}_{d}) = D_{KL}( p_{s} \Vert p_{avg} ) + \lambda_{tk} D_{KL}( p_{tk} \Vert p_{avg} ) - \lambda_{tu} D_{KL}( p_{tu} \Vert p_{avg} )+C_{0}.
\end{align*}
\end{theorem}
\begin{proof}
First, we change maximization problem into minimization problem, and substitute $D^{*}$ by using Eq. (\ref{eqn:optimal_D}) in the main paper. Note that $p_{avg}=(p_{s}(z) + \lambda_{tk} p_{tk}(z) + \lambda_{tu} p_{tu}(z))/2$.
\begin{align}
\min_{\theta_{g}}\ -& \mathcal{L}_{G}(\theta_{g}, \theta_{d}) = -\mathcal{L}_{d}^{s}(\theta_{g}, \theta_{d})  - \ \mathcal{L}_{d}^{tk}(\theta_{g}, \theta_{d}) + \mathcal{L}_{d}^{tu}(\theta_{g}, \theta_{d}) \\
 &=\int_{z} \left(p_{s}(z)\log D_{s}^{*}(z) +\lambda_{tk}  p_{tk}(z)\log D_{tk}^{*}(z) - \lambda_{tu} p_{tu}(z)\log D_{tu}^{*}(z)\right)\,dz \\   \label{sup-appendix_eq_from}
&= \int_{z} (\ p_{s}(z)\log {\frac{p_{s}(z)}{2 p_{avg}(z)}}  + \lambda_{tk} p_{tk}(z)\log{\frac{\lambda_{tk} p_{tk}(z)}{2 p_{avg}(z)}} -\lambda_{tu} p_{tu}(z)\log{\frac{\lambda_{tu} p_{tu}(z)}{2 p_{avg}(z)}}\ )\, dz \\  \label{sup-appendix_eq_to}
&=  D_{KL}(p_{s} \Vert p_{avg}) + \lambda_{tk} D_{KL}(p_{tk} \Vert  p_{avg}) - \lambda_{tu} D_{KL}(p_{tu} \Vert  p_{avg}) + C_{0}
\end{align}
where $C_{0}= -2\lambda_{tk}\log2 + \lambda_{tk}\log\lambda_{tk} - \lambda_{tu}\log\lambda_{tu}$.

We use below derivation for the last equation, i.e. from Eq. (\ref{sup-appendix_eq_from}) to Eq. (\ref{sup-appendix_eq_to}). 
\begin{align*}
D_{KL}(p_{s} \Vert p_{avg}) & =  \int_{z}p_{s}(z)\log {\frac{p_{s}(z)}{(p_{avg}(z)}} \,dz \\
& =  \int_{z}p_{s}(z)\log {\frac{2 p_{s}(z)}{p_{s}(z) + \lambda_{tk} p_{tk}(z) + \lambda_{tu} p_{tu}(z)}} \,dz \\
&= \int_{z}p_{s}(z)\log {\frac{p_{s}(z)}{p_{s}(z) + \lambda_{tk} p_{tk}(z) + \lambda_{tu} p_{tu}(z)}} \,dz \\
&\quad\quad + \int_{z}p_{s}(z)\log {2} \,dz \\
&= \int_{z}p_{s}(z)\log {\frac{p_{s}(z)}{2 p_{avg}(z)}} \,dz + \int_{z}p_{s}(z)\log {2} \,dz
\end{align*}
Thus,
\begin{align}
\int_{z}p_{s}(z)\log {\frac{p_{s}(z)}{2 p_{avg}(z)}} \,dz = D_{KL}(p_{s} \Vert p_{avg}) - \log{2} \label{sup-p_s}
\end{align}
For the second term, 
\begin{align*}
&D_{KL}(p_{tk} \Vert p_{avg}) \\
& =  \int_{z}p_{tk}(z)\log {\frac{p_{tk}(z)}{p_{avg}(z)}} \,dz \\
& =  \int_{z}p_{tk}(z)\log {\frac{2 p_{tk}(z)}{p_{s}(z) + \lambda_{tk} p_{tk}(z) + \lambda_{tu} p_{tu}(z)}} \,dz \\ 
&= \int_{z}p_{tk}(z)\log {\frac{p_{tk}(z)}{p_{s}(z) + \lambda_{tk} p_{tk}(z) + \lambda_{tu} p_{tu}(z)}} \,dz + \int_{z}p_{tk}(z)\log {2} \,dz \\
&= \int_{z}p_{tk}(z)\log {\frac{p_{tk}(z)}{p_{s}(z) + \lambda_{tk} p_{tk}(z) + \lambda_{tu} p_{tu}(z)}} \,dz + \log {2} \\
&= \int_{z}p_{tk}(z)(\log {\frac{p_{tk}(z)}{p_{s}(z) + \lambda_{tk} p_{tk}(z) + \lambda_{tu} p_{tu}(z)}} + \log\lambda_{tk} - \log\lambda_{tk}) \,dz + \log {2} \\
&= \int_{z}p_{tk}(z)\log {\frac{\lambda_{tk} p_{tk}(z)}{p_{s}(z) + \lambda_{tk} p_{tk}(z) + \lambda_{tu} p_{tu}(z)}}\,dz - \log\lambda_{tk} + \log {2} \\
&= \int_{z}p_{tk}(z)\log {\frac{\lambda_{tk} p_{tk}(z)}{2p_{avg}(z)}}\,dz - \log\lambda_{tk} + \log {2} \\
\end{align*}
By multiplying $\lambda_{tk}$,
\begin{align*}
\lambda_{tk} D_{KL}(p_{tk} \Vert p_{avg}) &= \int_{z}\lambda_{tk}  p_{tk}(z)\log {\frac{\lambda_{tk} p_{tk}(z)}{2p_{avg}(z)}}\,dz -  \lambda_{tk}\log{\frac{\lambda_{tk}}{2}}     
\end{align*}
Thus, 
\begin{align} \label{sup-p_tk}
\int_{z}\lambda_{tk}  p_{tk}(z)\log {\frac{\lambda_{tk} p_{tk}(z)}{2 p_{avg}(z)}}\,dz &= \lambda_{tk} D_{KL}(p_{tk} \Vert p_{avg})+   \lambda_{tk}\log{\frac{\lambda_{tk}}{2}}
\end{align}
Similarly, for the third term,
\begin{align}  \label{sup-p_tu}
\int_{z}\lambda_{tu}  p_{tu}(z)\log {\frac{\lambda_{tu}  p_{tu}(z)}{2p_{avg}(z)}}\,dz &= \lambda_{tu}  D_{KL}(p_{tu} \Vert p_{avg}) +   \lambda_{tu}\log{\frac{\lambda_{tu}}{2}}
\end{align}
In summary, from the Eq. (\ref{sup-p_s}), (\ref{sup-p_tk}), and (\ref{sup-p_tu}), we obtain the minimization problem with respect to $G$ as follows,
\begin{align}
\begin{split}\label{sup-eq_final_G}
\min_{\theta_{g}}\ - \mathcal{L}_{G}(\theta_{g}, \theta_{d})& = -\mathcal{L}_{d}^{s}(\theta_{g}, \theta_{d})  - \ \mathcal{L}_{d}^{tk}(\theta_{g}, \theta_{d}) + \mathcal{L}_{d}^{tu}(\theta_{g}, \theta_{d}) \\
&=  D_{KL}(p_{s} \Vert p_{avg}) + \lambda_{tk} D_{KL}(p_{tk} \Vert p_{avg}) - \lambda_{tu} D_{KL}(p_{tu} \Vert p_{avg}) + C_{0},
\end{split}
\end{align}
where $C_{0}=-\log{2}+\lambda_{tk}\log{\frac{\lambda_{tk}}{2}}-\lambda_{tu}\log{\frac{\lambda_{tu}}{2}}= -2\lambda_{tk}\log2 + \lambda_{tk}\log\lambda_{tk} - \lambda_{tu}\log\lambda_{tu}$.
\end{proof}

\subsubsection{Proof of Proposition \ref{thm:noninf2}} \label{sup-boundness}
\begin{proposition}\label{sup-appendix_thm:noninf}
The third term of the right-hand side in Eq. (\ref{eqn:optimal_G}) in the main paper,  $D_{KL}( p_{tu} \Vert p_{avg} )$, is bounded to\ $\log2-\log{\lambda_{tu}}$ .
\end{proposition}
\begin{proof}
\begin{align*}
&D_{KL}(p_{tu} \Vert p_{avg}) = \int_{z} p_{tu}(z) \log{\frac{p_{tu}(z)}{p_{avg}(z)}} \,dz \\
&= \int_{z}p_{tu}(z) \log{\frac{2 p_{tu}(z)}{p_{s}(z) + \lambda_{tk} p_{tk}(z) + \lambda_{tu} p_{tu}(z)}} \,dz \\
&=  \int_{z}p_{tu}(z) \log{2} \,dz + \int_{z} p_{tu}(z) \log{\frac{p_{tu}(z)}{p_{s}(z) + \lambda_{tk} p_{tk}(z) + \lambda_{tu} p_{tu}(z)}} \,dz \\
&=  \log{2}  + \int_{z} p_{tu}(z)\log{p_{tu}(z)} \,dz - \int_{z} p_{tu}(z)\log{(p_{s}(z) + \lambda_{tk}p_{tk}(z) + \lambda_{tu}p_{tu}(z))} \,dz \\
& \le \log{2}  + \int_{z}p_{tu}(z) \log{ p_{tu}(z)} \,dz -  \int_{z}p_{tu}(z) \log{\lambda_{tu} p_{tu}(z)} \,dz \\
& \quad(\ p_{s}(z) + \lambda_{k}p_{tk}(z) + \lambda_{tu}p_{tu}(z) \ge  \lambda_{tu}p_{tu}(z) \text{ for all } z) \\
& =\log{2}  + \int_{z}p_{tu}(z)\log{ p_{tu}(z)} \,dz -  \int_{z} p_{tu}(z) \log{p_{tu}(z)} \,dz -  \int_{z}p_{tu}(z) \log{\lambda_{tu}} \,dz  \\
& =\log{2} -\log{\lambda_{tu}}
\end{align*}
\text{Therefore,} 
\begin{align*}
 D_{KL}(p_{tu}  \Vert p_{avg})  \le  \log{2} -\log{\lambda_{tu}}.
\end{align*}
\end{proof}

\subsubsection{Proof of Proposition \ref{thm:f_div}} \label{sup-appendix_fdivergence}

\begin{proposition}\label{sup-appendix_thm:f_div}
Assume that $\text{supp}(p_s) \cap \text{supp}(p_{tu})=\emptyset$ and $\text{supp}(p_{tk}) \cap \text{supp}(p_{tu})=\emptyset$, where $\text{supp}(p):=\{ z \in \mathcal{Z} | p(z) > 0 \}$ is the support set of probability distribution $p$. Then, the minimization problem with respect to $G$, Eq. (\ref{eqn:optimal_G}) in the main paper, is equivalent to the minimization problem of summation on two $f$-divergences. 
\begin{align*} 
D_{f_1}(p_{s}||p_{tk}) + \lambda_{tk} D_{f_2}(p_{tk}||p_{s}),
\end{align*}

where $f_{1}(u)=u \log \frac{u}{(1-\alpha)u+\alpha}$, and $f_{2}(u)=u \log \frac{u}{\alpha u+(1-\alpha)}$.
Therefore, the minimum of Eq. (\ref{eqn:optimal_G}) in the main paper is achieved if and only if $p_{s}=p_{tk}$.
\end{proposition}

\begin{proof}
The minimization problem with respect to $G$ can be expressed as below:
\begin{align}\label{sup-appendix_eq_optimal_G}
& D_{KL}(p_{s}\Vert p_{avg})  + \lambda_{tk} D_{KL}(p_{tk} \Vert p_{avg})  - \lambda_{tu} D_{KL}(p_{tu} \Vert p_{avg}) + C_{0}
\end{align} 

where $C_{0}= -2\lambda_{tk}\log2 + \lambda_{tk}\log\lambda_{tk} - \lambda_{tu}\log\lambda_{tu}$.

We assume that (i) $\text{supp}(p_s) \cap \text{supp}(p_{tu})=\emptyset$ and (ii) $\text{supp}(p_{tk}) \cap \text{supp}(p_{tu})=\emptyset$, where $\text{supp}(p):=\{ z \in \mathcal{Z} | p(z) > 0 \}$ be the support set of probability distribution $p$. We denote $\mathcal{Z}_1:=\mathcal{Z} \setminus \text{supp}(p_{tu})$ and $\mathcal{Z}_2:=\text{supp}(p_{tu})$. Then, the first and second term in Eq. (\ref{sup-appendix_eq_optimal_G}) are written as below, respectively:
\begin{align}\label{sup-eq_optimal_s}
D_{KL}(p_{s} \Vert p_{avg}) =  \int_{\mathcal{Z}}p_{s}(z)\log {\frac{p_{s}(z)}{p_{avg}(z)}} \,dz =  \int_{\mathcal{Z}_{1}}p_{s}(z)\log {\frac{p_{s}(z)}{(p_{s}(z) + \lambda_{tk} p_{tk}(z))/2}} \,dz ,
\end{align} 
\begin{align}\label{sup-eq_optimal_tk}
\lambda_{tk}D_{KL}(p_{tk} \Vert p_{avg}) & =\lambda_{tk}  \int_{\mathcal{Z}}p_{tk}(z)\log {\frac{p_{tk}(z)}{p_{avg}(z)}} \,dz  \\
& =\lambda_{tk} \int_{\mathcal{Z}_{1}}p_{tk}(z)\log {\frac{p_{tk}(z)}{(p_{s}(z) + \lambda_{tk} p_{tk}(z))/2}} \,dz .
\end{align} 
since $p_{tu}(z)=0 \text{ for all } z \in \mathcal{Z}_{1}$ and $p_{s}(z)=p_{tk}(z)=0  \text{ for all }  z\in \mathcal{Z}_{2}$. Also, the third term in Eq. (\ref{sup-appendix_eq_optimal_G}) is as follows:
\begin{align}\label{sup-eq_optimal_tu}
\begin{split}
\lambda_{tu} D_{KL}(p_{tu}\Vert p_{avg}) & =\lambda_{tu}  \int_{\mathcal{Z}}p_{tu}(z)\log {\frac{p_{tu}(z)}{p_{avg}(z)}} \,dz  \\
& = \lambda_{tu} \int_{\mathcal{Z}_{2}}p_{tu}(z)\log {\frac{p_{tu}(z)}{(\lambda_{tu} p_{tu}(z))/2}} \,dz =  \lambda_{tu} \log\frac{2}{\lambda_{tu}},
\end{split}
\end{align} 

With Eq. (\ref{sup-eq_optimal_s}) to (\ref{sup-eq_optimal_tk}) and letting $C_{2} := \lambda_{tu}\cdot \log\frac{2}{\lambda_{tu}}$ in Eq. (\ref{sup-eq_optimal_tu}), Eq. (\ref{sup-appendix_eq_optimal_G}) is as below:
\begin{align*}
&\int_{\mathcal{Z}_{1}}p_{s}(z)\log {\frac{p_{s}(z)}{(p_{s}(z) + \lambda_{tk} p_{tk}(z))/2}} \,dz +\lambda_{tk} \int_{\mathcal{Z}_{1}}p_{tk}(z)\log {\frac{p_{tk}(z)}{(p_{tk}(z) + \lambda_{tk} p_{tk}(z))/2}} \,dz  - C_{2} \\
= &\int_{\mathcal{Z}_{1}}p_{s}(z)\log \Big[ {\frac{p_{s}(z)}{(p_{s}(z) + \lambda_{tk} p_{tk}(z))/(1+\lambda_{tk})}}\cdot \frac{2}{1+\lambda_{tk}} \Big] \,dz \\
 &+\lambda_{tk}  \int_{\mathcal{Z}_{1}}p_{tk}(z)\log \Big[{\frac{p_{tk}(z)}{(p_{tk}(z) + \lambda_{tk} p_{tk}(z))/(1+\lambda_{tk})}} \cdot\frac{2}{1+\lambda_{tk}}\Big] \,dz  -  C_{2} \\
=&\int_{\mathcal{Z}_{1}}p_{s}(z)\log {\frac{p_{s}(z)}{(p_{s}(z) + \lambda_{tk} p_{tk}(z))/(1+\lambda_{tk})}} \,dz+\int_{\mathcal{Z}_{1}}p_{s}(z)\log \frac{2}{1+\lambda_{tk}} \,dz \\
&+\lambda_{tk}  \int_{\mathcal{Z}_{1}}p_{tk}(z)\log {\frac{p_{tk}(z)}{(p_{tk}(z) + \lambda_{tk} p_{tk}(z))/(1+\lambda_{tk})}} \,dz +\lambda_{tk}  \int_{\mathcal{Z}_{1}}p_{tk}(z)\log \frac{2}{1+\lambda_{tk}} \,dz -  C_{2} \\
=&\int_{\mathcal{Z}_{1}}p_{s}(z)\log {\frac{p_{s}(z)}{(p_{s}(z) + \lambda_{tk} p_{tk}(z))/(1+\lambda_{tk})}} \,dz+\log \frac{2}{1+\lambda_{tk}} \\
&+\lambda_{tk}  \int_{\mathcal{Z}_{1}}p_{tk}(z)\log {\frac{p_{tk}(z)}{(p_{tk}(z) + \lambda_{tk} p_{tk}(z))/(1+\lambda_{tk})}} \,dz +\lambda_{tk} \log \frac{2}{1+\lambda_{tk}}  -  C_{2} \\
\end{align*} 
With denoting $\alpha:=\frac{\lambda_{tk}}{1+\lambda_{tk}}$, and $C_{3}:=\log \frac{2}{1+\lambda_{tk}}+\lambda_{tk}\log \frac{2}{1+\lambda_{tk}}  -C_{2}$, and satisfying $0<\alpha<1$,
\begin{align}\label{sup-appendix_eq_optimal_G2}
=\int_{\mathcal{Z}_{1}}\hspace{-0.1em}p_{s}(z)\log {\frac{p_{s}(z)}{(1-\alpha) p_{s}(z) + \alpha p_{tk}(z)}} \,dz  +\lambda_{tk}\hspace{-0.1em}  \int_{\mathcal{Z}_{1}}\hspace{-0.1em}p_{tk}(z)\log {\frac{p_{tk}(z)}{(1-\alpha) p_{tk}(z) + \alpha p_{tk}(z)}} \,dz + C_{3}.
\end{align} 
By the definition of the skewed $\alpha$-KL Divergence ($D^{(\alpha)}_{KL}$) \cite{yamano2019some}, Eq. (\ref{sup-appendix_eq_optimal_G2}) is written as follow:
\begin{align} \label{sup-eq_skew_kl}
 D^{(\alpha)}_{KL}(p_{s} \Vert p_{tk}) + \lambda_{tk}\cdot D^{(1-\alpha)}_{KL}(p_{tk} \Vert p_{s})+C_{3}.
\end{align}
The skewed $\alpha$-KL Divergence,  $D^{(\alpha)}_{KL}(p\Vert q)$, belongs to the $f$-divergence from $p$ to $q$ \cite{yamano2019some}.
\begin{align*}
D_{f}(p\Vert q) = \int q(x)f \Big(\frac{p(x)}{q(x)}\Big)dx, \ where \ f(u)=u \log \frac{u}{(1-\alpha)u+\alpha}, \ (u=\frac{p(x)}{q(x)}\neq1),
\end{align*} 
where $f(u)$ is a convex function with $f(1)=0$. Therefore, Eq. (\ref{sup-eq_skew_kl}) is equivalent to the summation of $f$-divergence as below.
\begin{align} \label{sup-appendix_eq_f_divergence}
D_{f_1}(p_{s}\Vert p_{tk}) + \lambda_{tk} D_{f_2}(p_{tk}\Vert p_{s})+C_{3},
\end{align}
where $f_{1}(u)=u \log \frac{u}{(1-\alpha)u+\alpha}$, and $f_{2}(u)=u \log \frac{u}{\alpha u+(1-\alpha)}$. Therefore, the minimum of Eq. (\ref{sup-appendix_eq_f_divergence}) is achieved when $p_{s}=p_{tk}$. 
\end{proof}

\subsection{Posterior Inference} \label{sup-bmm}
We provide the details of the posterior inference to estimate $w_x=p(\textit{known}|x)$ for a target instance, $x$. Thus, we model the mixture of two Beta distributions on the entropy values of the target instances. We estimate $w_x$ as the posterior probability by fitting the Beta mixture model through the Expectation-Maximization (EM) algorithm. Therefore, this section starts the details of the fitting process of the Beta mixture model.

\subsubsection{Fitting Process of Beta mixture model} \label{sup-BMMem}
We follow the fitting process of Beta mixture model by \cite{arazo2019unsupervised}. First, the probability density function (pdf) for the mixture of two Beta distributions on the entropy values is defined as follows,
\begin{align} 
p(\ell_{x})&=\lambda_{tk} p(\ell_{x}|known) + \lambda_{tu} p( \ell_{x}|unknown),\\
\text{with}\quad p(\ell_{x}|known) &\sim Beta(\alpha_{0}, \beta_{0}) \quad \text{and} \quad p(\ell_{x}|unknown) \sim Beta(\alpha_{1}, \beta_{1}), \label{sup-individual_pdf}
\end{align} 
where $\lambda_{tk}$ is $p(known)$; $\lambda_{tu}$ is $p(unknown)$; $\ell_{x}$ is the entropy value for the target instance, $x$, i.e. $\ell_{x}=H(E(G(x)))$ with entropy function $H$; $\alpha_{0}$ and $\beta_{0}$ represents the parameters of the Beta distribution for the \textit{known} component; and $\alpha_{1}$ and $\beta_{1}$ are the parameters for \textit{unknown} component.
Eq. (\ref{sup-individual_pdf}) represents the individual pdf for each component which is followed by the Beta distribution,

We fit the distribution through the Expectation-Maximization (EM) algorithm. 
We introduce the latent variables $\gamma_{0}(\ell_{x})=p(known|\ell_{x})$ and $\gamma_{1}(\ell_{x})=p(unknown|\ell_{x})$, and use an Expectation Maximization (EM) algorithm with a finite number of iterations (10 in ours). 

In E-step, we update the latent variables using Bayes' rule with fixing the other parameters, $\lambda_{tk}$, $\alpha_{0}$, $\beta_{0}$, $\lambda_{tu}$, $\alpha_{1}$, and $\beta_{1}$, as follows:
\begin{equation*}
\gamma_{0}(\ell_{x})= \frac{\lambda_{tk}p(\ell_{x} | known)}{ \lambda_{tk}p(\ell_{x} | known)+\lambda_{tu}p(\ell_{x} | unknown)},
\end{equation*}
where $p(\ell_{x}|known)$ and $p(\ell_{x}|unknown)$ are from Eq. (\ref{sup-individual_pdf}). $\gamma_{1}(\ell_{x})$ follows the same claculation.

In M-step, given the fixed $\gamma_{0}(\ell_{x})$ and $\gamma_{1}(\ell_{x})$ from the E-step, the parameters $\alpha_{k}$, $\beta_{k}$ are estimated by using a weighted method of moments as follows,
\begin{equation*}
\beta_{k} = \frac{\alpha_{k}(1-\bar{\ell}_{k})}{\bar{\ell}_{k}}, \quad  \alpha_{k} =\bar{\ell}_{k} (\frac{\bar{\ell}_{k}(1-\bar{\ell}_{k})}{s_{k}^{2}}-1),  \quad \text{where} \ k \in \left\{0, 1 \right\},
\end{equation*}
where $\bar{\ell}_{0}$ and $s_{0}^{2}$ are a weighted average and a weighted variance estimation of the entropy values, $\ell_{x}$, for \textit{known} component, respectively. $\bar{\ell}_{1}$ and $s_{1}^{2}$ are for \textit{unknown} component as follows, 
\begin{equation*}
\bar{\ell}_{k} = \frac{\sum_{x\in \chi_{t}} \gamma_{k} (\ell_{x}) \ell_{x} }{\sum_{x\in \chi_{t}} \gamma_{k} (\ell_{x})},\quad s_{k}^{2} = \frac{\sum_{x\in \chi_{t}} \gamma_{k} (\ell_{x}) (\ell_{x} -\bar{\ell}_{k} )^{2} }{\sum_{x\in \chi_{t}}\gamma_{k} (\ell_{x})}\ \text{where} \ k \in \left\{0, 1 \right\}.
\end{equation*}

Then, the mixing coefficients, $\lambda_{tk}$ and $\lambda_{tu}$, are calculated as follows,
\begin{equation}\label{sup-eqn:lambda_tk_tu}
\lambda_{tk} = \frac{1}{n_{t}}\sum_{x\in \chi_{t}} \gamma_{0}(\ell_{x}), \quad \lambda_{tu} = 1- \lambda_{tk},
\end{equation}
where $n_{t}$ is the number of instances in the target domain, $\chi_{t}$.

We conduct a finite number of iteration over E-step and M-step iteratively. Finally, the probability of a instance being \textit{known} or \textit{unknown} class through the posterior probability:
\begin{equation}
p(known|\ell_{x}) =\frac{\lambda_{tk}p(\ell_{x}|known)}{\lambda_{tk} p(\ell_{x}|known) + \lambda_{tu} p( \ell_{x}|unknown)},
\end{equation}
where $p(unknown|\ell_{x})$ follows the same calculation. 

\subsection{Training Details}\label{sup-appendix:trainig_algorithm}
This subsection provides the details for the training part of UADAL. The first part enumerates the training algorithm procedure, and the second part shows the computational complexity of UADAL during training.
\subsubsection{Training Algorithm of UADAL}\label{sup-appendix:trainig_algorithm_detail}
We provide a training algorithm of UADAL in detail. 
All equations in Algorithm \ref{sup-alg:algorithm1} of this script represent the equations in the main paper. 
The detailed settings for $n_{iter}$, $m$, and $\eta$ in Algorithm \ref{sup-alg:algorithm1} are described in Section \ref{sup-appendix:hyperparameter}.

\newcommand{\factorial}{\ensuremath{\mbox{\sc Factorial}}}
\begin{algorithm}[h]
\caption{Training algorithm of UADAL.}\label{sup-alg:algorithm1}
\begin{algorithmic}[1]
\Require
\Statex $\chi_{s}$: dataset from the source domain.
\Statex $\chi_{t}$: dataset from the target domain.
\Statex $n_{iter}$: the number of epochs for main training.
\Statex $m$: the batch size.
\Statex $\eta$: the frequency to fit the posterior inference.
\Ensure
\State Sample few source minibatches from $\chi_{s}$.
\State Update $\theta_{e}$, $\theta_{g}$ following Eq. (\ref{eqn:e_update}). 
\State {Fitting the posterior model through EM algorithm by Eq. (\ref{eqn:posterior})}
\For {$i =1,\ldots,n_{iter}$}
\State Sample minibatch of $m$ source samples from $\chi_{s}$. 
\State Sample minibatch of $m$ target samples from $\chi_{t}$. 
\State Update $\theta_{d}$ by Eq. (\ref{eqn:d_update}). 
\State Update $\theta_{e}$, $\theta_{g}$, $\theta_{c}$ by Eq. (\ref{eqn:e_update}, \ref{eqn:g_update}). 
\If{(i mod $\eta$) = 0}
\State {Fitting the posterior model through EM algorithm by Eq. (\ref{eqn:posterior})}
\EndIf
\EndFor
\end{algorithmic}
\end{algorithm}

\subsubsection{Computational Complexity of UADAL}\label{sup-appendix:computational_complexity}
The computational complexity of UADAL is increased because we need to fit the mixture model, which requires $O(nk)$, with the number of the target instances ($n$) and the iterations of EM ($k$). 
Figure \ref{sup-appendix:convergence_of_bmm} shows the negative log-likelihood from the fitted mixture model over EM iterations ($k$), along different initializations of $\lambda_{tk}$ and $\lambda_{tu}$.
Here, $k$ can be adjusted to trade the performance (negative log-likelihood) and the time-complexity ($k$). Moreover, it shows that the convergence of EM algorithm of the posterior inference, which makes we set $k$ as constant. 
As a measure of the computational complexity, we provide \textbf{Wall-clock-time} for a whole experimental procedure by following Algorithm \ref{sup-alg:algorithm1} (under {an RTX3090 GPU/i7-10700F CPU}). 
For A$\rightarrow$W in Office-31, the wall-clock-time is 1,484 and 1,415 seconds, with and without posterior inference, respectively (+5\% increment). 
From this +5\% increment, the performance of UADAL with the posterior inference has improved than with Entropy, as shown in Figure \ref{fig:entropy_trhesholding} in the main paper (see the red/blue solid lines).
\begin{figure}[h]\centering
    \includegraphics[width=0.6\linewidth]{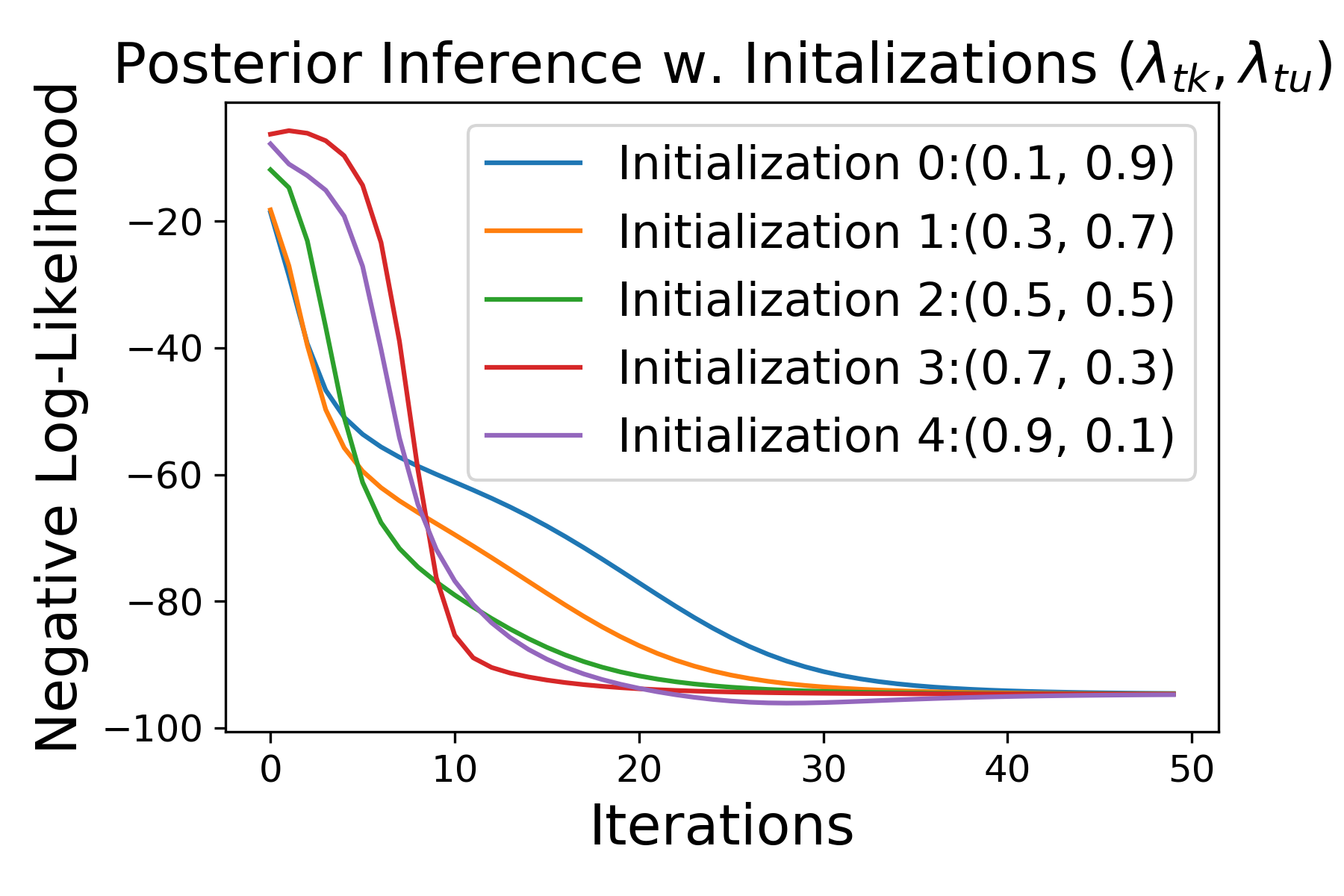}
        \caption{Convergence over EM iterations ($k$) of Posterior Inference} \label{sup-appendix:convergence_of_bmm}
\end{figure}

\newpage
\section{Experimental Part} \label{sup-exp_section}
\subsection{Implemenation details} \label{sup-appendix_imple_detail}
\subsubsection{Optimization Details}\label{sup-appendix_optim_detail}

We utilize the pre-trained ResNet-50 \cite{He_2016_CVPR}, DenseNet-121 \cite{huang2017densely}, EfficientNet-B0 \cite{tan2019efficientnet}, and VGGNet \cite{Simonyan15}, as a backbone network. 
For all cases of the experiments for the backbone networks and the datasets, we use the SGD optimizer with the cosine annealing \cite{loshchilov2016sgdr} schedule for the learning rate scheduling. 
For the parameters in the pre-trained network of ResNet-50, DenseNet, and EfficientNet, we set the learning rate 0.1 times smaller than the parameters from the scratch, followed by \cite{liu2019separate,ROS2020ECCV}. 
For VGGNet with VisDA dataset, we followed \cite{saito2018open}. Therefore, we did not update the parameters of VGGNet and constructed fully-connected layers with 100 hidden units after the FC8 layers. 
In terms of the entropy loss for the target domain, we adopt a variant of the loss, FixMatch \cite{sohn2020fixmatch}, in order to utilize the confident predictions of the target instances.
We run each setting \textbf{three times} and report the averaged accuracy with standard deviation. 
We conduct all experiments on an NVIDIA RTX 3090 GPU and an i7-10700F CPU. 

\subsubsection{Network Configurations} Except for the feature extractor network $G$, the configurations of the other networks, $E$, $C$, $D$, are based on the classification network. 
The feature dimensions from the ResNet-50, EfficientNet-B0, and DenseNet-121 are 2048, 1280, and 1024, respectively, which are the output dimensions of $G$ (100 for VGGNet by the construction of the fully connected layer). 
Given the dimension, the network $C$ utilizes one middle layer between the feature and classification layer with the dimension of 256. 
After the middle layer, we apply a batch normalization \cite{ioffe2015batch} and LeakyReLU \cite{xu2015empirical} with 0.2 parameter. 
Then, the feature after the middle layer passes to the classification layer, in which the output dimensions are $|\mathcal{C}_s|+1$. 
The network $E$ has the classification layer without any middle layer, in which the output dimension is $|\mathcal{C}_s|$.
The network $D$ consists of the two middle layers with applying the LeakyReLU, where the output dimension is 3. 

\subsubsection{Hyperparameter Settings}\label{sup-appendix:hyperparameter}

For all cases of the experiments, we set the batch size, $m$, to 32. 
For the number of epochs for main training, denoted as $n_{iter}$, we set 100 for Office-31 and OfficeHome datasets. 
For the VisDA dataset, we set  $n_{iter}$ as 10 epochs since it is a large scale. 
Associated with  $n_{iter}$, we set the frequency to fit the posterior inference, $\eta$, as 10 epochs for Office-31 and OfficeHome, as 1 epoch for the VisDA dataset.
We set 0.001 as a default learning rate with 0.1 times smaller for the network $G$ since we bring the pre-trained network. For the network $E$ of the open-set recognizer, we set 2 times larger than the default value because the shallow network $E$ should learn the labeled source domain quickly during $\eta$ epochs, followed by initializing the network $E$ after fitting the posterior inference. 
For the case of VGGNet with the VisDA dataset, we utilize the default learning rate for the network $G$ as same as the other networks since we only train the fully connected layers on the top of VGGNet.

\subsubsection{Baselines} \label{sup-appendix_baseline}
For the baselines, we implement the released codes for OSBP \cite{saito2018open} (\url{https://github.com/ksaito-ut/OPDA_BP}), STA \cite{liu2019separate} (\url{https://github.com/thuml/Separate_to_Adapt}), PGL \cite{luo2020progressive} (\url{https://github.com/BUserName/PGL}), ROS \cite{ROS2020ECCV} (\url{https://github.com/silvia1993/ROS}), DANCE \cite{DANCE2020NIPS} (\url{https://github.com/VisionLearningGroup/DANCE}), and DCC \cite{li2021domain} (\url{https://github.com/Solacex/Domain-Consensus-Clustering}). For OSLPP \cite{wang2021progressively}, we are not able to find the released code. 
Thus, the reported performances of OSLPP for Office-31 and Office-Home with ResNet-50 are only available. 
For the released codes, we follow their initial experimental settings. Especially, we set all experimental settings for DenseNet and EfficientNet, equal to the settings on their ResNet-50 experiments, since they do not conduct the experiments on the DenseNet-121 and EfficientNet-B0. 
For a \textbf{fair comparison}, we bring the reported results for the baselines from its papers on the datasets, i.e., Office-31 (with ResNet-50), OfficeHome (with ResNet-50), and VisDA(with VGGNet) dataset. The officially reported performances are marked as $^{*}$ in the tables.
Except for these cases, we all re-implement the experiments three times. 

\subsubsection{Dataset} \label{sup-appendix_dataset}
For the availability of the datasets, we utilize the following links; Office-31 (\url{https://www.cc.gatech.edu/~judy/domainadapt/#datasets_code}), Office-Home (\url{https://www.hemanthdv.org/officeHomeDataset.html}), and VisDA (\url{http://ai.bu.edu/visda-2017/#download}). We utilize the data transformations for training the proposed model, which are 1) resize, random horizontal flip, crop, and normalize by following \cite{DANCE2020NIPS}, and 2) RandAugment \cite{cubuk2020randaugment} by following \cite{prabhu2021sentry}.

\subsection{Experimental Results} \label{sup-appendix_result_analysis}
\subsubsection{Computational Complexity of OSLPP on VisDA dataset} \label{sup-appendix_oslpp}
As OSLPP \cite{wang2021progressively} said, their complexity is $\mathcal{O}(T(2n^{2}d_{PCA}+{d^{3}}_{PCA}))$, which is repeated for $T$ times. Here, $n$ is the number of samples with $n=n_s+n_t$, and $d_{PCA}$ is the dimension which is reduced by PCA. However, regarding memory usage, when the number of samples, $n$, is much greater than the dimensionality, the memory complexity is $\mathcal{O}(n^2)$. Therefore, they claimed that it has limitation of scaling up to the extremely large dataset (e.g., $n>100,000$). With this point, VisDA dataset consists of the source dataset with 79,765 instances and the target dataset with 55,388 instances, where the number of samples becomes 135,153. Therefore, OSLPP is infeasible to conduct the experiments for VisDA dataset.

\subsubsection{Low Accuracy of Baselines}\label{sup-appendix:low_accuracy_results}
{HOS} score is a harmonic mean of OS$^{*}$ and UNK. Therefore, HOS is higher when performing well in \textit{both} known and unknown classification. 
With this point, some baselines have very low {HOS} score in the Table \ref{tab:hos_office} and \ref{tab:res:visda} of the main paper. 
This is because their UNK performances are worse. 
For example, the reported OS and OS$^{*}$ of PGL \cite{luo2020progressive} in Office-Home are 74.0 and 76.1, respectively. 
Here, OS is the class-wise averaged accuracy over the classes including \textit{unknown} class.
With 25 known classes, {UNK} then becomes (25+1)$\times \text{OS}-$25$\times${OS*} = 25.1, which leads to {HOS} score as 33.5. 
It means that PGL fails in the open-set scenario because their adversarial loss includes \textit{all} target instances, which is critically weakened by the negative transfer. 

For DANCE, they only reported OS scores in the paper, which makes the calculation of HOS infeasible. 
Also, their class set (15 knowns in OfficeHome) is different from the standard OSDA scenario (25 knowns in OfficeHome by following \cite{saito2018open}).
It means that optimal hyper-parameters are not available. 
Therefore, we re-implement based on their official code.
Meanwhile, DANCE (also DCC) is based on clustering which means that they are weak on initializations or hyperparameters, empirically shown as higher standard deviations of the performances in the tables. 
The below is the detailed answer on the lower performances of baselines, especially DANCE, with EfficientNet. As we said, there is no reported performance of DANCE with the additional backbone choices. Therefore, we implemented additional variants of DANCE with EfficientNet and DenseNet by following their officially released codes. In order to compare fairly, we set all hyper-parameters with that of ResNet-50 case as UADAL is being set. Specifically, DANCE requires a threshold value ($\rho$) to decide whether a target instance belongs to “known” class or not, which is very sensitive to the performance. We confirmed that they utilize the different values over the experimental settings. This sensitivity may degrade the performance of DANCE. Unlike DANCE, UADAL does not require a threshold setting because it has a posterior inference to automatically find the threshold to decide open-set instances. Therefore, this becomes the key reason behind the performance difference.

For DCC, the experimental settings for VisDA are not available with the comments of ``the clustering on VisDA is not very stable'' in their official code repository. 
Threfore, our re-implementation of DCC with VisDA (with EfficientNet-B0, DenseNet-121, and ResNet50) was also unstable. 
For a fair comparison with DCC, please refer to the performances which is marked as $^{*}$ in the Table \ref{tab:hos_office} and \ref{tab:res:visda} of the main paper.

\subsubsection{t-SNE Visualization} \label{sup-appendix_tsne}
Figure \ref{sup-appendix_fig:tsne_effinet} and \ref{sup-appendix_fig:tsne_resnet} in this script represents the t-SNE visualizations of the learned features extracted by EfficientNet-B0 and ResNet-50, respectively. It should be noted that EfficientNet-B0 (5.3M) has only 20\% of parameters than ResNet-50 (25.5M). We observe that the target-unknown (red) features from the baselines are not discriminated with the source (green) and the target-known (blue) features. On the contrary, UADAL and cUADAL align the features from the source and the target-known instances accurately with clear \textit{segregation} of the target-unknown features. It means that UADAL learns the feature spaces effectively even in the less complexity.

\begin{figure*}[h]
\centering
\resizebox{\textwidth}{!}{%
\begin{subfigure}[t]{0.174\textwidth}
    \includegraphics[width=\textwidth]{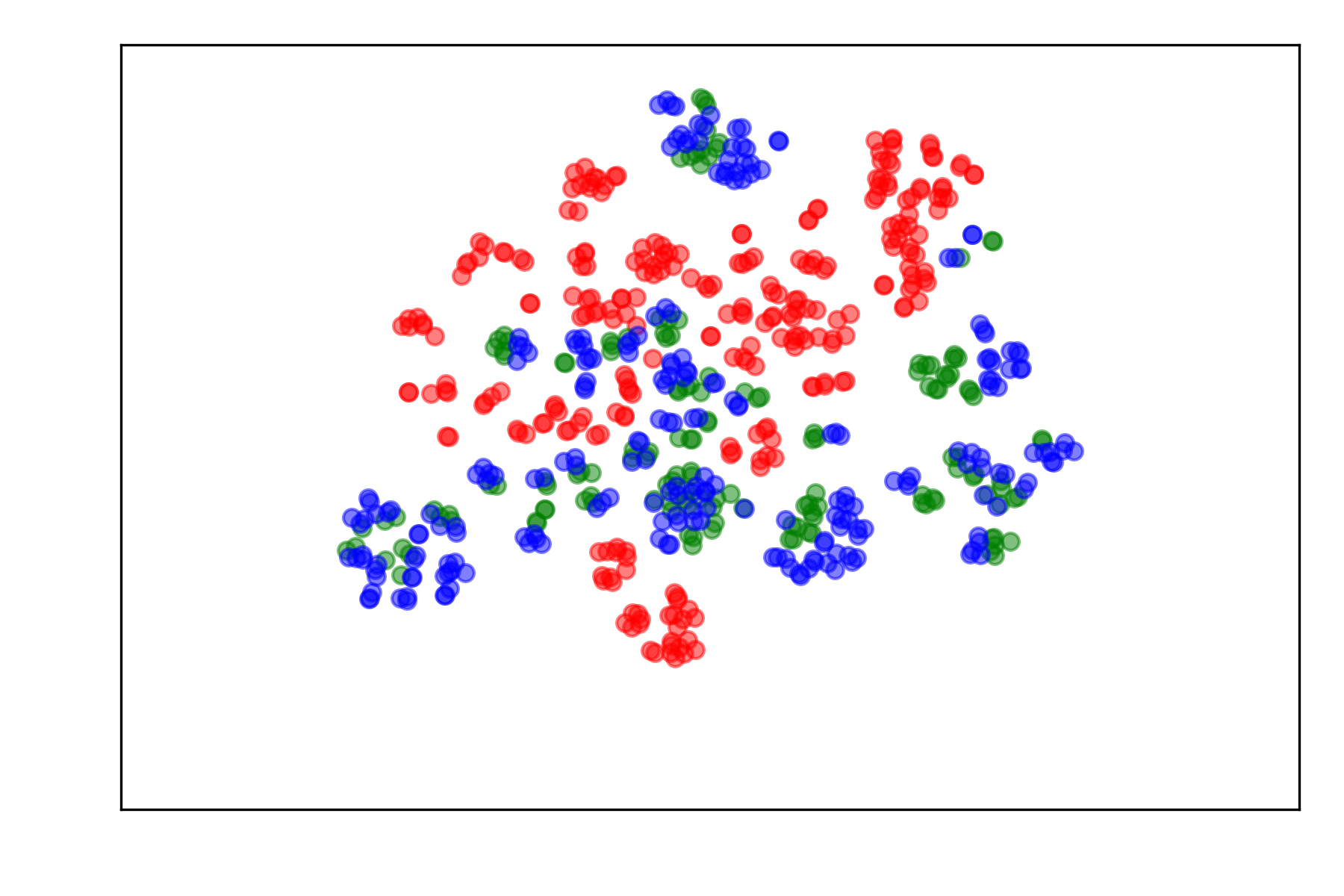}
    \caption{\small{DANN}}
    \label{sup-fig:tsne_dann_effi}
\end{subfigure} \hspace{-0.2em}%
\begin{subfigure}[t]{0.174\textwidth}
    \includegraphics[width=\textwidth]{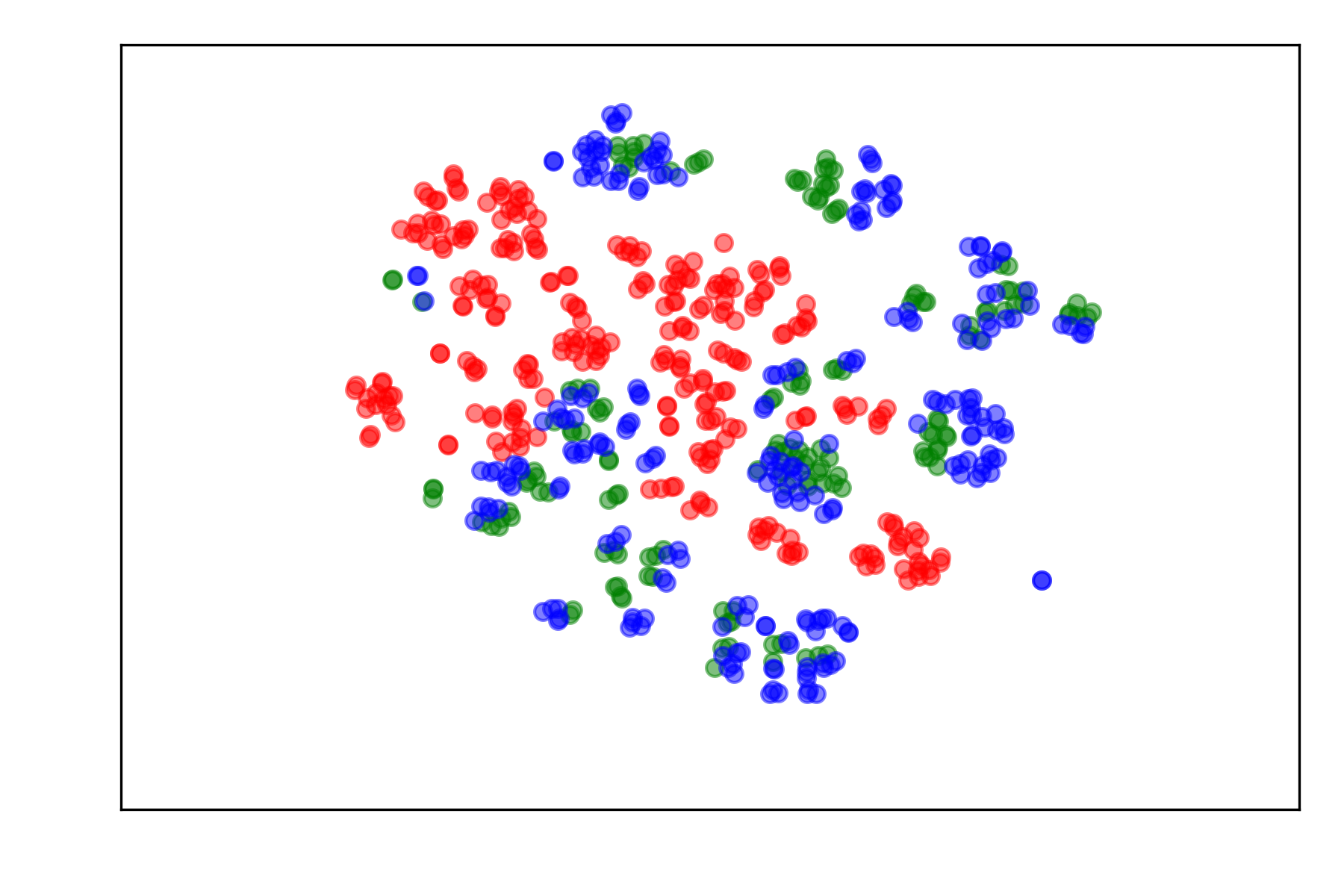}
    \caption{\small{STA}}
    \label{sup-fig:tsne_sta_effi}
\end{subfigure} \hspace{-0.2em}%
\begin{subfigure}[t]{0.174\textwidth}
    \includegraphics[width=\textwidth]{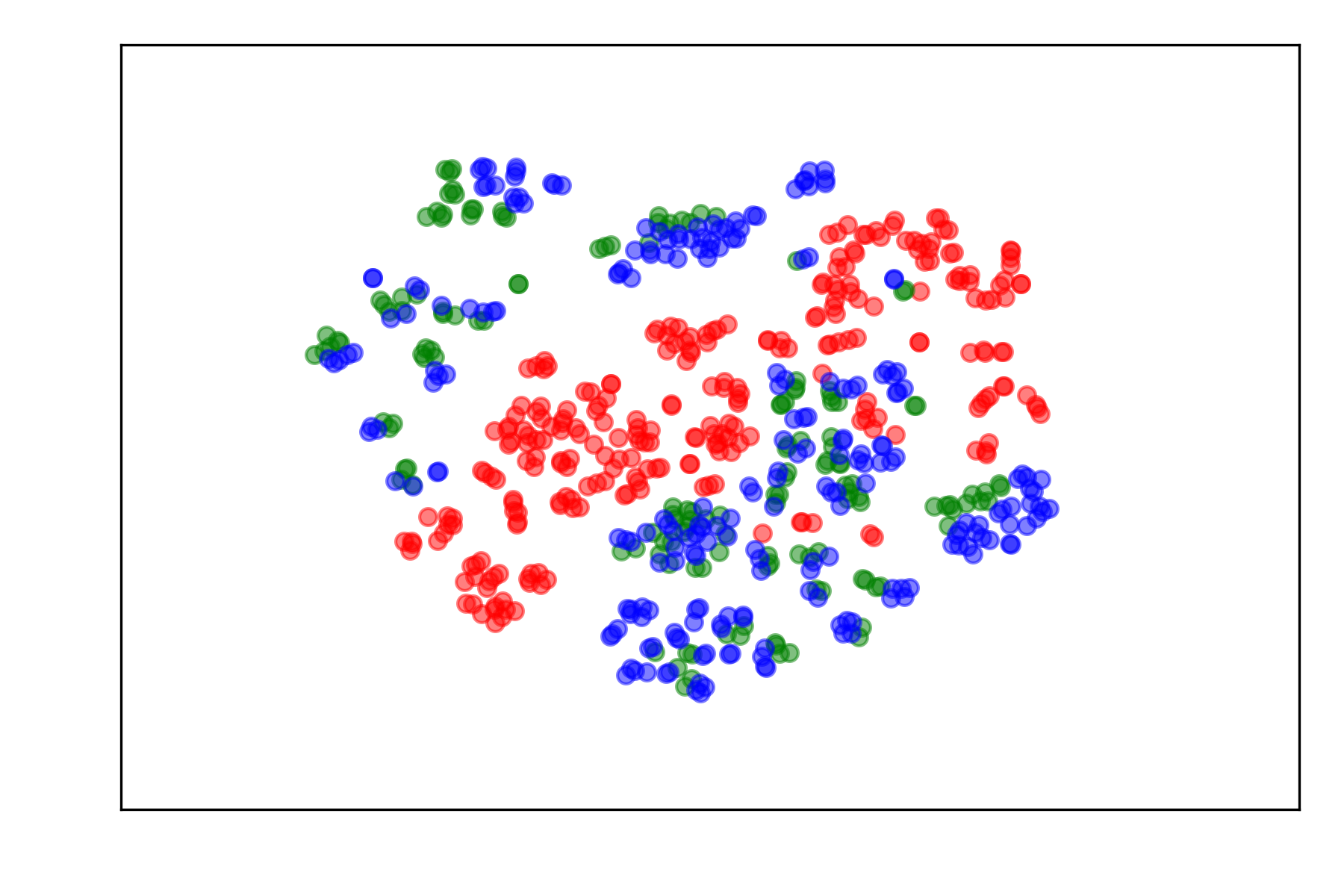}
    \caption{\small{OSBP}}
    \label{sup-fig:tsne_2d_effi}
\end{subfigure}\hspace{-0.1em}%
\begin{subfigure}[t]{0.174\textwidth}
    \includegraphics[width=\textwidth]{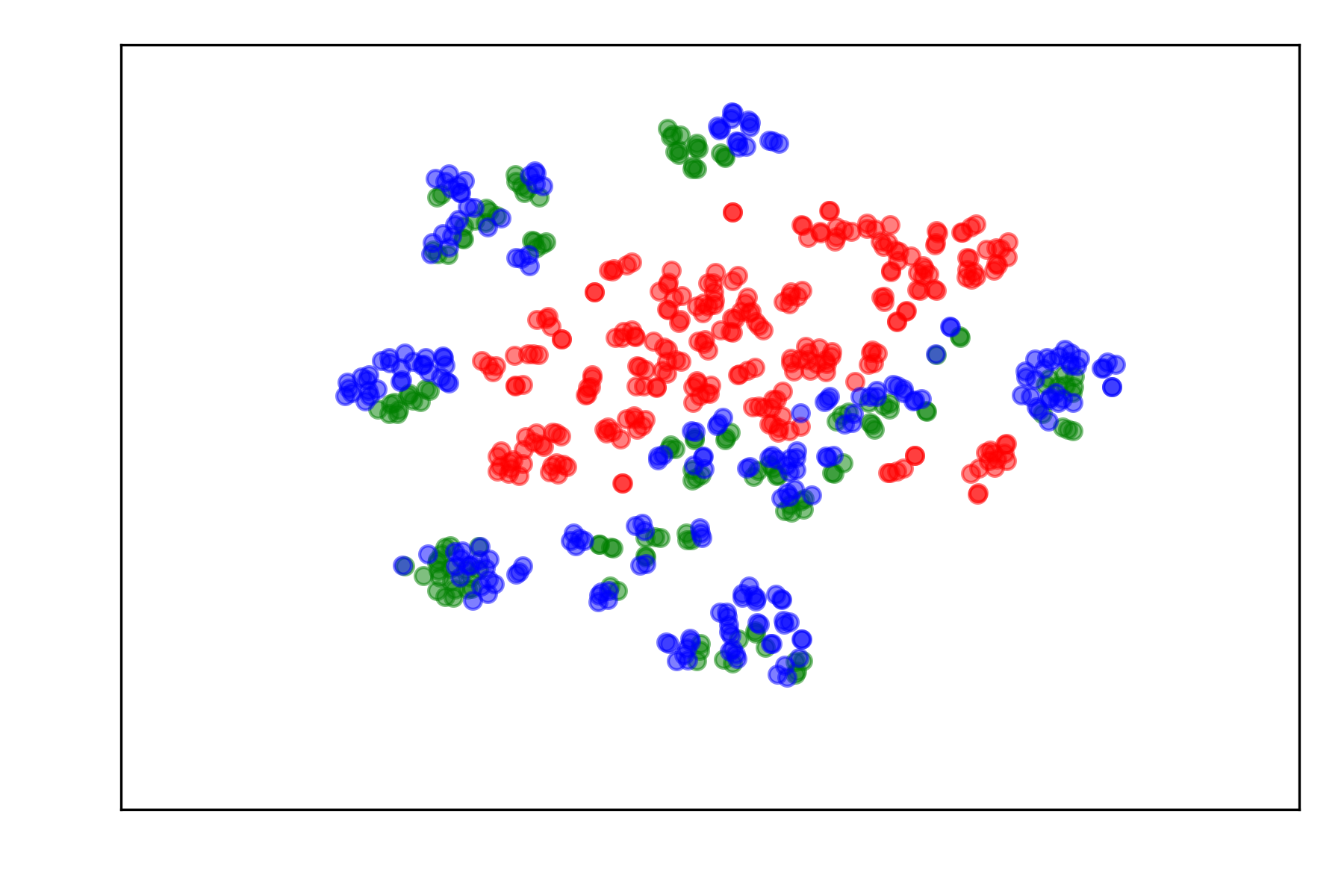}
    \caption{\small{DCC}}
    \label{sup-fig:tsne_dance_effi}
\end{subfigure} \hspace{-0.2em}%
\begin{subfigure}[t]{0.174\textwidth}
    \includegraphics[width=\textwidth]{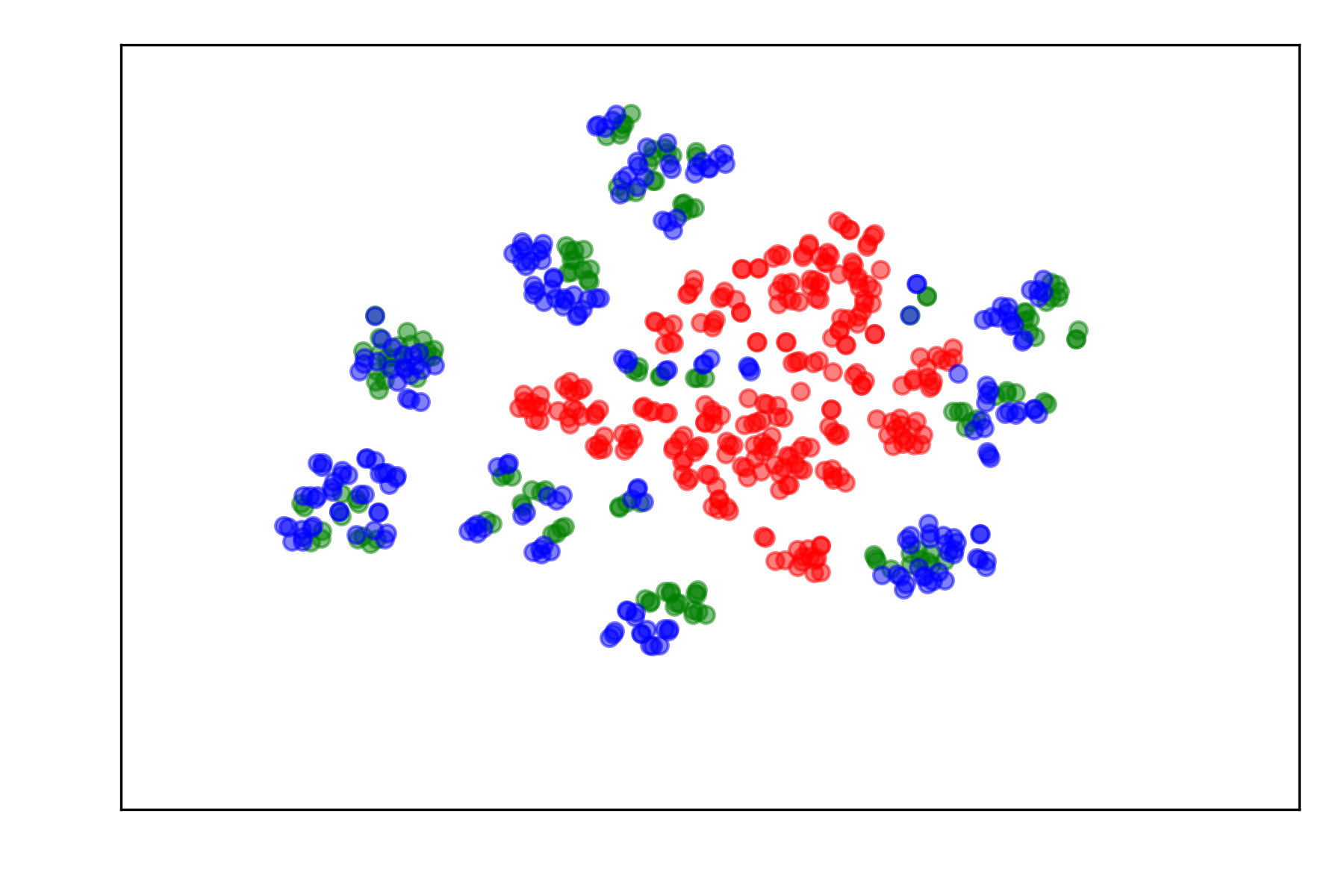}
    \caption{\small{UADAL}}
    \label{sup-fig:tsne_uadal_effi}
\end{subfigure} \hspace{-0.2em}%
\begin{subfigure}[t]{0.174\textwidth}
    \includegraphics[width=\textwidth]{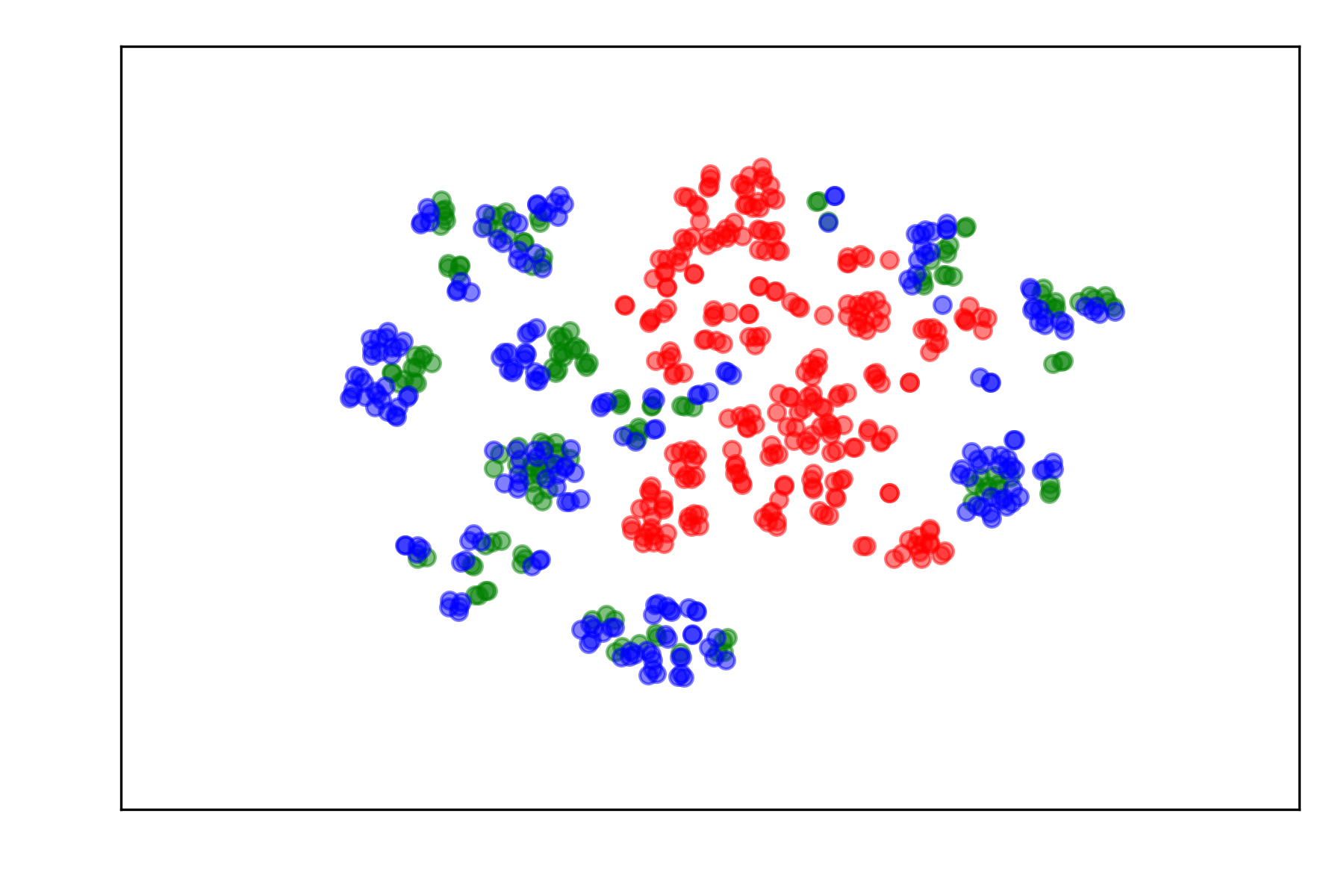}
    \caption{\small{cUADAL}}
    \label{sup-fig:tsne_cuadal_effi}
\end{subfigure} \hspace{-0.2em}%
}
\caption{{t-SNE of the features by the EfficientNet-B0 on the task D $\rightarrow$ W of Office-31. (blue: target-known, red: target-unkonwn, green: source)}}
\label{sup-appendix_fig:tsne_effinet}
\end{figure*}

\begin{figure*}[h]
\centering
\resizebox{\textwidth}{!}{%
\begin{subfigure}[t]{0.16\textwidth}
    \includegraphics[width=\textwidth]{figs/Tsne_GRL_resnet50_D_W_0.png}
    \caption{\small{DANN}}
    \label{sup-fig:tsne_dann}
\end{subfigure} \hspace{-0.2em}%
\begin{subfigure}[t]{0.16\textwidth}
    \includegraphics[width=\textwidth]{figs/Tsne_STAmax_resnet50_D_W_0.png}
    \caption{\small{STA}}
    \label{sup-fig:tsne_sta}
\end{subfigure} \hspace{-0.2em}%
\begin{subfigure}[t]{0.16\textwidth}
    \includegraphics[width=\textwidth]{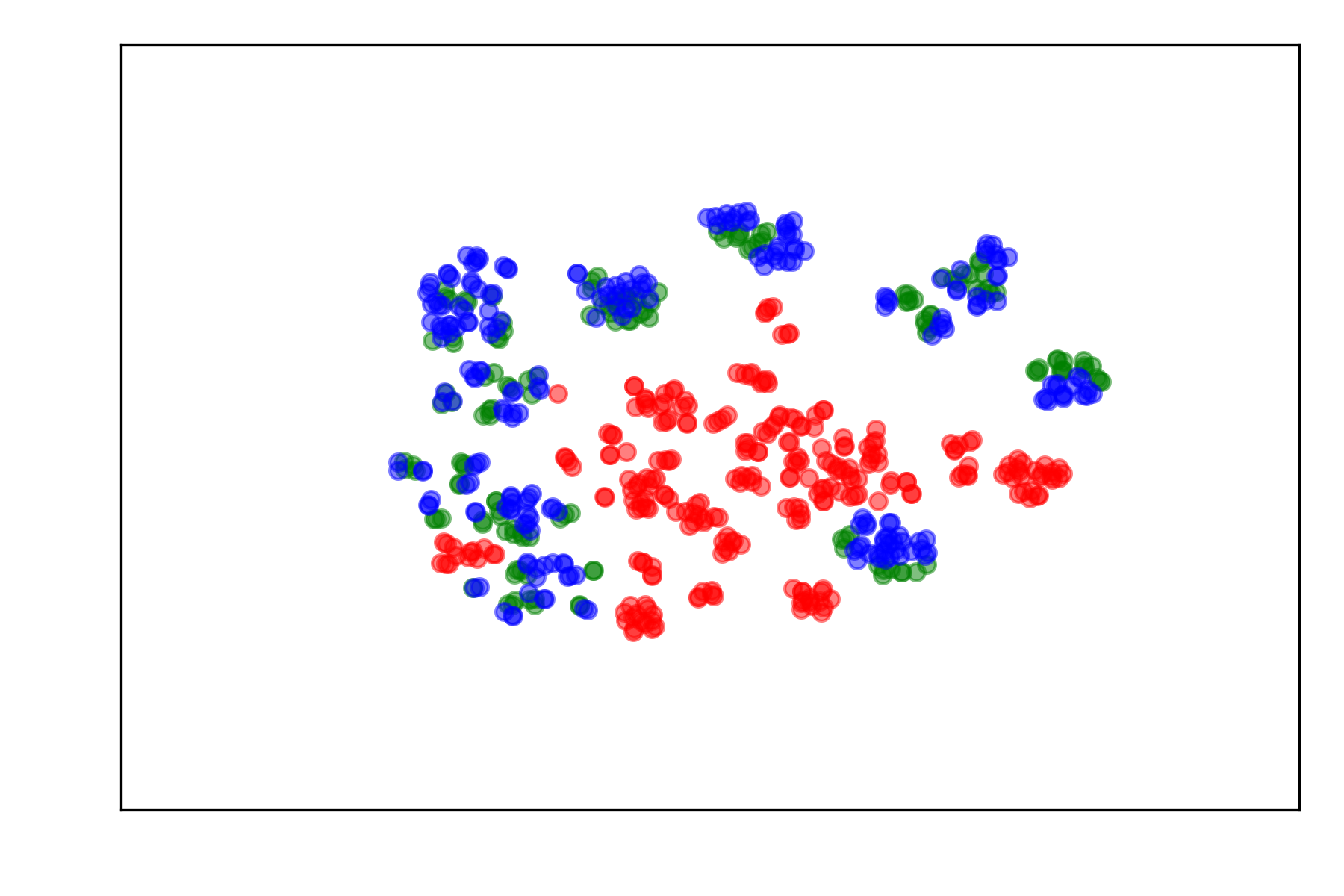}
    \caption{\small{OSBP}}
    \label{sup-fig:tsne_2d}
\end{subfigure}\hspace{-0.1em}%
\begin{subfigure}[t]{0.16\textwidth}
    \includegraphics[width=\textwidth]{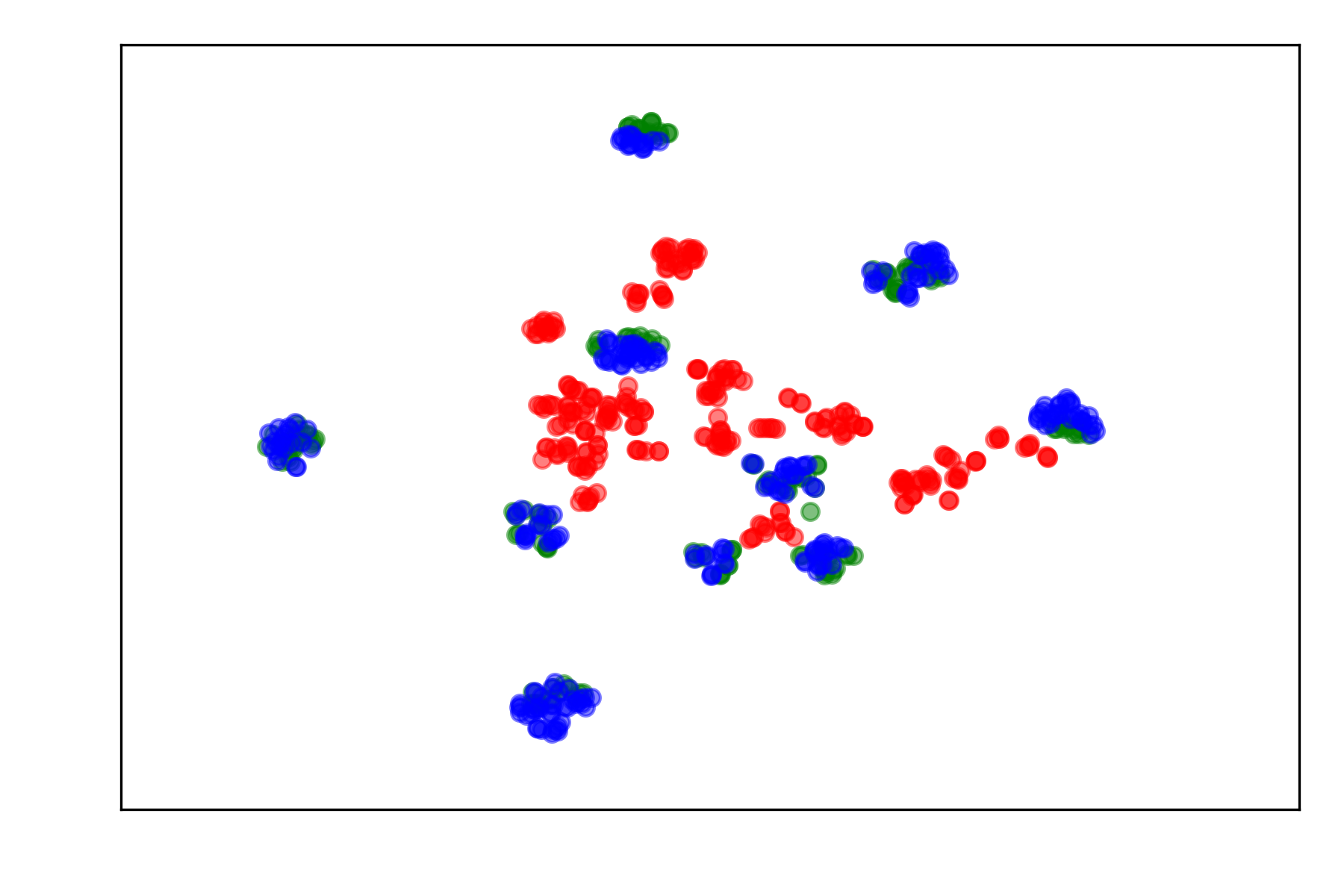}
    \caption{\small{DCC}}
    \label{sup-fig:tsne_dance}
\end{subfigure} \hspace{-0.2em}%
\begin{subfigure}[t]{0.16\textwidth}
    \includegraphics[width=\textwidth]{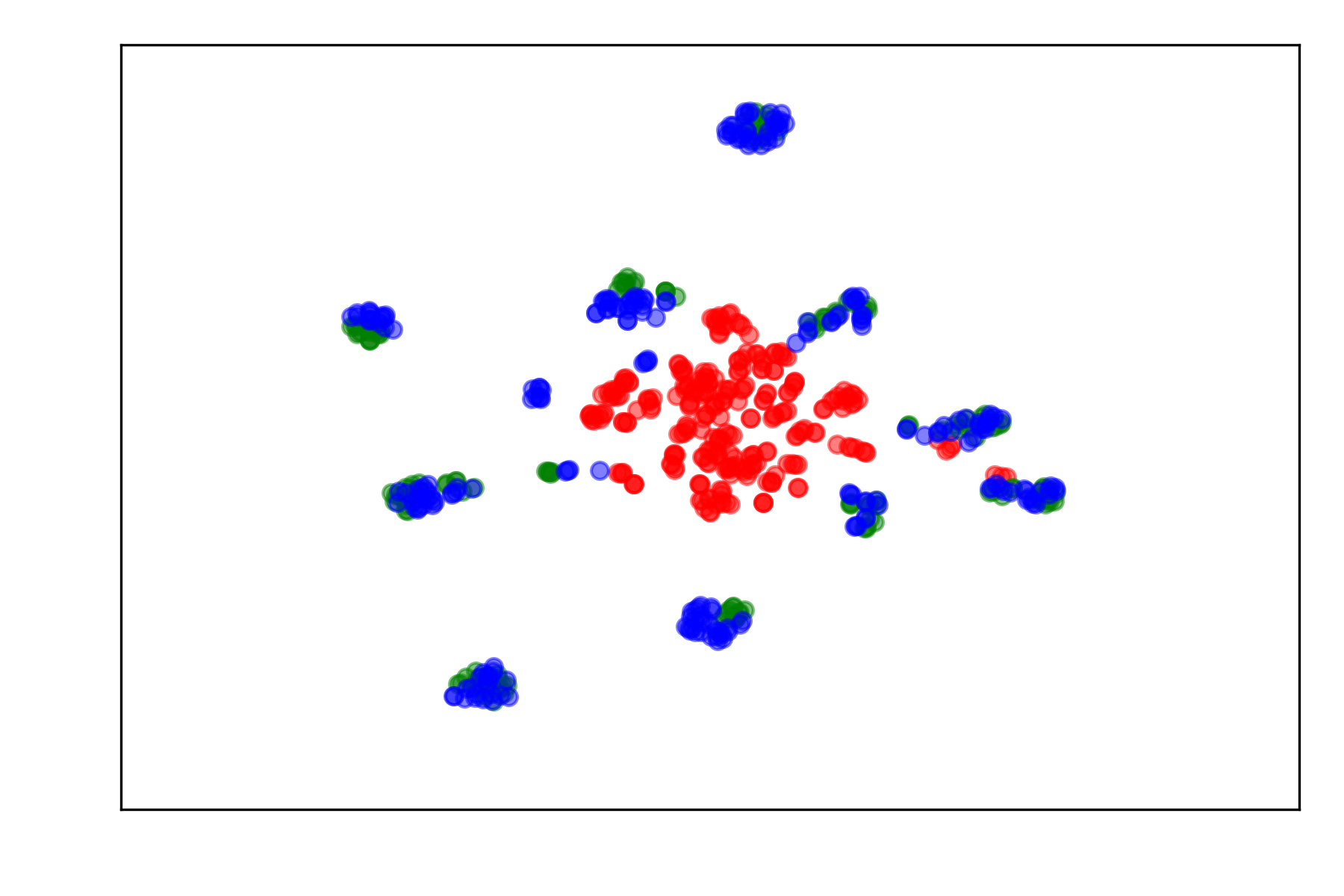}
    \caption{\small{UADAL}}
    \label{sup-fig:tsne_uadal}
\end{subfigure} \hspace{-0.2em}%
\begin{subfigure}[t]{0.16\textwidth}
    \includegraphics[width=\textwidth]{figs/Tsne_cUADAL_resnet50_D_W_0.png}
    \caption{\small{cUADAL}}
    \label{sup-fig:tsne_uadal}
\end{subfigure} \hspace{-0.2em}%
}
\caption{\small{t-SNE of the features by the ResNet-50 on the task D $\rightarrow$ W of Office-31. (blue: target-known, red: target-unkonwn, green: source) }}
\label{sup-appendix_fig:tsne_resnet}
\end{figure*}

\subsubsection{Proxy $\mathcal{A}$-Distance (PAD) } \label{sup-appendix_proxy}
Proxy $\mathcal{A}$-Distance (PAD) is an empirical measure of distance between domain distributions, which is proposed by \cite{ganin2016domain}. Given a generalization error $\epsilon$ of discriminating data which sampled by the domain distributions, PAD value can be computed as $\hat{d}_{\mathcal{A}}=2(1-2\epsilon)$. We compute the PAD value between target-known and target-unknown features from the feature extractor, $G$. We follow the detailed procedure in \cite{ganin2016domain}. Note that high PAD value means two domain distributions are well discriminated.


\subsubsection{Robust on Early Stage Iterations}\label{sup-appendix:robust_ws}
We investigate the effects of the learning iterations for the early stage training to fit the posterior inference, which is considered as a hyper-parameter of UADAL. 
Figure \ref{sup-fig:res:ws} represents the performance metrics such as OS$^{*}$, UNK, and HOS over the number of the initial training iterations of UADAL.
It shows that UADAL is not sensitive to the number of iterations for the early stage. Taken together with Figure \ref{fig:histo_warm-up} in the main paper, these results represent that our two modality assumption for the target entropy values robustly holds.

\begin{figure*}[h]
\centering 
    \includegraphics[width=0.45\textwidth]{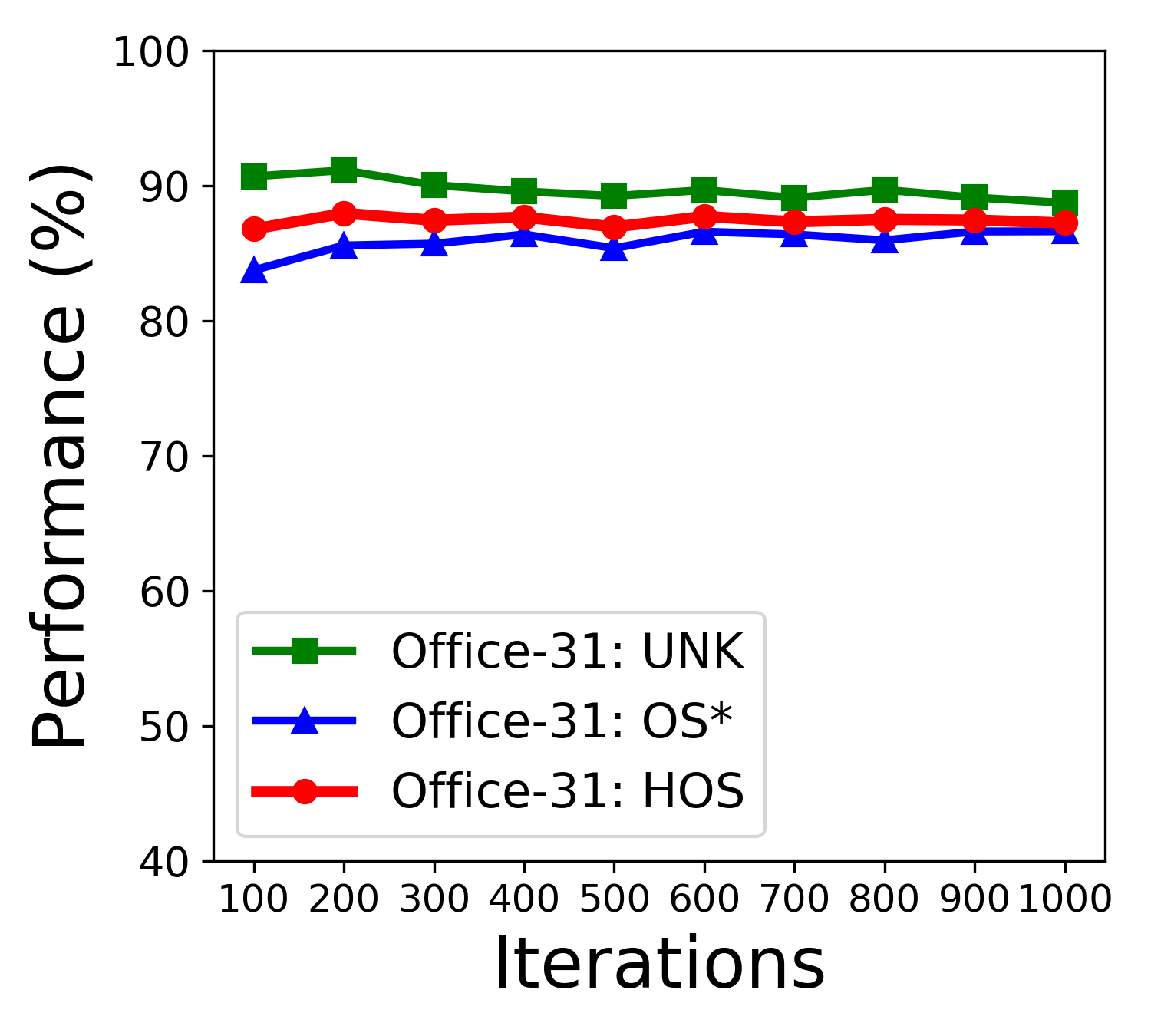}
\caption{Averaged performance over the tasks in Office-31 varying the number of iterations.} \label{sup-fig:res:ws}
\end{figure*}

\newpage

\subsubsection{Ablation Studies on Using Classifier $C$ as Entropy }\label{sup-appendix:EvsC}

\begin{table*}
\centering
\resizebox{\textwidth}{!}{%
\begin{tabular}{c|ccc|ccc|ccc|ccc|ccc|ccc||ccc}
\hline \hline
\multicolumn{22}{c}{\textbf{Office31 (ResNet-50)}} \\ \hline
  \multirow{2}{*}{\textbf{Network}} & \multicolumn{3}{|c}{{A $\rightarrow$ W}} & \multicolumn{3}{|c}{{A $\rightarrow$ D}} & \multicolumn{3}{|c}{{D $\rightarrow$ W}} & \multicolumn{3}{|c}{{W $\rightarrow$ D}} & \multicolumn{3}{|c}{{D $\rightarrow$ A}} & \multicolumn{3}{|c||}{{W $\rightarrow$ A}} & \multicolumn{3}{|c}{\textbf{Avg.}} \\
    & OS*    & UNK    & \textbf{HOS}   & OS*    & UNK    & \textbf{HOS}   & OS*    & UNK    & \textbf{HOS}   & OS*    & UNK    & \textbf{HOS}   & OS*    & UNK    & \textbf{HOS}   & OS*    & UNK    & \textbf{HOS}   &   OS*&UNK &\textbf{HOS}  \\ \hline
\textbf{C} & 80.5 & 92.4 & 86.1 & 79.4 & 91.0 & 84.8 & 99.0 & 98.3 & \textbf{98.6} & 98.7 & 100.0 & 99.3 & 68.7 & 89.7 & 77.9 & 56.5 & 89.1 & 68.9 & 80.5 & 93.4 & 85.9 \\
\textbf{E} & 84.3 & 94.5 & \textbf{89.1} & 85.1 & 87.0 & \textbf{86.0} & 99.3 & 96.3 & 97.8 & 99.5 & 99.4 &\textbf{ 99.5} & 73.3 & 87.3 & \textbf{79.7} & 67.4 & 88.4 & \textbf{76.5} & 84.8 & 92.1 & \textbf{88.1}\\ \hline \hline
\multicolumn{22}{c}{}      \\ \hline\hline
\multicolumn{22}{c}{\textbf{Office-Home (ResNet-50)}}                                                                                                                                                   \\ \hline
  \multirow{2}{*}{\textbf{Network}} & \multicolumn{3}{|c}{P$\rightarrow$R} & \multicolumn{3}{|c}{P$\rightarrow$C} & \multicolumn{3}{|c}{P$\rightarrow$A} & \multicolumn{3}{|c}{A$\rightarrow$P} & \multicolumn{3}{|c}{A$\rightarrow$R} & \multicolumn{3}{|c||}{A$\rightarrow$C} &  \multicolumn{3}{|c}{} \\  
   & OS*    & UNK    & \textbf{HOS}   & OS*    & UNK    & \textbf{HOS}   & OS*    & UNK    & \textbf{HOS}   & OS*    & UNK    & \textbf{HOS}   & OS*    & UNK    & \textbf{HOS}   & OS*    & UNK    & \textbf{HOS}   &   &&  \\ \cline{0-18}
\textbf{C} & 68.8 & 83.7 & 75.5 & 39.3 & 79.4 & 52.6 & 47.9 & 84.3 & 61.1 & 65.0 & 76.3 & 70.2 & 78.4 & 75.6 & 77.0 & 46.7 & 77.2 & 58.2 &  &  &   \\
\textbf{E} & 71.6 & 83.1 & \textbf{76.9} & 43.4 & 81.5 & \textbf{56.6} & 50.5 & 83.7 & \textbf{63.0} & 69.1 & 72.5 & \textbf{70.8} & 81.3 & 73.7 & \textbf{77.4} & 54.9 & 74.7 & \textbf{63.2} &  &  &   \\  \hline 

  \multirow{2}{*}{\textbf{Network}} & \multicolumn{3}{|c}{R$\rightarrow$A} & \multicolumn{3}{|c}{R$\rightarrow$P} & \multicolumn{3}{|c}{R$\rightarrow$C} & \multicolumn{3}{|c}{C$\rightarrow$R} & \multicolumn{3}{|c}{C$\rightarrow$A} & \multicolumn{3}{|c||}{C$\rightarrow$P} &  \multicolumn{3}{|c}{\textbf{Avg.}} \\  
    & OS*    & UNK    & \textbf{HOS}   & OS*    & UNK    & \textbf{HOS}   & OS*    & UNK    & \textbf{HOS}   & OS*    & UNK    & \textbf{HOS}   & OS*    & UNK    & \textbf{HOS}   & OS*    & UNK    & \textbf{HOS}   &   OS*&UNK &\textbf{HOS}  \\   \hline
\textbf{C} & 64.2 & 79.2 & 71.0 & 75.5 & 78.6 & \textbf{77.0} & 44.9 & 71.7 & 55.2 & 63.0 & 75.0 & 68.5 & 49.9 & 77.4 & 60.7 & 57.8 & 79.2 & 66.8 & 58.4 & 78.1 & 66.1 \\
\textbf{E} & 66.7 & 78.6 & \textbf{72.1} & 77.4 & 76.2 & 76.8 & 51.1 & 74.5 & \textbf{60.6} & 69.1 & 78.3 & \textbf{73.4} & 53.5 & 80.5 & \textbf{64.2} & 62.1 & 78.8 & \textbf{69.5} & 62.6 & 78.0 & \textbf{68.7} \\  \hline \hline
\end{tabular}%
}
\caption{{Ablation study for the network $C$ and $E$ to generate the entropy value in UADAL w.r.t. classification accuracies (\%) on Office-31 and Office-Home wiht ResNet-50 (bold: best performer).}} \label{sup-tab:albation_EvsC}
\end{table*}

This part introduces an ablation study for utilizing the classifier $C$ to generate the entropy values instead of an open-set recognizer $E$. 
Using the classifier $C$ (except the last unknown dimension) directly is also feasible, as $E$ does. 
Although it is feasible, however, it leads to wrong decisions for the target-unknown instances. 
As we explained, the network $E$ learns the decision boundary over $\mathcal{C}_s$ classes while the network $C$ does over $\mathcal{C}_{s}+1$ including \textit{unknown} class. 
Especially, in the case of the target-unknown instances, the network $C$ is enforced to classify the instances as \textit{unknown class}, which is $C_{s}+1$-th dimension. 
When optimizing the network $C$ with the target-unknown instances, we expect that the output of the network $C$ would have a higher value on the unknown dimension. 
With this point, if we use the first $C_s$ dimension of $C$ to calculate the entropy value, there is no evidence that the distribution of the $C$'s output is flat over $\mathcal{C}_s$ classes which implies to a higher entropy value. 
Even though the largest predicted value except the last dimension is small, the entropy value might be lower due to imbalance in the output. 
Then, it becomes to be considered as \textit{known} class, which is wrong decision for the target-unknown instances. Therefore, it gives the negative effects on the open-set recognition, and it adversely affects when training the model.

As an ablation experiment, we conduct the experiments to compare using $E$ or $C$ for the entropy. The experimental results on both cases are shown in Table \ref{sup-tab:albation_EvsC} of this section, applied to Office-31 and Offce-Home datasets. The network $E$ in Table \ref{sup-tab:albation_EvsC} is the current UADAL model and $C$ represents that the entropy values are generated by the classifier $C$ without introducing the network $E$. As you can see, the performances with $E$ is better than $C$. It means that the network $E$ learns the decision boundary for the known classes, and it leads to recognize the open-set instances effectively. It should be noted that we utilize the structure of $E$ as an one-layered network to reduce the computation burden.


\subsubsection{Ablation Studies on Entropy Minimization}\label{sup-appendix:entropymin}

The entropy minimization is important part for the fields such as semi-supervised learning \cite{grandvalet2005semi, sohn2020fixmatch} and domain adaptation \cite{long2016unsupervised,liu2019separate,DANCE2020NIPS,ROS2020ECCV} where the label information of the dataset is not available.  In order to show the effect of this term, we conduct the ablation study on the datasets of Office-31 and Office-Home. We provide the experimental results in Table \ref{sup-tab:ablation_entropy_loss}. Combined with the results in the main paper, the experimental result shows that UADAL without the entropy minimization loss still performs better than other baselines. It represents that UADAL learns the feature space appropriately as we intended to suit Open-Set Domain Adaptation. The properly learned feature space leads to effectively classify the target instances without the entropy minimization.

\begin{table*}
\centering
\resizebox{\textwidth}{!}{%
\begin{tabular}{c|ccc|ccc|ccc|ccc|ccc|ccc||ccc}
\hline \hline
\multicolumn{22}{c}{\textbf{Office31 (ResNet-50)}} \\ \hline
\textbf{Entropy} & \multicolumn{3}{|c}{{A $\rightarrow$ W}} & \multicolumn{3}{|c}{{A $\rightarrow$ D}} & \multicolumn{3}{|c}{{D $\rightarrow$ W}} & \multicolumn{3}{|c}{{W $\rightarrow$ D}} & \multicolumn{3}{|c}{{D $\rightarrow$ A}} & \multicolumn{3}{|c||}{{W $\rightarrow$ A}} & \multicolumn{3}{|c}{\textbf{Avg.}} \\
  \textbf{Minimization}  & OS*    & UNK    & \textbf{HOS}   & OS*    & UNK    & \textbf{HOS}   & OS*    & UNK    & \textbf{HOS}   & OS*    & UNK    & \textbf{HOS}   & OS*    & UNK    & \textbf{HOS}   & OS*    & UNK    & \textbf{HOS}   &   OS*&UNK &\textbf{HOS}  \\ \hline
\textbf{X} &85.9 & 84.4 & 85.1 & 84.7 & 83.6 & 84.2 & 95.6 & 98.9 & 97.2 & 98.7 & 100.0 & 99.3 & 75.2 & 86.0 & \textbf{80.3} & 72.6 & 87.4 & \textbf{79.3} & 85.4 & 90.0 & 87.5\\
\textbf{O} & 84.3 & 94.5 & \textbf{89.1} & 85.1 & 87.0 & \textbf{86.0} & 99.3 & 96.3 & \textbf{97.8} & 99.5 & 99.4 &\textbf{99.5} & 73.3 & 87.3 & {79.7} & 67.4 & 88.4 & {76.5} & 84.8 & 92.1 & \textbf{88.1}\\  \hline \hline
\multicolumn{22}{c}{}      \\ \hline\hline

\multicolumn{22}{c}{\textbf{Office-Home (ResNet-50)}}                                                                                                                                                   \\ \hline
 \textbf{Entropy} & \multicolumn{3}{|c}{P$\rightarrow$R} & \multicolumn{3}{|c}{P$\rightarrow$C} & \multicolumn{3}{|c}{P$\rightarrow$A} & \multicolumn{3}{|c}{A$\rightarrow$P} & \multicolumn{3}{|c}{A$\rightarrow$R} & \multicolumn{3}{|c||}{A$\rightarrow$C} &  \multicolumn{3}{|c}{} \\  
 \textbf{Minimization}  & OS*    & UNK    & \textbf{HOS}   & OS*    & UNK    & \textbf{HOS}   & OS*    & UNK    & \textbf{HOS}   & OS*    & UNK    & \textbf{HOS}   & OS*    & UNK    & \textbf{HOS}   & OS*    & UNK    & \textbf{HOS}   &   &&  \\ \cline{0-18}
\textbf{X} & 70.7 & 78.2 & 74.2 & 48.4 & 76.6 & \textbf{59.3} & 49.9 & 76.2 & 60.3 & 64.5 & 79.4 & \textbf{71.2} & 78.8 & 75.2 & 77.0 & 56.3 & 75.1 & \textbf{64.3} &  &  &  \\
\textbf{O} & 71.6 & 83.1 & \textbf{76.9} & 43.4 & 81.5 & {56.6} & 50.5 & 83.7 & \textbf{63.0} & 69.1 & 72.5 & {70.8} & 81.3 & 73.7 & \textbf{77.4} & 54.9 & 74.7 & {63.2} &  &  &   \\  \hline 

 \textbf{}& \multicolumn{3}{|c}{R$\rightarrow$A} & \multicolumn{3}{|c}{R$\rightarrow$P} & \multicolumn{3}{|c}{R$\rightarrow$C} & \multicolumn{3}{|c}{C$\rightarrow$R} & \multicolumn{3}{|c}{C$\rightarrow$A} & \multicolumn{3}{|c||}{C$\rightarrow$P} &  \multicolumn{3}{|c}{\textbf{Avg.}} \\  
 \textbf{}    & OS*    & UNK    & \textbf{HOS}   & OS*    & UNK    & \textbf{HOS}   & OS*    & UNK    & \textbf{HOS}   & OS*    & UNK    & \textbf{HOS}   & OS*    & UNK    & \textbf{HOS}   & OS*    & UNK    & \textbf{HOS}   &   OS*&UNK &\textbf{HOS}  \\   \hline
\textbf{X} &62.3 & 76.4 & 68.6 & 71.3 & 81.0 & 75.8 & 57.3 & 65.6 & \textbf{61.2} & 68.1 & 75.9 & 71.8 & 54.7 & 70.9 & 61.7 & 60.1 & 72.6 & 65.8 & 61.9 & 75.2 & 67.6 \\
\textbf{O} & 66.7 & 78.6 & \textbf{72.1} & 77.4 & 76.2 & \textbf{76.8} & 51.1 & 74.5 & {60.6} & 69.1 & 78.3 & \textbf{73.4} & 53.5 & 80.5 & \textbf{64.2} & 62.1 & 78.8 & \textbf{69.5} & 62.6 & 78.0 & \textbf{68.7} \\  \hline \hline

\end{tabular}%
}
\caption{{Ablation study for the entropy minimization loss in UADAL w.r.t. classification accuracies (\%) on Office-31 and Office-Home wiht ResNet-50 (bold: best performer).}} \label{sup-tab:ablation_entropy_loss}
\end{table*}

\subsubsection{Posterior Inference with Efficiency}\label{sup-appendix:sampling_bmm}

In terms of complexity, the posterior inference increases the computational complexity because we need to fit the mixture model. As we provided at the section \ref{sup-appendix:computational_complexity}, Wall-clock-time is increased as 5\% with the posterior inference in the case of full data utilization. From this +5\% increment, the performance has improved significantly than that without the posterior inference (as shown in Figure \ref{fig:entropy_trhesholding} in the paper). In addition, by utilizing the posterior inference, we avoid introducing any extra hyper-parameter to recognize the unknown instances, which is also our contribution. 

As an alternative, we fit the mixture model only by sampling the target instances in order to reduce the computation time because the computational complexity is $O(nk)$ where $n$ is the number of samples and $k$ is the number of fitting iterations (we fixed it as 10). Figure \ref{sup-appendix_fig:target_sample} represents the wall-clock time and the performance measures by sampling ratio (\%) for the target domain. Since the computational complexity is linearly increased by the number of samples, the wall-clock time is also linearly increased by increasing the sampling ratio. Interestingly, we observed that even though the sampling ratio is small, i.e. 10\%, the performances of UADAL w.r.t. HOS, OS*, and UNK does not decreased, on both Office-31 and Office-Home datasets.

\begin{figure*}[h!]
\centering
\resizebox{\textwidth}{!}{%
\begin{subfigure}[h]{0.4\textwidth} \centering
    \includegraphics[width=0.75\textwidth]{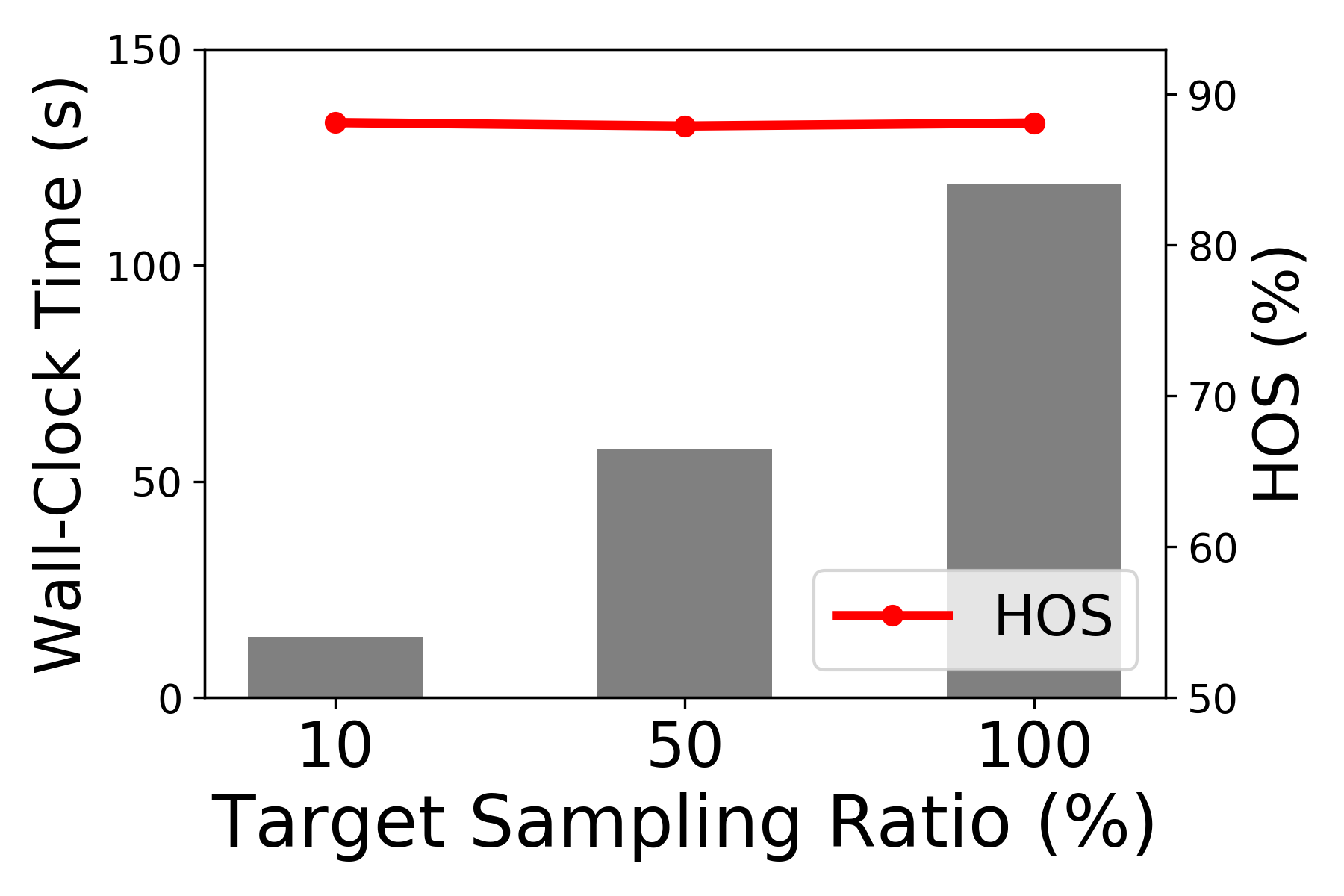}
    \caption{{Office-31}}
    \label{sup-fig:abl_sample_31}
\end{subfigure}%
\begin{subfigure}[h]{0.4\textwidth}\centering
    \includegraphics[width=0.75\textwidth]{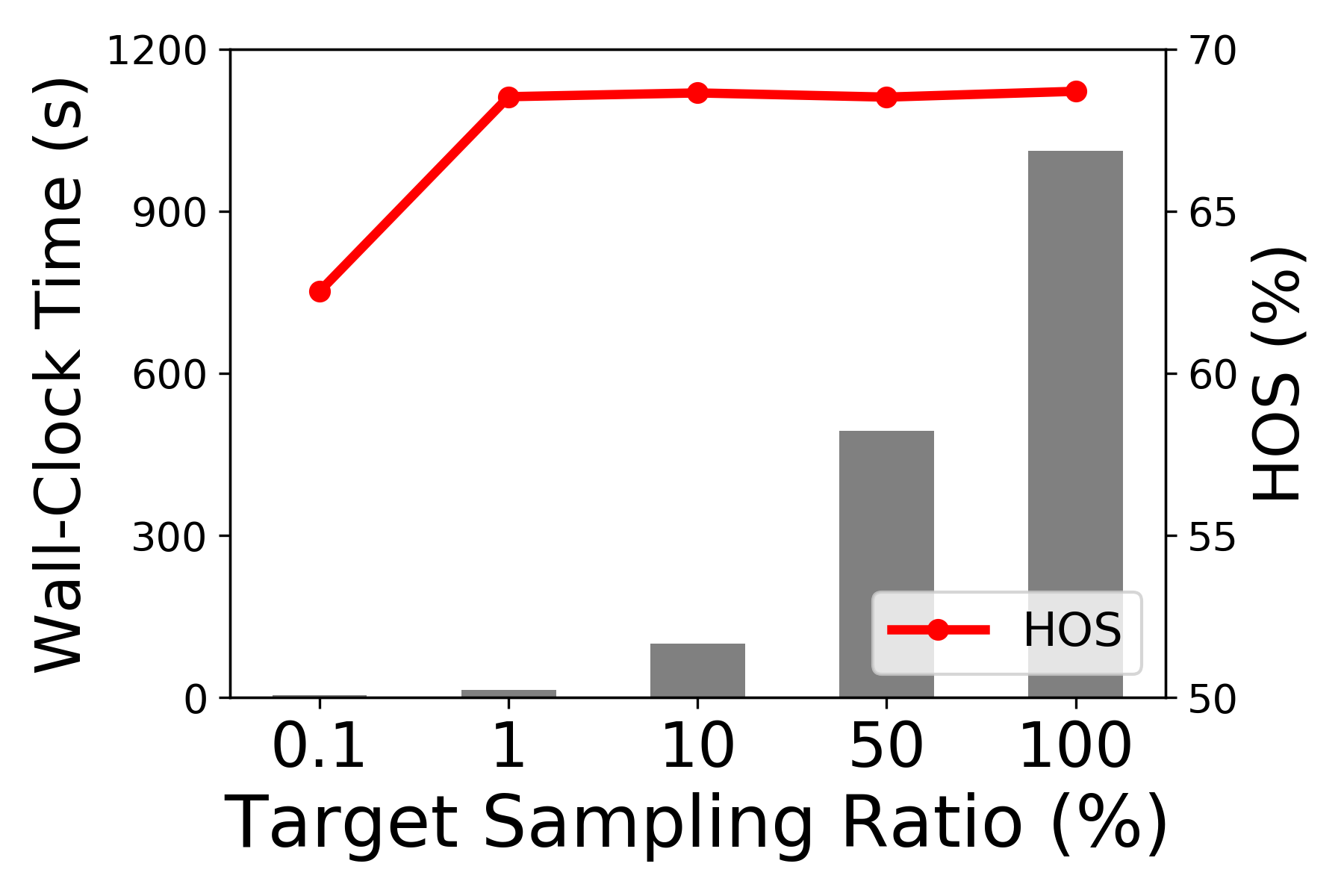}
    \caption{{Office-Home}}
    \label{sup-fig:abl_sample_home}
\end{subfigure}%
}
\caption{{Quantitative analysis for ablation study of applying sampling on the target domain with Office-31 (a) and Office-Home (b). Each subfigure represents the Wall-Clock Time (s) increased by fitting process of the mixture model during training and HOS ($\%$) over the target sampling ratio. (All records are the averaged values over the tasks in each dataset, not just single task.)}}
\label{sup-appendix_fig:target_sample}
\end{figure*}
In order to investigate the robustness on the sampling ratio, we provide the qualitative analysis in Figure \ref{sup-appendix_fig:target_sample_quali}. For each sampling ratio, the left figure represents the original target entropy distribution, and the middle shows the sampled target entropy values and the fitted BMM densities. Finally, the right figure represents the weight distribution by the posterior inference. As you can see, our posterior inference takes the entropy values, and fits the mixture model without any thresholds. Therefore, even if the sampling ratio is small, the observation that the target-unknown instances have higher entropy values than the target-known instances still holds. Therefore, the open-set recognition on the target domain is still informative, and it leads to maintain the performances of UADAL.

\begin{figure*}[h!]
\centering
\begin{subfigure}[h]{0.95\textwidth}
    \includegraphics[width=0.95\textwidth]{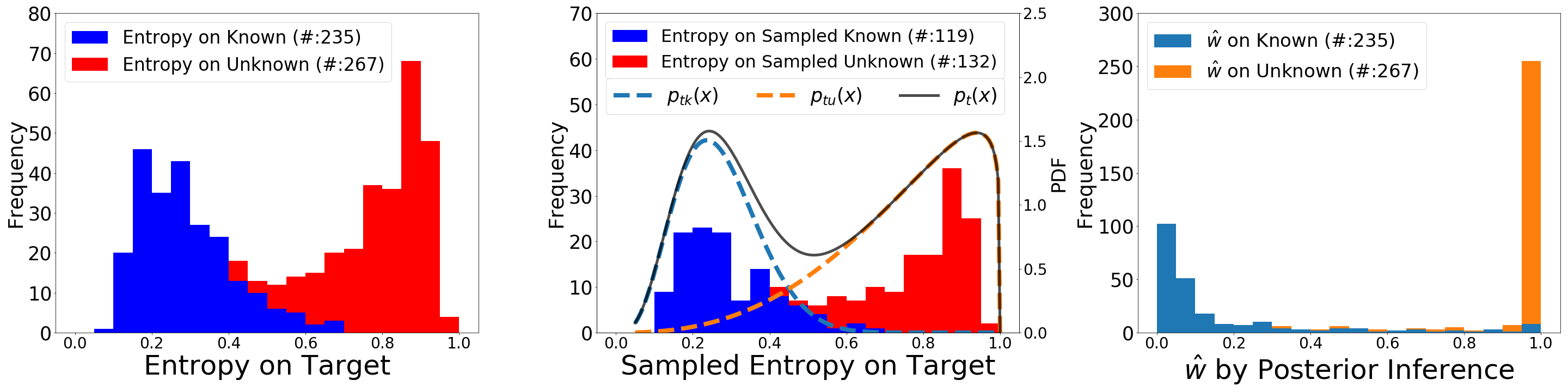}
    \caption{{Target Sampling Ratio: 50$\%$}}
    \label{sup-fig:abl_sample_31}
\end{subfigure} 
\begin{subfigure}[h]{0.95\textwidth}
    \includegraphics[width=0.95\textwidth]{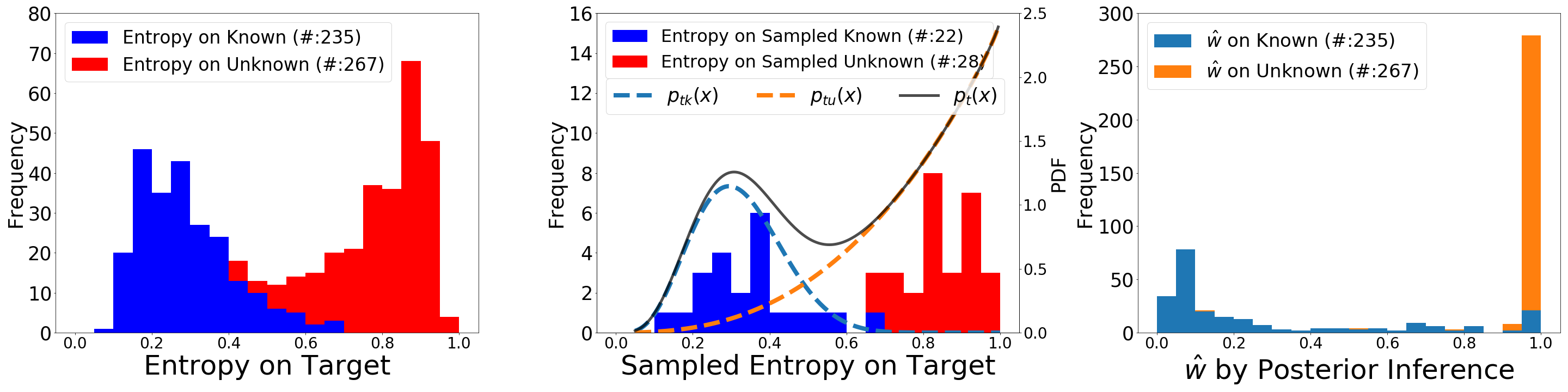}
    \caption{{Target Sampling Ratio: 10$\%$}}
    \label{sup-fig:abl_sample_home}
\end{subfigure} \hspace{-0.2em}
\caption{Qualitative analysis for ablation study of applying sampling on the target domain (D$\rightarrow$W task in Office-31). The subfigure (a) and (b) represent the different sampling ratio, 50$\%$ and 10$\%$, respectively. Each subfigure consists of 1) left: the original target entropy distribution, 2) midde: the sampled target entropy distribution with the fitted Beta Mixture Model (BMM), and 3) right: the weight ($\hat{w}$) distribution (by the fitted BMM in middle) on the target domain.}
\label{sup-appendix_fig:target_sample_quali}
\end{figure*}

\subsubsection{Full Experimental Results with All Metrics}\label{sup-appendix:full_table}

As a reminder, HOS metric is a harmonic mean of OS* and UNK where OS* is accuracy for the known class classification and UNK is for the unknown classification. 
Since Open-Set Domain Adaptation should perform well on both tasks, we choose HOS metric as a primary metric. 
For other metrics such as OS, OS*, and UNK, we provide the full experimental results including OS, OS*, and UNK in this section. 
First of all, we provide the summary table of the experimental results with the officially reported performances of the baselines, which is denoted as * for reliable and fair comparisons. 
Table \ref{sup-tab:summary} in this appendix shows that UADAL outperforms the baselines over all datasets, in the conventional setting of the backbone networks (such as Office-31/Office-Home with ResNet-50 and VisDA with VGGNet). 
The detailed results are shown in Table \ref{sup-tab:appendix_res:full_office31} for Office-31 and Table \ref{sup-tab:appendix_res:full_officehome_resnet} for Office-Home, in this appendix.


\begin{table*}[h]
\centering
\resizebox{0.8\textwidth}{!}{%
\begin{tabular}{c|cccc|cccc|cccc}
 \hline \hline
   \multirow{2}{*}{\textbf{Method}}    & \multicolumn{4}{|c|}{\textbf{Office31 (ResNet-50)}} & \multicolumn{4}{c|}{\textbf{Office-Home (ResNet-50)}} & \multicolumn{4}{c}{\textbf{VisDA (VGGNet)}} \\   \cline{2-13}
       &OS    & OS*   & UNK & \textbf{\underline{Avg. HOS}}   &OS    & OS*   & UNK & \textbf{\underline{Avg. HOS}} & OS    & OS*   & UNK & \textbf{\underline{HOS}}\\ \hline
DANN & 85.4 & 87.1 & 68.3 & 75.9\small{$\pm$0.5} & 53.5 & 52.6 & 77.1 & 60.7\small{$\pm$0.2} & - & - & - & - \\
CDAN & 86.1 & 88.3 & 63.9 & 73.4\small{$\pm$1.3} & 55.3 & 54.5 & 74.6 & 61.4\small{$\pm$0.3} & - & - & - & - \\
OSBP$^{*}$ & 86.6 & 87.2 & 80.4 & 83.7\small{$\pm$0.4} & 64.2 & 64.1 & 66.3 & 64.7\small{$\pm$0.2} & 62.9 & 59.2 & 85.1 & 69.8 \\
STA$^{*}$ & 82.5 & 84.3 & 64.8 & 72.5\small{$\pm$0.8} & 61.9 & 61.8 & 63.3 & 61.1\small{$\pm$0.3} & 66.8 & 63.9 & 84.2 & 72.7 \\
PGL$^{*}$ & 81.1 & 82.7 & 64.7 & 72.6\small{$\pm$1.5} & 74.1 & 76.1 & 25.0 & 35.2 & 80.7 & 82.8 & 68.1 & 74.7 \\
ROS$^{*}$ & 86.5 & 86.6 & 85.8 & 85.9\small{$\pm$0.2} & 62.0 & 61.6 & 72.4 & 66.2\small{$\pm$0.3} & - & - & - & - \\
DANCE & 91.0 & 94.0 & 60.2 & 73.1\small{$\pm$1.0} & 72.8 & 74.4 & 35.0 & 44.2\small{$\pm$0.6} & - & - & - & - \\
DCC$^{*}$ & - & - & - & 86.8 & - & - & - & 64.2 & 68.8 & 68.0 & 73.6 & 70.7 \\
LGU$^{*}$ & - & - & - & - & 71.4 & 72.7 & 38.9 & 50.7 & 70.1 & 69.2 & 75.5 & 72.2 \\
OSLPP$^{*}$ & 89.0 & 89.3 & 85.6 & 87.4 & 64.1 & 63.8 & 71.7 & 67.0 & - & - & - & - \\ \hline
\textbf{UADAL} & 85.5 & 84.8 & 92.1 & \underline{88.1\small{$\pm$0.2}} & 63.1 & 62.6 & 78.0 & \textbf{68.7\small{$\pm$0.2}} & 67.4 & 63.1 & 93.3 & \underline{75.3} \\
\textbf{cUADAL} & 85.6 & 84.8 & 93.0 & \textbf{88.5\small{$\pm$0.3}} & 63.1 & 62.5 & 77.6 & \underline{68.5\small{$\pm$0.1}} & 68.3 & 64.3 & 92.6 & \textbf{75.9}\\ \hline  \hline
\end{tabular}%
}
\caption{{Summary of the OSDA experimental results. The results in Office-31 and Office-Home are the averaged accuracies over the tasks because there are the multiple domains. (bold: best performer, underline: second-best performer, $^{*}$: officially reported performances.)}} \label{sup-tab:summary}
\end{table*}

\begin{table*}
\centering
\resizebox{\textwidth}{!}{%
\begin{tabular}{c|cccc|cccc|cccc||cccc}
\hline \hline
\multicolumn{17}{c}{\textbf{Office31 (ResNet-50)}} \\ \hline
& \multicolumn{4}{|c}{{A $\rightarrow$ W}} & \multicolumn{4}{|c}{{A $\rightarrow$ D}} & \multicolumn{4}{|c||}{{D $\rightarrow$ W}} &  &&&\\
\textbf{Model}& OS    & OS*   & UNK & \textbf{HOS}   &OS    & OS*   & UNK & \textbf{HOS} & OS    & OS*   & UNK & \textbf{HOS}  & &&&\\ \cline{1-13} 
DANN & 84.5 & 87.4 & 55.7 & 68.1 & 87.9 & 90.8 & 59.2 & 71.5 & 97.3 & 99.3 & 77.0 & 86.7 &  &  &  &  \\
CDAN & 86.7 & 90.3 & 50.7 & 64.9 & 88.6 & 92.2 & 52.4 & 66.8 & 97.2 & 99.6 & 73.2 & 84.3 &  &  &  &  \\
OSBP & 86.1 & 86.8 & 79.2 & 82.7 & 89.1 & 90.5 & 75.5 & 82.4 & 97.6 & 97.7 & 96.7 & 97.2 &  &  &  &  \\
STA & 85.0 & 86.7 & 67.6 & 75.9 & 88.5 & 91.0 & 63.9 & 75.0 & 90.6 & 94.1 & 55.5 & 69.8 &  &  &  &  \\
PGL & 81.4 & 82.7 & 67.9 & 74.6 & 80.5 & 82.1 & 65.4 & 72.8 & 85.7 & 87.5 & 68.1 & 76.5 &  &  &  &  \\
ROS & 87.3 & 88.4 & 76.7 & 82.1 & 86.6 & 87.5 & 77.8 & 82.4 & 98.7 & 99.3 & 93.0 & 96.0 &  &  &  &  \\
DANCE & 94.3 & 98.7 & 50.7 & 66.9 & 92.8 & 96.5 & 55.9 & 70.7 & 97.0 & 100.0 & 66.8 & 80.0 &  &  &  &  \\
DCC & - & - & - & 87.1 & - & - & - & 85.5 & - & - & - & 91.2 &  &  &  &  \\
LGU & - & - & - & - & - & - & - & - & - & - & - & - &  &  &  &  \\
OSLPP & 89.4 & 89.5 & 88.4 & 89.0 & 92.4 & 92.6 & 90.4 & \textbf{91.5} & 96.1 & 96.9 & 88.0 & 92.3 &  &  &  &  \\ \cline{1-13}
\textbf{UADAL} & 85.3 & 84.3 & 94.5 & \underline{89.1} & 85.2 & 85.1 & 87.0 & 86.0 & 99.0 & 99.3 & 96.3 & \underline{97.8} &  &  &  &  \\
\textbf{cUADAL} & 86.4 & 85.5 & 95.1 &\textbf{90.1} & 86.0 & 85.6 & 90.4 & \underline{87.9} & 98.6 & 98.7 & 97.7 &\textbf{98.2} &  &  &  &  \\ \hline \hline
  & \multicolumn{4}{|c}{{W $\rightarrow$ D}} & \multicolumn{4}{|c}{{D $\rightarrow$ A}} & \multicolumn{4}{|c||}{{W $\rightarrow$ A}} & \multicolumn{4}{|c}{\textbf{AVG.}}\\
  & OS    & OS*   & UNK & \textbf{HOS}   &OS    & OS*   & UNK & \textbf{HOS} & OS    & OS*   & UNK & \textbf{HOS}  & OS    & OS*   & UNK & \textbf{HOS}\\ \hline
DANN & 97.3 & 100.0 & 70.2 & 82.5 & 73.0 & 72.9 & 74.5 & 73.7 & 72.2 & 72.1 & 73.1 & 72.6 & 85.4 & 87.1 & 68.3 & 75.9$\pm$0.5 \\
CDAN & 97.0 & 100.0 & 67.3 & 80.5 & 74.5 & 74.9 & 70.6 & 72.7 & 72.5 & 72.8 & 69.3 & 71.0 & 86.1 & 88.3 & 63.9 & 73.4$\pm$1.3 \\
OSBP & 97.7 & 99.1 & 84.2 & 91.1 & 75.8 & 76.1 & 72.3 & 75.1 & 73.1 & 73.0 & 74.4 & 73.7 & 86.6 & 87.2 & 80.4 & 83.7$\pm$0.4 \\
STA & 83.3 & 84.9 & 67.8 & 75.2 & 81.5 & 83.1 & 65.9 & 73.2 & 66.4 & 66.2 & 68.0 & 66.1 & 82.5 & 84.3 & 64.8 & 72.5$\pm$0.8 \\
PGL & 81.1 & 82.8 & 64.0 & 72.2 & 78.8 & 80.6 & 61.2 & 69.5 & 79.1 & 80.8 & 61.8 & 70.1 & 81.1 & 82.7 & 64.7 & 72.6$\pm$1.5 \\
ROS & 99.9 & 100.0 & 99.4 & \textbf{99.7} & 75.4 & 74.8 & 81.2 & 77.9 & 71.2 & 69.7 & 86.6 & 77.2 & 86.5 & 86.6 & 85.8 & 85.9$\pm$0.2 \\
DANCE & 97.6 & 100.0 & 73.7 & 84.8 & 82.4 & 85.3 & 53.6 & 65.8 & 81.6 & 83.7 & 60.6 & 70.2 & 91.0 & 94.0 & 60.2 & 73.1$\pm$1.0 \\
DCC & - & - & - & 87.1 & - & - & - & \textbf{85.5} & - & - & - & \textbf{84.4} & - & - & - & 86.8 \\
LGU & - & - & - & - & - & - & - & - & - & - & - & - & - & - & - & - \\
OSLPP & 95.4 & 95.8 & 91.5 & 93.6 & 81.6 & 82.1 & 76.6 & 79.3 & 78.9 & 78.9 & 78.5 & \underline{78.7} & 89.0 & 89.3 & 85.6 & 87.4 \\ \hline
\textbf{UADAL} & 99.5 & 99.5 & 99.4 & \underline{99.5} & 74.5 & 73.3 & 87.3 & 79.7 & 69.3 & 67.4 & 88.4 & 76.5 & 85.5 & 84.8 & 92.1 & \underline{88.1$\pm$0.2} \\
\textbf{cUADAL} & 99.3 & 99.3 & 99.4 & 99.4 & 75.4 & 74.2 & 87.8 & \underline{80.5} & 67.6 & 65.6 & 87.8 & 75.1 & 85.6 & 84.8 & 93.0 & \textbf{88.5$\pm$0.3} \\ \hline \hline
\end{tabular}%
}
\caption{{Classification accuracy (\%) on Office-31 dataset using ResNet-50 as the backbone network. (bold: best performer, underline: second-best performer)}} \label{sup-tab:appendix_res:full_office31}
\end{table*}

 \begin{table*}[hbt!]
\centering
\resizebox{\textwidth}{!}{%
\begin{tabular}{c|cccc|cccc|cccc|cccc|cccc}
\hline \hline 
\multicolumn{21}{c}{\textbf{Office-Home (ResNet-50)}}                                                                                                                                                   \\ \hline
      & \multicolumn{4}{|c}{P$\rightarrow$R} & \multicolumn{4}{|c}{P$\rightarrow$C} & \multicolumn{4}{|c}{P$\rightarrow$A} & \multicolumn{4}{|c|}{A$\rightarrow$P} & && \\  
  \textbf{Model}    & OS    & OS*   & UNK & \textbf{HOS}   &OS    & OS*   & UNK & \textbf{HOS} & OS    & OS*   & UNK & \textbf{HOS}  & OS    & OS*   & UNK & \textbf{HOS} &  \multicolumn{4}{|c}{}   \\ \cline{1-17}
DANN & 67.9 & 67.7 & 72.0 & 69.8 & 32.3 & 30.1 & 86.3 & 44.6 & 44.0 & 42.4 & 83.9 & 56.3 & 60.4 & 60.0 & 71.3 & 65.2 \\
CDAN & 69.8 & 69.8 & 69.7 & 69.7 & 35.0 & 33.1 & 82.4 & 47.2 & 47.1 & 45.8 & 81.2 & 58.6 & 62.0 & 61.7 & 68.8 & 65.1 \\
OSBP & 76.0 & 76.2 & 71.7 & 73.9 & 45.3 & 44.5 & 66.3 & 53.2 & 59.4 & 59.1 & 68.1 & \underline{63.2} & 71.3 & 71.8 & 59.8 & 65.2 \\
STA & 75.7 & 76.2 & 64.3 & 69.5 & 45.1 & 44.2 & 67.1 & 53.2 & 54.9 & 54.2 & 72.4 & 61.9 & 67.2 & 68.0 & 48.4 & 54.0 \\
PGL & 82.6 & 84.8 & 27.6 & 41.6 & 58.4 & 59.2 & 38.4 & 46.6 & 72.2 & 73.7 & 34.7 & 47.2 & 77.1 & 78.9 & 32.1 & 45.6 \\
ROS & 71.1 & 70.8 & 78.4 & 74.4 & 47.5 & 46.5 & 71.2 & 56.3 & 57.6 & 57.3 & 64.3 & 60.6 & 68.5 & 68.4 & 70.3 & 69.3 \\
DANCE & 84.2 & 86.5 & 27.1 & 41.2 & 48.9 & 48.2 & 67.4 & 55.7 & 69.7 & 70.7 & 43.9 & 54.2 & 82.2 & 84.0 & 35.4 & 49.8 \\
DCC &  -& - &-  & 64.0 &-  &-  &-  & 52.8 &- & - & - & 59.5 &  -& - &-  & 67.4 \\
LGU & 81.2 & 82.8 & 41.2 & 55.0 & 53.1 & 54.5 & 18.1 & 27.2 & 68.4 & 69.1 & 50.9 & 58.6 & 79.3 & 80.5 & 49.3 & 61.2 \\
OSLPP & 76.8 & 77.0 & 71.2 & 74.0 & 53.6 & 53.1 & 67.1 & \textbf{59.3} & 55.4 & 54.6 & 76.2 & \textbf{63.6} & 72.5 & 72.5 & 73.1 & \textbf{72.8} \\ \cline{1-17}
\textbf{UADAL} & 72.1 & 71.6 & 83.1 & \textbf{76.9} & 44.9 & 43.4 & 81.5 & \underline{56.6} & 51.8 & 50.5 & 83.7 & 63.0 & 69.2 & 69.1 & 72.5 & 70.8 \\
\textbf{cUADAL} & 71.7 & 71.2 & 83.4 & \underline{76.8} & 42.7 & 41.2 & 80.7 & 54.6 & 52.1 & 50.9 & 82.4 & 62.9 & 69.6 & 69.4 & 73.9 & \underline{71.6}\\ \hline \hline
 & \multicolumn{4}{|c}{A$\rightarrow$R} & \multicolumn{4}{|c|}{A$\rightarrow$C}  & \multicolumn{4}{|c}{R$\rightarrow$A} & \multicolumn{4}{|c}{R$\rightarrow$P}&  \multicolumn{4}{|c}{} \\
      &OS    & OS*   & UNK & \textbf{HOS}   &OS    & OS*   & UNK & \textbf{HOS} & OS    & OS*   & UNK & \textbf{HOS}  & OS    & OS*   & UNK & \textbf{HOS}   &  \multicolumn{4}{|c}{}  \\  \cline{1-17}
DANN & 74.8 & 75.1 & 67.3 & 71.0 & 38.9 & 37.1 & 82.7 & 51.2 & 57.6 & 56.8 & 77.1 & 65.4 & 69.5 & 69.6 & 67.2 & 68.4 \\
CDAN & 74.8 & 75.2 & 66.7 & 70.7 & 41.2 & 39.7 & 78.9 & 52.9 & 60.4 & 59.8 & 73.6 & 66.0 & 70.6 & 70.9 & 64.6 & 67.6 \\
OSBP & 78.8 & 79.3 & 67.5 & 72.9 & 50.6 & 50.2 & 61.1 & 55.1 & 66.1 & 66.1 & 67.3 & 66.7 & 76.0 & 76.3 & 68.6 & 72.3 \\
STA & 77.9 & 78.6 & 60.4 & 68.3 & 47.0 & 46.0 & 72.3 & 55.8 & 67.5 & 67.5 & 66.7 & 67.1 & 76.3 & 77.1 & 55.4 & 64.5 \\
PGL & 85.9 & 87.7 & 40.9 & 55.8 & 61.6 & 63.3 & 19.1 & 29.3 & 78.6 & 81.5 & 6.1 & 11.4 & 83.0 & 84.8 & 38.0 & 52.5 \\
ROS & 75.9 & 75.8 & 77.2 & 76.5 & 51.5 & 50.6 & 74.1 & 60.1 & 67.1 & 67.0 & 70.8 & 68.8 & 72.3 & 72.0 & 80.0 & 75.7 \\
DANCE & 87.4 & 89.8 & 25.3 & 39.4 & 54.4 & 54.4 & 53.7 & 53.1 & 76.8 & 79.2 & 16.7 & 27.5 & 84.1 & 86.2 & 29.6 & 44.0 \\
DCC & - & - & - & \textbf{80.6} & - & - & - & 52.9 & - & - & - & 56.0 & - & - & - & 62.7 \\
LGU & 85.0 & 86.5 & 47.5 & 61.3 & 57.6 & 58.6 & 32.6 & 41.9 & 76.4 & 77.5 & 48.9 & 60.0 & 81.8 & 83.2 & 46.8 & 59.9 \\
OSLPP & 79.7 & 80.1 & 69.4 & 74.3 & 56.3 & 55.9 & 67.1 & 61.0 & 61.3 & 60.8 & 75.0 & 67.2 & 78.1 & 78.4 & 70.8 & 74.4 \\ \cline{1-17}
\textbf{UADAL} & 81.0 & 81.3 & 73.7 & 77.4 & 55.7 & 54.9 & 74.7 & \underline{63.2} & 67.1 & 66.7 & 78.6 & \underline{72.1} & 77.3 & 77.4 & 76.2 & \textbf{76.8} \\
\textbf{cUADAL} & 81.8 & 82.2 & 73.3 & \underline{77.5} & 55.8 & 55.0 & 75.6 & \textbf{63.6} & 67.3 & 66.8 & 79.6 & \textbf{72.6} & 77.7 & 77.8 & 75.6 & \underline{76.7}\\ \hline \hline
& \multicolumn{4}{|c}{R$\rightarrow$C} & \multicolumn{4}{|c}{C$\rightarrow$R} & \multicolumn{4}{|c}{C$\rightarrow$A} & \multicolumn{4}{|c|}{C$\rightarrow$P} &  \multicolumn{4}{|c}{\textbf{AVG.}} \\  
      & OS    & OS*   & UNK & \textbf{HOS}   &OS    & OS*   & UNK & \textbf{HOS} & OS    & OS*   & UNK & \textbf{HOS}  & OS    & OS*   & UNK & \textbf{HOS}  & OS    & OS*   & UNK & \textbf{HOS}   \\ \hline
DANN & 38.8 & 37.1 & 80.9 & 50.9 & 61.6 & 61.1 & 73.5 & 66.7 & 45.4 & 43.8 & 84.3 & 57.6 & 51.2 & 50.1 & 77.6 & 60.9 & 53.5 & 52.6 & 77.1 & 60.7$\pm$0.2 \\
CDAN & 41.7 & 40.3 & 75.8 & 52.7 & 62.0 & 61.5 & 73.7 & 67.1 & 46.4 & 44.9 & 82.8 & 58.2 & 52.6 & 51.6 & 76.8 & 61.7 & 55.3 & 54.5 & 74.6 & 61.4$\pm$0.3 \\
OSBP & 48.6 & 48.0 & 63.0 & 54.5 & 71.9 & 72.0 & 69.2 & 70.6 & 59.8 & 59.4 & 70.3 & \underline{64.3} & 66.8 & 67.0 & 62.7 & 64.7 & 64.2 & 64.1 & 66.3 & 64.7$\pm$0.2 \\
STA & 50.3 & 49.9 & 61.1 & 54.5 & 67.0 & 67.0 & 66.7 & 66.8 & 51.9 & 51.4 & 65.0 & 57.4 & 61.7 & 61.8 & 59.1 & 60.4 & 61.9 & 61.8 & 63.3 & 61.1$\pm$0.3 \\
PGL & 66.2 & 68.8 & 0.0 & 0.0 & 68.8 & 70.2 & 33.8 & 45.6 & 82.8 & 85.9 & 5.3 & 10.0 & 72.0 & 73.9 & 24.5 & 36.8 & 74.1 & 76.1 & 25.0 & 35.2 \\
ROS & 52.3 & 51.5 & 73.0 & 60.4 & 65.6 & 65.3 & 72.2 & 68.6 & 54.1 & 53.6 & 65.5 & 58.9 & 60.3 & 59.8 & 71.6 & 65.2 & 62.0 & 61.6 & 72.4 & 66.2$\pm$0.3 \\
DANCE & 59.4 & 60.1 & 41.3 & 48.3 & 81.3 & 83.9 & 18.4 & 30.2 & 71.2 & 72.9 & 28.4 & 40.9 & 74.6 & 76.3 & 32.8 & 45.9 & 72.8 & 74.4 & 35.0 & 44.2$\pm$0.6 \\
DCC & - & - & - & \textbf{76.9} & - & - & - & 67.0 & - & - & - & 49.8 & - & - & - & 66.6 & - & - & - & 64.2 \\
LGU & 62.1 & 63.4 & 29.6 & 40.4 & 76.4 & 77.6 & 46.4 & 58.1 & 65.8 & 67.2 & 30.8 & 42.2 & 69.1 & 71.7 & 4.1 & 7.8 & 71.4 & 72.7 & 38.9 & 50.7 \\
OSLPP & 54.8 & 54.4 & 64.3 & 59.0 & 67.5 & 67.2 & 73.9 & 70.4 & 50.7 & 49.6 & 79.0 & 60.9 & 62.1 & 61.6 & 73.3 & 66.9 & 64.1 & 63.8 & 71.7 & 67.0 \\ \hline
\textbf{UADAL} & 52.0 & 51.1 & 74.5 & \underline{60.6} & 69.4 & 69.1 & 78.3 & \textbf{73.4} & 54.5 & 53.5 & 80.5 & 64.2 & 62.8 & 62.1 & 78.8 & \textbf{69.5} & 63.1 & 62.6 & 78.0 & \textbf{68.7$\pm$0.2} \\
\textbf{cUADAL} & 52.5 & 51.8 & 71.1 & 59.9 & 69.5 & 69.3 & 76.3 & \underline{72.6} & 54.9 & 53.8 & 82.0 & \textbf{65.0} & 61.8 & 61.1 & 77.4 & \underline{68.3} & 63.1 & 62.5 & 77.6 & \underline{68.5$\pm$0.1} \\ \hline \hline

\end{tabular}%
}
\caption{{Classification accuracy (\%) on Office-Home dataset using ResNet-50 as the backbone network. (bold: best performer, underline: second-best performer)}} \label{sup-tab:appendix_res:full_officehome_resnet}
\end{table*}

\section{Limitations and Potential Negative Societal Impacts}\label{sup-sec:limitation}

\textbf{Limitations} Our domain adaptation setting assumes that we have an access to a labeled source dataset and an unlabeled target dataset, simultaneously.
Thus, we may encounter the situation where the access for the source dataset and the target dataset is not available at the same time, i.e. streamlined data gathering. 
In addition, our work solves the Open-Set Domain Adaptation problem. It intrinsically assumes the existence of ‘unknown’ information in the target domain. Our open-set recognition is based on this assumption, thus we fit the mixture model where each mode represents for known/unknown information. We think that the common assumption of the high entropy value on target-unknowns could be considered as a limitation, as well.

\textbf{Potential Negative Societal Impacts} 
Because open-set domain adaptation focuses on categories belonging to the class of the source dataset, it is infeasible to distinguish differences between categories that are only within the target dataset. Therefore, if the source dataset’s categories are not sufficient, important categories within the target dataset may not be classified, which would lead to only limited applications when we have social stratifications.

\end{document}